\documentclass[aos]{imsart}
\RequirePackage[OT1]{fontenc}
\RequirePackage{amsthm,amsmath,mathrsfs,multirow,morefloats,booktabs,subfigure,pdfpages,color,algorithm,array,hhline,makecell,epstopdf,algorithm}
\RequirePackage[numbers]{natbib}
\RequirePackage[colorlinks,citecolor=blue,urlcolor=blue]{hyperref}
\RequirePackage{algorithm,algpseudocode}
\algrenewcommand\algorithmicrequire{\textbf{Input:}}
\algrenewcommand\algorithmicensure{\textbf{Output:}}
\usepackage{amsmath,amssymb}
\usepackage{dsfont}

\startlocaldefs
\numberwithin{equation}{section}
\theoremstyle{plain}
\newtheorem{thm}{Theorem}[section]
\newtheorem{defin}[thm]{Definition}
\newtheorem{lem}[thm]{Lemma}

\newtheorem{rem}[thm]{Remark}
\newtheorem{cor}[thm]{Corollary}

\endlocaldefs
\allowdisplaybreaks[4]
\usepackage{titlesec}
\titleformat*{\section}{\bfseries\Large\rmfamily}
\titleformat*{\subsection}{\bfseries\large\rmfamily}
\titleformat*{\subsubsection}{\bfseries\normalsize\rmfamily}
\begin{document}
\begin{frontmatter}
\title{\Large
A useful criterion on studying consistent estimation in community detection}
\begin{aug}
\author[A]{\fnms{Huan} \snm{Qing}\ead[label=e1]{qinghuan@cumt.edu.cn}}
\runauthor{H.Qing}
\address[A]{
	School of Mathematics,
	China University of Mining and Technology,
	\printead{e1}}
\end{aug}
\begin{abstract}
\begin{center}
\textbf{Abstract}
\end{center}
In network analysis, developing a unified theoretical framework that can compare methods under different models is an interesting problem. This paper proposes a partial solution to this problem. We summarize the idea of using separation condition for a standard network and sharp threshold of Erd\"os-R\'enyi random graph to study consistent estimation, compare theoretical error rates and requirements on network sparsity of spectral methods under models that can degenerate to stochastic block model as a four-step criterion SCSTC. Using SCSTC, we find some inconsistent phenomena on separation condition and sharp threshold in community detection. Especially, we find  original theoretical results of the SPACL algorithm introduced to estimate network memberships under the mixed membership stochastic blockmodel were sub-optimal. To find the formation mechanism of inconsistencies, we re-establish theoretical convergence rates of this algorithm by applying recent techniques on row-wise eigenvector deviation. The results are further extended to the degree corrected mixed membership model. By comparison, our results enjoy smaller error rates, lesser dependence on the number of communities, weaker requirements on network sparsity, and so forth. Furthermore, separation condition and sharp threshold obtained from our theoretical results match classical results, which shows the usefulness of this criterion on studying consistent estimation.
\end{abstract}
\end{frontmatter}
\section{Introduction}\label{sec1}
Estimating mixed memberships of network whose node may belong to multiple communities has received a lot of attention \cite{MMSB,ball2011efficient,wang2011community,gopalan2013efficient,airoldi2013multi,anandkumar2014a,kaufmann2017a,panov2017consistent,OCCAM,MixedSCORE,GeoNMF,MaoSVM,mao2020estimating}. To capture the structure of network with mixed memberships, \cite{MMSB} proposes the popular mixed membership stochastic blockmodel (MMSB), which is an extension of the famous stochastic blockmodels \cite{SBM} for non-overlapping networks. It is well known that the degree corrected stochastic blockmodel (DCSBM) \cite{DCSBM} is an extension of SBM by considering degree heterogeneity of nodes to fit the real world networks with various nodes degrees, similarly, \cite{MixedSCORE} proposes a model named degree corrected mixed membership (DCMM) model as is an extension of MMSB by considering degree heterogeneity of nodes. There are alternative models based on MMSB such as the OCCAM model of \cite{OCCAM} and the stochastic blockmodel with overlap (SBMO) proposed of \cite{kaufmann2017a} which can also model networks with mixed memberships. As discussed in Section \ref{intDCMM}, OCCAM equals DCMM while SBMO is a special case of DCMM. For these models, many researchers focus on designing algorithms with provable consistent theoretical guarantees. \cite{lei2015consistency} studies the consistences of two spectral clustering algorithms under SBM and DCSBM. \cite{mao2020estimating} designs an algorithm SPACL based on the finding that there exists simplex structure in the eigen-decomposition of the population adjacency matrix and studies SPACL's theoretical properties under MMSB. To fit DCMM, \cite{MixedSCORE} designs Mixed-SCORE algorithm based on the finding that there exists a simplex structure in the entry-wise ratio matrix obtained from the eigen-decomposition of the population adjacency matrix, where the entry-wise ratio idea comes from \cite{SCORE} which designs the SCORE algorithm with theoretical guarantee under DCSBM. \cite{MaoSVM} finds the cone structure inherent in the normalization of eigenvectors of the population adjacency matrix under DCMM as well as OCCAM, and develops an algorithm to hunt corners in the cone structure.

In this paper, we focus on the consistency of spectral method in community detection. The study of consistency is developed by obtaining theoretical upper bound of error rate for a spectral method through analyzing the properties of the population adjacency matrix under statistical model. To compare consistencies of theoretical results under different models, it is meaningful to study that whether the separation condition of a balanced network and sharp threshold of the Erd\"os-R\'enyi (ER) random graph $G(n,p)$ \citep{erdos2011on} obtained from upper bounds of theoretical error rates for different methods under different models are consistent or not. Meanwhile, separation condition and sharp threshold can also be seen as alternative unified theoretical frameworks to compare all methods and model parameters mentioned in the concluding remarks of \cite{lei2015consistency}. Furthermore, when $\mathrm{Method}_{a}$ and $\mathrm{Method}_{b}$ are designed under the framework of a same model, theoretical results about error rates developed for $\mathrm{Method}_{a}$ should be consistent with those developed for $\mathrm{Method}_{b}$ at least under mild conditions. Based on the three ideas, now we are ready to describe some phenomenons of inconsistency in community detection area. We find that the separation conditions of a balanced network obtained from the error rates developed in \cite{MixedSCORE, mao2020estimating, MaoSVM} under DCMM or MMSB are not consistent with that obtained from main results of \cite{lei2015consistency} under SBM, sharp threshold obtained from main results of \cite{mao2020estimating, MaoSVM} do not match classical results. Meanwhile, though both \cite{MixedSCORE} and \cite{MaoSVM} study the consistencies of their spectral algorithms under DCMM, their theoretical upper bounds of error rates do not match even under mild conditions. A summary of these inconsistencies are provided in Tables \ref{SCST} and \ref{aSCST}.  Furthermore, after delicate analysis, we find that the requirement on network sparsity of \cite{mao2020estimating, MaoSVM} are stronger than that of \cite{MixedSCORE, lei2015consistency}, and \cite{lei2019unified} also finds that \cite{mao2020estimating}'s requirement of network sparsity is sub-optimal.
\begin{table}
\scriptsize
\centering
\caption{Comparison of separation condition and sharp threshold. Details of this table are given in Section \ref{MainSCSTC}. The classical result on separation condition given in Corollary 1 of \cite{mcsherry2001spectral} is $\sqrt{\frac{\mathrm{log}(n)}{n}}$. The classical result on sharp threshold is $\frac{\mathrm{log}(n)}{n}$ given in \citep{erdos2011on}, Theorem 4.6 \citep{blum2020foundations} and the first bullet in Section 2.5 \citep{abbe2017community}. In this paper, $n$ is the number of nodes in a network, $A$ is the adjacency matrix, $\Omega$ is the expectation of $A$ under some models, $A_{\mathrm{re}}$ is a regularization of $A$, $\rho$ is the sparsity parameter such that $\rho\geq \mathrm{max}_{i,j}\Omega(i,j)$ and it controls the overall sparsity of a network, $\|\cdot\|$ denotes spectral norm, and $\xi>1$.}
\label{SCST}
\begin{tabular}{cccccccccc}
\toprule
&model&separation condition&sharp threshold\\
\midrule
Ours using $\|A_{\mathrm{re}}-\Omega\|\leq C\sqrt{\rho n}$ &MMSB\&DCMM&$\sqrt{\frac{\mathrm{log}(n)}{n}}$&$\frac{\mathrm{log}(n)}{n}$\\
Ours using $\|A-\Omega\|\leq C\sqrt{\rho n\mathrm{log}(n)}$ &MMSB\&DCMM&$\sqrt{\frac{\mathrm{log}(n)}{n}}$&$\frac{\mathrm{log}(n)}{n}$\\
\hline
\cite{MixedSCORE} using $\|A_{\mathrm{re}}-\Omega\|\leq C\sqrt{\rho n}$ (original)&DCMM&$\sqrt{\frac{\mathrm{log}(n)}{n}}$&$\frac{\mathrm{log}(n)}{n}$\\
\cite{MixedSCORE} using $\|A-\Omega\|\leq C\sqrt{\rho n\mathrm{log}(n)}$&DCMM&$\sqrt{\frac{\mathrm{log}(n)}{n}}$&$\frac{\mathrm{log}(n)}{n}$\\
\hline
\cite{MaoSVM,mao2020estimating} using $\|A_{\mathrm{re}}-\Omega\|\leq C\sqrt{\rho n}$ (original)&MMSB\&DCMM&$\frac{\mathrm{log}^{\xi}(n)}{\sqrt{n}}$&$\frac{\mathrm{log}^{2\xi}(n)}{n}$\\
\cite{MaoSVM,mao2020estimating} using $\|A-\Omega\|\leq C\sqrt{\rho n\mathrm{log}(n)}$ &MMSB\&DCMM&$\frac{\mathrm{log}^{\xi+0.5}(n)}{\sqrt{n}}$&$\frac{\mathrm{log}^{2\xi+1}(n)}{n}$\\
\hline
\cite{lei2015consistency} using $\|A_{\mathrm{re}}-\Omega\|\leq C\sqrt{\rho n}$ (original)&SBM\&DCSBM&$\sqrt{\frac{1}{n}}$&$\frac{1}{n}$\\
\cite{lei2015consistency} using $\|A-\Omega\|\leq C\sqrt{\rho n\mathrm{log}(n)\mathrm{log}(n)}$&SBM\&DCSBM&$\sqrt{\frac{\mathrm{log}(n)}{n}}$&$\frac{\mathrm{log}(n)}{n}$\\
\bottomrule
\end{tabular}
\end{table}

\begin{table}
\scriptsize
\centering
\caption{Comparison of alternative separation condition.}
\label{aSCST}
\begin{tabular}{cccccccccc}
\toprule
&model&alternative separation condition\\
\midrule
Ours using $\|A_{\mathrm{re}}-\Omega\|\leq C\sqrt{\rho n}$ &MMSB\&DCMM&$1$\\
Ours using $\|A-\Omega\|\leq C\sqrt{\rho n\mathrm{log}(n)}$ &MMSB\&DCMM&$1$\\
\hline
\cite{MixedSCORE} using $\|A_{\mathrm{re}}-\Omega\|\leq C\sqrt{\rho n}$ (original)&DCMM&$1$\\
\cite{MixedSCORE} using $\|A-\Omega\|\leq C\sqrt{\rho n\mathrm{log}(n)}$&DCMM&$1$\\
\hline
\cite{MaoSVM,mao2020estimating} using $\|A_{\mathrm{re}}-\Omega\|\leq C\sqrt{\rho n}$ (original)&MMSB\&DCMM&$\mathrm{log}^{\xi-0.5}(n)$\\
\cite{MaoSVM,mao2020estimating} using $\|A-\Omega\|\leq C\sqrt{\rho n\mathrm{log}(n)}$ &MMSB\&DCMM&$\mathrm{log}^{\xi}(n)$\\
\hline
\cite{lei2015consistency} using $\|A_{\mathrm{re}}-\Omega\|\leq C\sqrt{\rho n}$ (original)&SBM\&DCSBM&$\sqrt{\frac{1}{\mathrm{log}(n)}}$\\
\cite{lei2015consistency} using $\|A-\Omega\|\leq C\sqrt{\rho n\mathrm{log}(n)\mathrm{log}(n)}$&SBM\&DCSBM&$1$\\
\bottomrule
\end{tabular}
\end{table}
For readers' convenience to have a better understanding of Tables \ref{SCST} and \ref{aSCST}, the definitions of separation condition, sharp threshold and alternative separation condition are given here. Consider a network with $K$ communities and $n$ nodes where sizes of each communities are in the same order, nodes have close degrees and $K$ is small. Such network is called standard network (or balanced network) in this paper. In a standard network, nodes connect with probability $p_{\mathrm{in}}$ within clusters and $p_{\mathrm{out}}$ across clusters. When $K\geq 2$, the lower bound requirement on $\frac{|p_{\mathrm{in}}-p_{\mathrm{out}}|}{\sqrt{p_{\mathrm{in}}}}$ for consistent estimation of spectral methods is called  separation condition; when $K=1$ such that $p=p_{\mathrm{in}}=p_{\mathrm{out}}$, the network degenerates to Erd\"os-R\'enyi (ER) random graph $G(n,p)$. The lower bound requirement on $p$ for generating a connected ER random graph is sharp threshold. Let
$p_{\mathrm{in}}=\alpha_{\mathrm{in}}\frac{\mathrm{log}(n)}{n}, p_{\mathrm{out}}=\alpha_{\mathrm{out}}\frac{\mathrm{log}(n)}{n}$. The alternative separation condition is defined as the lower bound requirement on $\frac{|\alpha_{\mathrm{in}}-\alpha_{\mathrm{out}}|}{\sqrt{\alpha_{\mathrm{in}}}}$ for consistent estimation when $K\geq 2$.

The separation condition of a standard network under SBM has been studied in \cite{mcsherry2001spectral,lei2015consistency,mao2020estimating,qingDiSC,qingDiDCMM} for their spectral methods. Especially, \cite{mcsherry2001spectral} finds that for large enough constant $c$, spectral methods can exactly recover communities with high probability as $n\rightarrow\infty$ if $\frac{p_{\mathrm{in}}-p_{\mathrm{out}}}{\sqrt{p_{\mathrm{in}}}}\geq c\sqrt{\frac{\mathrm{log}(n)}{n}}$ (i.e., $\frac{p_{\mathrm{in}}-p_{\mathrm{out}}}{\sqrt{p_{\mathrm{in}}}}\gg \sqrt{\frac{\mathrm{log}(n)}{n}}$) when $K=2$ for the case $p_{\mathrm{in}}>p_{\mathrm{out}}$, and this condition is the same as requiring that $\frac{\alpha_{\mathrm{in}}-\alpha_{\mathrm{out}}}{\sqrt{\alpha_{\mathrm{in}}}}\geq c$ (i.e., $\frac{\alpha_{\mathrm{in}}-\alpha_{\mathrm{out}}}{\sqrt{\alpha_{\mathrm{in}}}}\gg 1$). The sharp threshold of ER random graph $G(n,p)$ has been studied in \cite{erdos2011on,blum2020foundations,abbe2017community,qingDiSC,qingDiDCMM}.  Especially, \cite{erdos2011on} finds that the ER random graph is connected with high probability if $p\geq \frac{\mathrm{log}(n)}{n}$. Instead of showing or designing algorithms that can exactly recover labels with high probability  when $\alpha_{\mathrm{in}}$ and $\alpha_{\mathrm{out}}$ are close to the above limits, we find that separation condition  $\frac{|p_{\mathrm{in}}-p_{\mathrm{out}}|}{\sqrt{p_{\mathrm{in}}}}\gg \sqrt{\frac{\mathrm{log}(n)}{n}}$ (or the alternative separation condition $\frac{|\alpha_{\mathrm{in}}-\alpha_{\mathrm{out}}|}{\sqrt{\alpha_{\mathrm{in}}}}\gg1$) and sharp threshold $p\geq \frac{\mathrm{log}(n)}{n}$ are useful to compare different spectral methods under various models. In this paper, we summarize the idea of using separation condition and sharp threshold to study the consistencies, compare the error rates and requirements on network sparsity of different spectral methods under different models as a four step criterion which we call separation condition and sharp threshold criterion (SCSTC for short). With an application of this criterion, this paper provides an attempt to answer the following questions: how the above inconsistency phenomenons occur, and how to obtain consistency results with weaker requirements on network sparsity of \cite{mao2020estimating} and \cite{MaoSVM}. To answer the two questions, we use the recent techniques on row-wise eigenvector deviation developed in \cite{chen2020spectral} and \cite{cape2019the} to obtain consistent theoretical results directly related with model parameters for the SPACL algorithm of \cite{mao2020estimating} and the SVM-cone-DCMMSB algorithm of \cite{MaoSVM}. The two questions are then answered by delicate analysis with an application of SCSTC to theoretical upper bounds of error rates in this paper and some previous spectral methods. The main contributions in this paper are as follows:
\begin{itemize}
   \item[(i)] We summarize the idea of using separation condition of a standard network and sharp threshold of the ER random graph $G(n,p)$ to study consistent estimations of different spectral methods designed via eigen-decomposition or singular value decomposition of the adjacency matrix or its variants under different models that can degenerate to SBM under mild conditions as a four step criterion SCSTC. The separation condition is used to study the consistency of theoretical upper bound for spectral method, and the sharp threshold can be used to study the network sparsity. Theoretical results of upper bounds for different spectral methods can be compared by SCSTC. Using this criterion, a few inconsistent phenomenons of some previous works are found.
   \item [(ii)] Under MMSB and DCMM, we study the consistencies of the SPACL algorithm proposed in \cite{mao2020estimating} and its extended version using recent techniques on row-wise eigenvector deviation developed in  \cite{chen2020spectral,cape2019the}. Compared with the original results of \cite{mao2020estimating, MaoSVM}, our main theoretical results enjoy smaller error rates by lesser dependence on $K$ and $\mathrm{log}(n)$. Meanwhile, our main theoretical results have weaker requirements on the network sparsity and the lower bound of the smallest nonzero singular value of the population adjacency matrix. For detail, see Tables \ref{CompareMMSB} and \ref{CompareDCMM}.
   \item [(iii)] Our results for DCMM are consistent with those for MMSB when DCMM degenerates to MMSB under mild conditions. Using SCSTC, under mild conditions, our main theoretical results under DCMM are consistent with that of \cite{MixedSCORE}. This answers the question that the phenomenon that main results of \cite{mao2020estimating} and \cite{MaoSVM} do not match those of \cite{MixedSCORE} occurs due to the fact \cite{mao2020estimating} and \cite{MaoSVM}'s theoretical results of error rates are sub-optimal. We also find that our theoretical results (as well as that of \cite{MixedSCORE}) under both MMSB and DCMM match classical results on separation condition and sharp threshold. Using the bound of $\|A-\Omega\|$ instead of $\|A_{\mathrm{re}}-\Omega\|$ to establish upper bound of error rate under SBM in \cite{lei2015consistency}, the separation condition of a standard network obtained from \cite{lei2015consistency}'s error rate matches classical results, this answer the question that why separation condition obtained from error rate of \cite{MixedSCORE} does not match that obtained from error rate of \cite{lei2015consistency}. Using $\|A_{\mathrm{re}}-\Omega\|$ or $\|A-\Omega\|$ influences the row-wise eigenvector deviations in Theorem 3.1 of \cite{mao2020estimating} and Theorem I.3 of \cite{MaoSVM}, therefore whether using $\|A_{\mathrm{re}}-\Omega\|$ or $\|A-\Omega\|$ influences the separation conditions and sharp thresholds of \cite{MaoSVM,mao2020estimating}. For comparison, our bound on row-wise eigenvector deviation is obtained by using techniques developed in \cite{chen2020spectral,cape2019the} and that of \cite{MixedSCORE} is obtained by applying the modified Theorem 2.1 of \cite{abbe2020entrywise}, therefore whether using $\|A_{\mathrm{re}}-\Omega\|$ or $\|A-\Omega\|$ has no influences on separation conditions and sharp thresholds of ours and that of \cite{MixedSCORE}. For detail, see Tables \ref{SCST} and \ref{aSCST}.
\end{itemize}

The article is organized as follows. In Section \ref{intMMSB}, we give formal introduction to the mixed membership stochastic blockmodel and review the algorithm SPACL considered in this paper. The theoretical results of consistency for mixed membership stochastic blockmodel are presented and compared to related works in Section \ref{ConsistencyMMSB}. After delicate analysis, the separation condition and sharp threshold criterion is presented in Section \ref{MainSCSTC}. Based on an application of this criterion, improvement consistent estimation results for the extended version of SPACL under the degree corrected mixed membership model are provided in Section \ref{intDCMM}. Conclusion is given in Section \ref{Conclusion}.

\textbf{\textit{Notations.}}
We take the following general notations in this paper. Write $[m]:=\{1,2,\ldots,m\}$ for any positive integer $m$. For a vector $x$ and fixed $q>0$, $\|x\|_{q}$ denotes its $l_{q}$-norm. We drop the subscript if $q=2$ occasionally. For a matrix $M$, $M'$ denotes the transpose of the matrix $M$, $\|M\|$ denotes the spectral norm, $\|M\|_{F}$ denotes the Frobenius norm, $\|M\|_{2\rightarrow\infty}$ denotes the maximum $l_{2}$-norm of all the rows of $M$, and $\|M\|_{\infty}:=\mathrm{max}_{i}\sum_{j}|M(i,j)|$ denotes the maximum absolute row sum of $M$. Let $\mathrm{rank}(M)$ denote the rank of matrix $M$. Let $\sigma_{i}(M)$ be the $i$-th largest singular value of matrix $M$, $\lambda_{i}(M)$ denote the $i$-th largest eigenvalue of the matrix $M$ ordered by the magnitude, and $\kappa(M)$ denote the condition number of $M$. $M(i,:)$ and $M(:,j)$ denote the $i$-th row and the $j$-th column of matrix $M$, respectively. $M(S_{r},:)$ and $M(:,S_{c})$ denote the rows and columns in the index sets $S_{r}$ and $S_{c}$ of matrix $M$, respectively. For any matrix $M$, we simply use $Y=\mathrm{max}(0, M)$ to represent $Y_{ij}=\mathrm{max}(0, M_{ij})$ for any $i,j$. For any matrix $M\in\mathbb{R}^{m\times m}$, let $\mathrm{diag}(M)$ be the $m\times m$ diagonal matrix whose $i$-th diagonal entry is $M(i,i)$. $\mathbf{1}$ and $\mathbf{0}$ are  column vectors with all entries being ones and zeros, respectively.  $e_{i}$ is a column vector whose $i$-th entry is 1 while other entries are zero. In this paper, $C$ is a positive constant which may vary occasionally. $f(n)=O(g(n))$ means there exists a constant $c>0$ such that $|f(n)|\leq c|g(n)|$ holds for all sufficiently large $n$. $x\succeq y$ means there exists  a constant $c>0$ such that $|x|\geq c|y|$. $f(n)=o(g(n))$ indicates that $\frac{f(n)}{g(n)}\rightarrow0$ as $n\rightarrow\infty$.
\section{Mixed membership stochastic blockmodel}\label{intMMSB}
Let $A\in \{0,1\}^{n\times n}$ be a symmetric adjacency matrix such that $A(i,j)=1$ if there is an edge between node $i$ to node $j$, and $A(i,j)=0$ otherwise. Mixed membership stochastic blockmodel (MMSB) \cite{MMSB} for generating $A$ is as follows.
\begin{align}\label{MMSBmodel}
\Omega:=\rho \Pi \tilde{P}\Pi'~~~~~~~~~A(i,j)\sim\mathrm{Bernoulli}(\Omega(i,j))~~~~i,j\in[n],
\end{align}
where $\Pi\in\mathbb{R}^{n\times K}$ is called the membership matrix with $\Pi(i,k)\geq 0$ and $\sum_{k=1}^{K}\Pi(i,k)=1$ for $i\in[n]$ and $k\in[K]$, $\tilde{P}\in \mathbb{R}^{K\times K}$ is an nonnegative symmetric matrix with $\mathrm{max}_{k,l\in[K]}\tilde{P}(k,l)=1$ for model identifiability under MMSB, $\rho$ is called the sparsity parameter which controls the sparsity of the network, and $\Omega\in \mathbb{R}^{n\times n}$ is called the population adjacency matrix since $\mathbb{E}[A]=\Omega$. As mentioned in \cite{MixedSCORE,mao2020estimating}, $\sigma_{K}(\tilde{P})$ is a measure of the separation between communities, and we call it separation parameter in this paper. $\rho$ and $\sigma_{K}(\tilde{P})$ are two important model parameters directly related with the separation condition and sharp criterion, and they will be considered throughout this paper.
\begin{defin}
Call model (\ref{MMSBmodel}) the mixed membership stochastic blockmodel (MMSB), and denote it by $MMSB_{n}(K,\tilde{P},\Pi,\rho)$.
\end{defin}
Call node $i$ `pure' if $\Pi(i,:)$ is degenerate (i.e., one entry is 1, all others $K-1$ entries are 0) and `mixed' otherwise. By Theorems 2.1 and 2.2 \cite{mao2020estimating}, the following conditions are sufficient for the identifiability of MMSB, when $\rho \tilde{P}(k,l)\in[0,1]$ for all $k,l\in[K]$,
\begin{itemize}
  \item (I1) $\mathrm{rank}(\tilde{P})=K$.
  \item (I2) There is at least one pure node for each of the $K$ communities.
\end{itemize}
Unless specified, we treat conditions (I1) and (I2) as default from now on.

For $k\in[K]$, let $\mathcal{I}^{(k)}$ be the set of pure nodes in community $k$ such that $\mathcal{I}^{(k)}=\{i\in[n]: \Pi(i,k)=1\}$. For $k\in[K]$, select one node from $\mathcal{I}^{(k)}$ to construct the index set $\mathcal{I}$, i.e., $\mathcal{I}$ is the indices of nodes corresponding to $K$ pure nodes, one from each community. W.L.O.G., let $\Pi(\mathcal{I},:)=I_{K}$ where $I_{K}$ is the $K\times K$ identity matrix. Recall that $\mathrm{rank}(\Omega)=K$. Let $\Omega=U\Lambda U'$ be the compact eigen-decomposition of $\Omega$ such that $U\in\mathbb{R}^{n\times K}, \Lambda\in \mathbb{R}^{K\times K}$, and $U'U=I_{K}$. Lemma 2.1 \cite{mao2020estimating} gives that $U=\Pi U(\mathcal{I},:)$ and such form is called Ideal Simplex (IS for short)\cite{MixedSCORE, mao2020estimating} since all rows of $U$ form a $K$-simplex in $\mathbb{R}^{K}$ and the $K$ rows of $U(\mathcal{I},:)$ are the vertices of the $K$-simplex. Given $\Omega$ and $K$, as long as we know $U(\mathcal{I},:)$, we can exactly recover $\Pi$ by $\Pi=UU^{-1}(\mathcal{I},:)$ since $U(\mathcal{I},:)\in\mathbb{R}^{K\times K}$ is a full rank matrix.  As mentioned in \cite{MixedSCORE,mao2020estimating}, for such IS, the successive projection (SP) algorithm \cite{gillis2015semidefinite} (i.e., Algorithm \ref{alg:SP}) can be applied to $U$ with $K$ communities to exactly find the corner matrix $U(\mathcal{I},:)$.  For convenience, set $Z=UU^{-1}(\mathcal{I},:)$. Since $\Pi=Z$, we have $\Pi(i,:)=\frac{Z(i,:)}{\|Z(i,:)\|_{1}}$ for $i\in[n]$.

Based on the above analysis, we are now ready to give the ideal SPACL algorithm. Input $\Omega, K$. Output: $\Pi$.
\begin{itemize}
  \item Let $\Omega=U\Lambda U'$ be the top-$K$ eigen decomposition of $\Omega$ such that $U\in\mathbb{R}^{n\times K}, \Lambda\in\mathbb{R}^{K\times K},U'U=I$.
  \item Run SP algorithm on the rows of $U$ assuming that there are $K$ communities to obtain $\mathcal{I}$.
  \item Set $Z=UU^{-1}(\mathcal{I},:)$.
  \item Recover $\Pi$ by setting $\Pi(i,:)=\frac{Z(i,:)}{\|Z(i,:)\|_{1}}$ for $i\in[n]$.
\end{itemize}
With given $U$ and $K$, since SP algorithm returns $U(\mathcal{I},:)$, we see that the ideal SPACL exactly (for detail, see Appendix \ref{VHAandExactly}) returns $\Pi$.

Now, we review the SPACL algorithm of \cite{mao2020estimating}. Set $\tilde{A}=\hat{U}\hat{\Lambda}\hat{U}'$ be the top $K$ eigen-decomposition of $A$ such that $\hat{U}\in \mathbb{R}^{n\times K}, \hat{\Lambda}\in \mathbb{R}^{K\times K},\hat{U}'\hat{U}=I_{K}$, and $\hat{\Lambda}$ contains the top $K$ eigenvalues of $A$. For the real case, use $\hat{Z},\hat{\Pi}$ given in Algorithm \ref{alg:SPACL} to estimate $Z, \Pi$, respectively. Algorithm \ref{alg:SPACL} is the SPACL algorithm \cite{mao2020estimating} where we only care about the estimation of the membership matrix $\Pi$, and omit the estimation of $P$ and $\rho$. Meanwhile, Algorithm \ref{alg:SPACL} is a directly extension of the ideal SPACL algorithm from oracle case to real case, and we omit the prune step in the original SPACL algorithm of \cite{mao2020estimating}.
\begin{algorithm}
\caption{SPACL \cite{mao2020estimating}}
\label{alg:SPACL}
\begin{algorithmic}[1]
\Require The adjacency matrix $A\in \mathbb{R}^{n\times n}$ and the number of communities $K$.
\Ensure The estimated $n\times K$ membership matrix $\hat{\Pi}$.
\State Obtain $\tilde{A}=\hat{U}\hat{\Lambda}\hat{U}'$, the top $K$ eigen-decomposition of $A$.
\State Apply SP algorithm (i.e., Algorithm \ref{alg:SP}) on the rows of $\hat{U}$ assuming there are $K$ communities to obtain $\mathcal{\hat{I}}$, the index set returned by SP algorithm.
\State Set $\hat{Z}=\hat{U}\hat{U}^{-1}(\hat{\mathcal{I}},:)$. Then set $\hat{Z}=\mathrm{max}(0, \hat{Z})$.
\State Estimate $\Pi(i,:)$ by $\hat{\Pi}(i,:)=\hat{Z}(i,:)/\|\hat{Z}(i,:)\|_{1}, i\in[n]$.
\end{algorithmic}
\end{algorithm}
\section{Consistency under MMSB}\label{ConsistencyMMSB}
Our main result under MMSB provides an upper bound on estimation error of each node's membership in terms of several model parameters. Throughout this paper, $K$ is a known positive integer. Assume that
\begin{itemize}
  \item [(A1)]$\rho n\geq\mathrm{log}(n)$.
\end{itemize}
Assumption (A1) provides a requirement on the lower bound of the sparsity parameter $\rho$ such that it should be at least $\mathrm{log}(n)/n$. Then we have the following lemma.
\begin{lem}\label{BoundAOmega}
Under $MMSB_{n}(K,\tilde{P},\Pi,\rho)$, when Assumption (A1) holds, with probability at least $1-o(n^{-\alpha})$ for any $\alpha>0$, we have
\begin{align*}
\|A-\Omega\|\leq \frac{\alpha+1+\sqrt{(\alpha+1)(\alpha+19)}}{3}\sqrt{\rho n\mathrm{log}(n)}.
\end{align*}
\end{lem}
In Lemma \ref{BoundAOmega}, instead of simply using a constant $C_{\alpha}$ to denote $\frac{\alpha+1+\sqrt{(\alpha+1)(\alpha+19)}}{3}$, we keep the explicit form here.
\begin{rem}
When Assumption (A1) holds, the upper bound of $\|A-\Omega\|$ in Lemma \ref{BoundAOmega} is consistent with Corollary 6.5 in \cite{cai2015robust} since $\mathrm{Var}(A(i,j))\leq \rho$ under $MMSB_{n}(K,P,\Pi,\rho)$.
\end{rem}
Lemma \ref{BoundAOmega} is obtained via Theorem 1.4 (Bernstein inequality) in \cite{tropp2012user}. For comparison, \cite{mao2020estimating} applies Theorem 5.2 \cite{lei2015consistency} to bound $\|A-\Omega\|$ (see, for example, Eq (14) of \cite{mao2020estimating}) and obtains a bound as $C\sqrt{\rho n}$ for some $C>0$. However, $C\sqrt{\rho n}$ is the bound between a regularization of $A$ and $\Omega$ as stated in the proof of Theorem 5.2 \cite{lei2015consistency}, where such regularization of $A$ is obtained from $A$ with some constraints in Lemmas 4.1 and 4.2 of the supplement material \cite{lei2015consistency}. Meanwhile, Theorem 2 \cite{zhou2019analysis} also gives that the bound between a regularization of $A$ and $\Omega$ is $C\sqrt{\rho n}$ where such regularization of $A$ should also satisfy few constraints on $A$, see Theorem 2 \cite{zhou2019analysis} for detail. Instead of bounding the difference between a regularization of $A$ and $\Omega$, we are interested in bounding $\|A-\Omega\|$ by Bernstein inequality which has no constraints on $A$. For convenience, use $A_{\mathrm{re}}$ to denote the regularization of $A$ in this paper. Hence, $\|A_{\mathrm{re}}-\Omega\|\leq C\sqrt{\rho n}$with high probability, and this bound is model independent as shown by Theorem 5.2 \cite{lei2015consistency} and Theorem 2 \cite{zhou2019analysis} as long as $\rho\geq\mathrm{max}_{i,j}\Omega(i,j)$ (here, let $\Omega=\mathbb{E}[A]$ without considering models, a $\rho$ satisfying $\rho\geq\mathrm{max}_{i,j}\Omega(i,j)$ is also the sparsity parameter which controls the overall sparsity of a network).  Note that $A_{\mathrm{re}}$ is not $\tilde{A}$ where $\tilde{A}=\hat{U}\Lambda\hat{U}'$ is obtained by the top $K$ eigen-decomposition of $A$, while $A_{\mathrm{re}}$ is obtained by adding constrains on degrees of $A$, see Theorem 2 \cite{zhou2019analysis} for detail.

In \cite{MixedSCORE, MaoSVM, mao2020estimating}, main theoretical results for  their proposed membership estimating methods hinge on a row-wise deviation bound for the eigenvectors of the adjacency matrix whether under MMSB or DCMM. Different from the theoretical technique applied in Theorem 3.1 \cite{mao2020estimating} which provides sup-optimal dependencies on $\mathrm{log}(n)$ and $K$, and needs sub-optimal requirements on the sparsity parameter $\rho$ and the lower bound of $\sigma_{K}(\Omega)$, to obtain the row-wise deviation bound for the singular eigenvector of $\Omega$, we use Theorem 4.3.1 \cite{chen2020spectral} and Theorem 4.2 \cite{cape2019the}.
\begin{lem}\label{rowwiseerror}
(Row-wise eigenspace error) Under $MMSB_{n}(K,\tilde{P},\Pi,\rho)$, when Assumption (A1) holds, suppose $\sigma_{K}(\Omega)\geq C\sqrt{\rho n\mathrm{log}(n)}$, with probability at least $1-o(n^{-\alpha})$,
\begin{itemize}
  \item when we apply Theorem 4.2.1 of \cite{chen2020spectral}, we have
\begin{align*}
\|\hat{U}\hat{U}'-UU'\|_{2\rightarrow\infty}=O(\frac{\sqrt{K}(\kappa(\Omega)\sqrt{\frac{n}{K\lambda_{K}(\Pi'\Pi)}}+\sqrt{\mathrm{log}(n)})}{\sigma_{K}(\tilde{P})\sqrt{\rho}\lambda_{K}(\Pi'\Pi)}),
\end{align*}
  \item when we apply Theorem 4.2 of \cite{cape2019the}, we have
\begin{align*}
\|\hat{U}\hat{U}'-UU'\|_{2\rightarrow\infty}=O(\frac{\sqrt{n\mathrm{log}(n)}}{\sigma_{K}(\tilde{P})\sqrt{\rho}\lambda^{1.5}_{K}(\Pi'\Pi)}).
\end{align*}
\end{itemize}
\end{lem}
For convenience, set $\varpi=\|\hat{U}\hat{U}'-UU'\|_{2\rightarrow\infty}$, and let $\varpi_{1}, \varpi_{2}$ denote the upper bound in Lemma \ref{rowwiseerror} when applying Theorem 4.2.1 of \cite{chen2020spectral} and  Theorem 4.2 of \cite{cape2019the}, respectively.  Note that
When $\lambda_{K}(\Pi'\Pi)=O(\frac{n}{K})$, we have $\varpi_{1}=\varpi_{2}=O(\frac{K^{1.5}}{\sigma_{K}(\tilde{P})}\frac{1}{\sqrt{n}}\sqrt{\frac{\mathrm{log}(n)}{\rho n}})$, therefore we simply let $\varpi_{2}$ be the bound since its form is slightly simpler than $\varpi_{1}$.

Compared with Theorem 3.1 of \cite{mao2020estimating}, since we apply Theorem 4.2.1 of \cite{chen2020spectral} and  Theorem 4.2 of \cite{cape2019the} to obtain the bound of row-wise eigenspace error under MMSB, our bounds do not rely on $\mathrm{min}(K^{2},\kappa^{2}(\Omega))$ while Theorem 3.1 \cite{mao2020estimating} does. Meanwhile, our bound in Lemma \ref{rowwiseerror} is sharper with lesser dependence on $K$ and $\mathrm{log}(n)$, has weaker requirements on the lower bounds of $\sigma_{K}(\Omega), \lambda_{K}(\Pi'\Pi)$ and the sparsity parameter $\rho$. The details are given below:
\begin{itemize}
  \item We'd emphasize that the bound of Theorem 3.1 of \cite{mao2020estimating} should be $\|\hat{U}\hat{U}'-UU'\|_{2\rightarrow\infty}=O(\frac{\psi(\Omega)\sqrt{Kn}\mathrm{log}^{\xi}(n)}{\sigma_{K}(\tilde{P})\sqrt{\rho}\lambda^{1.5}_{K}(\Pi'\Pi)})$ instead of $\|\hat{U}\hat{U}'-UU'\|_{2\rightarrow\infty}=O(\frac{\psi(\Omega)\sqrt{Kn}}{\sigma_{K}(\tilde{P})\sqrt{\rho}\lambda^{1.5}_{K}(\Pi'\Pi)})$ for $\xi>1$ where the function $\psi$ is defined in Eq (7) of \cite{mao2020estimating}, and this is also pointed out by Table 2 of \cite{lei2019unified}. The reason is: in the proof part of Theorem 3.1 \cite{mao2020estimating}, from their step (iii) to step (iv), they should keep the term $\mathrm{log}^{\xi}(n)$ since this term is much larger than 1. And we can also find that bound in Theorem 3.1 \cite{mao2020estimating} should multiply $\mathrm{log}^{\xi}(n)$ from Theorem VI.1 \cite{mao2020estimating} directly. For comparison, this bound $O(\frac{\psi(\Omega)\sqrt{Kn}\mathrm{log}^{\xi}(n)}{\sigma_{K}(\tilde{P})\sqrt{\rho}\lambda^{1.5}_{K}(\Pi'\Pi)})$ is $K^{0.5}\mathrm{log}^{\xi-0.5}(n)$ times than our bound in Lemma \ref{rowwiseerror}. Meanwhile, by the proof of the bound in Theorem 3.1 of \cite{mao2020estimating}, we see that the bound depends on the upper bound of $\|A-\Omega\|$, and \cite{mao2020estimating} applies Theorem 5.2 of \cite{lei2015consistency} such that $\|A_{\mathrm{re}}-\Omega\|\leq C\sqrt{\rho n}$ with high probability. Since $C\sqrt{\rho n}$ is the upper bound of the difference between a regularization of $A$ and $\Omega$. Therefore, if we are only interested in bounding $\|A-\Omega\|$ instead of $\|A_{\mathrm{re}}-\Omega\|$, the upper bound of Theorem 3.1 \cite{mao2020estimating} should be $O(\frac{\psi(\Omega)\sqrt{Kn}\mathrm{log}^{\xi+0.5}(n)}{\sigma_{K}(\tilde{P})\lambda^{1.5}_{K}(\Pi'\Pi)})$, which is at least $K^{0.5}\mathrm{log}^{\xi}(n)$ times than our bound in Lemma \ref{rowwiseerror}. Furthermore, the upper bound of the row-wise eigenspace error in Lemma \ref{rowwiseerror} does not rely on the upper bound of $\|A-\Omega\|$ as long as $\sigma_{K}(\Omega)\geq C\sqrt{\rho n\mathrm{log}(n)}$ holds. Therefore, whether using $\|A_{\mathrm{re}}-\Omega\|\leq C\sqrt{\rho n}$ or $\|A-\Omega\|\leq C\sqrt{\rho n\mathrm{log}(n)}$ does not change the bound in Lemma \ref{rowwiseerror}.
  \item  Our Lemma \ref{rowwiseerror} requires $\sigma_{K}(\Omega)\geq C\sqrt{\rho n\mathrm{log}(n)}$, while Theorem 3.1 \cite{mao2020estimating} requires $\sigma_{K}(\Omega)\geq 4\sqrt{\rho n}\mathrm{log}^{\xi}(n)$ by their Assumption 3.1. Therefore, our Lemma \ref{rowwiseerror} has weaker requirement on the lower bound of $\sigma_{K}(\Omega)$ than that of Theorem 3.1 \cite{mao2020estimating}. Meanwhile, Theorem 3.1 \cite{mao2020estimating} requires $\lambda_{K}(\Pi'\Pi)\geq \frac{1}{\rho}$ while our Lemma \ref{rowwiseerror} has no lower bound requirement on $\lambda_{K}(\Pi'\Pi)$ as long as it is positive.
  \item Since $\|\Omega\|=\|\rho \Pi \tilde{P}\Pi'\|\leq C\rho n$ by basic algebra, the lower bound requirement on $\sigma_{K}(\Omega)$ in Assumption 3.1 of \cite{mao2020estimating} gives that $4\sqrt{\rho n}\mathrm{\log}^{\xi}(n)\leq \sigma_{K}(\Omega)\leq \|\Omega\|\leq C\rho n$, which suggests that Theorem 3.1 \cite{mao2020estimating} requires $\rho n\geq C\mathrm{log}^{2\xi}(n)$, and this also matches with the requirement on $\rho n$ in Theorem VI.1 of \cite{mao2020estimating} (and this is also pointed out by Table 1 of \cite{lei2019unified}). For comparison, our requirement on sparsity given in Assumption (A1) is $\rho n\geq \mathrm{log}(n)$, which is weaker than  $\rho n\geq C\mathrm{log}^{2\xi}(n)$. Similarly, in our Lemma \ref{rowwiseerror}, the requirement $\sigma_{K}(\Omega)\geq C\sqrt{\rho n\mathrm{log}(n)}$ gives $C\sqrt{\rho n\mathrm{log}(n)}\leq \sigma_{K}(\Omega)\leq \|\Omega\|\leq C\rho n$, thus we have $\mathrm{log}(n)\leq C\rho n$ which is consistent with our Assumption (A1).
\end{itemize}
If we further assume that $K=O(1), \lambda_{K}(\Pi'\Pi)=O(\frac{n}{K})$ and $\sigma_{K}(\tilde{P})=O(1)$, the row-wise eigenspace error is of order $\frac{1}{\sqrt{n}}\sqrt{\frac{\mathrm{log}(n)}{\rho n}}$, which is consistent with the row-wise eigenvector deviation of \cite{lei2019unified}'s result shown in their Table 2. Next theorem gives theoretical bounds on the estimations of memberships under MMSB.
\begin{thm}\label{Main}
Under $MMSB_{n}(K,\tilde{P},\Pi,\rho)$, suppose conditions in Lemma \ref{rowwiseerror} hold, there exists a permutation matrix $\mathcal{P}\in\mathbb{R}^{K\times K}$ such that with probability at least $1-o(n^{-\alpha})$, we have
\begin{align*}
\mathrm{max}_{i\in[n]}\|e'_{i}(\hat{\Pi}-\Pi\mathcal{P})\|_{1}=O(\varpi K\kappa(\Pi'\Pi)\sqrt{\lambda_{1}(\Pi'\Pi)}).
\end{align*}
\end{thm}
\begin{rem}
(Comparison to Theorem 3.2 \cite{mao2020estimating}) Consider a special case by setting $\kappa(\Pi'\Pi)=O(1)$, i.e., $\lambda_{K}(\Pi'\Pi)=O(\frac{n}{K})$ and $\lambda_{1}(\Pi'\Pi)=O(\frac{n}{K})$. We focus on comparing the dependencies on $K$ in bounds of our Theorem \ref{Main} and Theorem 3.2 \cite{mao2020estimating}. Under this case, the bound of our Theorem \ref{Main} is proportional to $K^{2}$ by basic algebra; since $\mathrm{min}(K^{2},\kappa^{2}(\Omega))=\mathrm{min}(K^{2},O(1))=O(1)$ and the bound in Theorem 3.2 \cite{mao2020estimating} should multiply $\sqrt{K}$ because (in \cite{mao2020estimating}'s language) $\|\hat{V}^{-1}_{p}\|_{F}\leq\frac{\sqrt{K}}{\sigma_{K}(\hat{V}_{p})}$ instead of $\|\hat{V}^{-1}_{p}\|_{F}=\frac{1}{\lambda_{K}(\hat{V}_{p})}$ in Eq (45) \cite{mao2020estimating}, the power of $K$ is 2 by checking the bound of Theorem 3.2 \cite{mao2020estimating}. Meanwhile, note that our bound in Theorem \ref{MainDCMM} is $l_{1}$ bound while bound in Theorem 3.2 \cite{mao2020estimating} is $l_{2}$ bound, when we translate the $l_{2}$ bound of Theorem 3.2 \cite{mao2020estimating} into $l_{1}$ bound, the power of $K$ is 2.5 for Theorem 3.2 \cite{mao2020estimating}. Hence, our bound in Theorem \ref{Main} has less dependence on $K$ than that of Theorem 3.2 \cite{mao2020estimating}, and this is also consistent with the first bullet given after Lemma \ref{rowwiseerror}.
\end{rem}

\begin{table}
\scriptsize
\centering
\caption{Comparison of error rates between our Theorem \ref{Main} and Theorem 3.2 \cite{mao2020estimating} under $MMSB_{n}(K,\tilde{P},\Pi,\rho)$. The dependence on $K$ is obtained when $\kappa(\Pi'\Pi)=O(1)$. For comparison, we have adjusted the $l_{2}$ error rates of Theorem 3.2 \cite{mao2020estimating} into $l_{1}$ error rates. Note that as analyzed in the first bullet given after Lemma \ref{rowwiseerror}, whether using $\|A-\Omega\|\leq C\sqrt{\rho n\mathrm{log}(n)}$ or $\|A_{\mathrm{re}}-\Omega\|\leq C\sqrt{\rho n}$ does not change our $\varpi$, and has no influence on bound in Theorem \ref{Main}. For \cite{mao2020estimating}: using $\|A_{\mathrm{re}}-\Omega\|\sqrt{\rho n}$, the power of $\mathrm{log}(n)$ in their Theorem 3.2 is $\xi$; using $\|A-\Omega\|\sqrt{\rho n\mathrm{log}(n)}$, the power of $\mathrm{log}(n)$ in their Theorem 3.2 is $\xi+0.5$.}
\label{CompareMMSB}
\begin{tabular}{c|ccc|cc|cccc}
\toprule
&$\rho n$&$\sigma_{K}(\Omega)$&$\lambda_{K}(\Pi'\Pi)$&Dependence on $K$&Dependence on $\mathrm{log}(n)$\\
\midrule
Ours&$\geq\mathrm{log}(n)$&$\succeq \sqrt{\rho n\mathrm{log}(n)}$&$>0$&$K^{2}$&$\mathrm{log}^{0.5}(n)$\\
\hline
\cite{mao2020estimating}&$\geq\mathrm{log}^{2\xi}(n)$&$\succeq \sqrt{\rho n}\mathrm{log}^{\xi}(n)$&$\geq1/\rho$&$K^{2.5}$&$\mathrm{log}^{\xi}(n)$\\
\bottomrule
\end{tabular}
\end{table}
Table \ref{CompareMMSB} summaries the necessary conditions and dependence on model parameters of rates in Theorem \ref{Main} and Theorem 3.2 \cite{mao2020estimating} for comparison. The following corollary is obtained by adding conditions on model parameters similar as Corollary 3.1 in \cite{mao2020estimating}.
\begin{cor}\label{AddConditions}
Under $MMSB_{n}(K,\tilde{P},\Pi,\rho)$, when conditions of Lemma \ref{rowwiseerror} hold, suppose $\lambda_{K}(\Pi'\Pi)=O(\frac{n}{K})$ and $K=O(1)$, with probability at least $1-o(n^{-\alpha})$, we have
\begin{align*}
\mathrm{max}_{i\in[n]}\|e'_{i}(\hat{\Pi}-\Pi\mathcal{P})\|_{1}=O(\frac{1}{\sigma_{K}(\tilde{P})}\sqrt{\frac{\mathrm{log}(n)}{\rho n}}).
\end{align*}
\end{cor}
\begin{rem}
Consider a special case in Corollary \ref{AddConditions} by setting $\sigma_{K}(\tilde{P})$ as a constant, we see that the error bound $O(\sqrt{\frac{\mathrm{log}(n)}{\rho n}})$ in Corollary \ref{AddConditions} is directly related with Assumption (A1), and for consistent estimation, $\rho$ should grow faster than $\frac{\mathrm{log}(n)}{n}$.
\end{rem}
\begin{rem}
Under the setting of Corollary \ref{AddConditions}, the requirement $\sigma_{K}(\Omega)\geq C\sqrt{\rho n\mathrm{log}(n)}$ in Lemma \ref{rowwiseerror} holds naturally. By Lemma II.4 \cite{mao2020estimating}, we know that $\sigma_{K}(\Omega)\geq \rho \sigma_{K}(\tilde{P})\lambda_{K}(\Pi'\Pi)=C\rho n\sigma_{K}(\tilde{P})$. To make the requirement $\sigma_{K}(\Omega)\geq C\sqrt{\rho n\mathrm{log}(n)}$ always hold, we just need $C\rho n\sigma_{K}(\tilde{P})\geq C\sqrt{\rho n\mathrm{log}(n)}$, which gives that $\sigma_{K}(\tilde{P})\geq C\sqrt{\frac{\mathrm{log}(n)}{\rho n}}$, and it just matches with the requirement of consistent estimation of memberships in Corollary \ref{AddConditions}.
\end{rem}
\begin{rem}\label{ComMaoSeparation}
(Comparison to Theorem 3.2 \cite{mao2020estimating}) When $K=O(1)$ and $\lambda_{K}(\Pi'\Pi)=O(\frac{n}{K})$, by the first bullet in the analysis given after Lemma \ref{rowwiseerror}, the row-wise eigenspace error of Theorem 3.1 \cite{,mao2020estimating} is $O(\frac{\mathrm{log}^{\xi}(n)}{\sigma_{K}(\tilde{P})\sqrt{\rho}n})$, and it gives that their error bound on estimation membership given in their Eq (3) is $O(\frac{\mathrm{log}^{\xi}(n)}{\sigma_{K}(\tilde{P})\sqrt{\rho n}})$, which is $\mathrm{log}^{\xi-0.5}(n)$ times of the bound in our Lemma \ref{AddConditions}.
\end{rem}
\begin{rem}
(Comparison to Theorem 2.2 \cite{MixedSCORE}) Replacing the $\Theta$ in \cite{MixedSCORE} by $\Theta=\sqrt{\rho}I$, their DCMM model  degenerates to MMSB. Then their conditions in Theorem 2.2 are our Assumption (A1) and $\lambda_{K}(\Pi'\Pi)=O(\frac{n}{K})$ for MMSB. When $K=O(1)$, the error bound in Theorem 2.2 in \cite{MixedSCORE} is
$O(\frac{1}{\sigma_{K}(\tilde{P})}\sqrt{\frac{\mathrm{log}(n)}{\rho n}})$, which is consistent with ours.
\end{rem}
\section{Separation condition and sharp threshold criterion}\label{MainSCSTC}
After obtaining the Corollary \ref{AddConditions} under MMSB, now we are ready to give our criterion after introducing separation condition of a standard network and sharp threshold of ER random graph $G(n,p)$ in this section.

\textbf{\textit{Separation condition.}} Consider a standard network by setting $\tilde{P}=\omega I_{K}+(1-\omega)\mathbf{1}\mathbf{1}'$ for $\omega\in(0,1]$ (we have $\sigma_{K}(\tilde{P})=\omega$) under the settings of Corollary \ref{AddConditions}. Note that we have $\Omega=\Pi \rho \tilde{P}\Pi'=\Pi'P\Pi'$, where $P=\rho \tilde{P}$ and $P$ is the probability matrix. For convenience, set $p_{\mathrm{in}}=\rho, p_{\mathrm{out}}=\rho(1-\omega)$ (note that we have $p_{\mathrm{in}}>p_{\mathrm{out}}$ when $\omega\in(0,1]$.). (a) Under such $\tilde{P}$ and settings in Corollary \ref{AddConditions}, since the error rate is $O(\frac{1}{\omega}\sqrt{\frac{\mathrm{log}(n)}{\rho n}})$, to obtain consistency estimation, $\omega$ should grow faster than $\sqrt{\frac{\mathrm{log}(n)}{\rho n}}$. Therefore, the separation condition  $\frac{|p_{\mathrm{in}}-p_{\mathrm{out}}|}{\sqrt{p_{\mathrm{in}}}}=\omega\sqrt{\rho}$ (also known as relative edge probability gap) should grow faster than $\sqrt{\frac{\mathrm{log}(n)}{n}}$ which is consistent with Corollary 1 of \cite{mcsherry2001spectral} and Eq (17) of \cite{joseph2016impact}. (b) Undoubtedly, this separation condition is consistent with that of \cite{MixedSCORE},  since Theorem 2.2 \citep{MixedSCORE} shares the same error rate $O(\frac{1}{\sigma_{K}(\tilde{P})}\sqrt{\frac{\mathrm{log}(n)}{\rho n}})$ for this standard network. (c) Furthermore, by Remark \ref{ComMaoSeparation}, using $\|A_{\mathrm{re}}-\Omega\|\leq C\sqrt{\rho n}$, we know that \cite{mao2020estimating}'s Eq (3) is $O(\frac{\mathrm{log}^{\xi}(n)}{\sigma_{K}(\tilde{P})\sqrt{\rho n}})$, follow similar analysis, we see that the separation condition for \cite{mao2020estimating} is $\frac{\mathrm{log}^{\xi}(n)}{\sqrt{n}}$, which is sub-optimal compared with ours. Using $\|A-\Omega\|\leq C\sqrt{\rho n\mathrm{log}(n)}$, \cite{mao2020estimating}'s Eq (3) is $O(\frac{\mathrm{log}^{\xi+0.5}(n)}{\sigma_{K}(\tilde{P})\sqrt{\rho n}})$, follow similar analysis, we see that the separation condition for \cite{mao2020estimating} now is $\frac{\mathrm{log}^{\xi+0.5}(n)}{\sqrt{n}}$. (d) For comparison, the error bound of Corollary 3.2 \cite{lei2015consistency} built under SBM for community detection is $O(\frac{1}{\sigma^{2}_{K}(\tilde{P})\rho n})$ when $\kappa(\Pi'\Pi)=O(1)$ and $K=O(1)$. Then follow similar analysis, we see that the separation condition for \cite{lei2015consistency} should grow faster than $\frac{1}{\sqrt{n}}$. However, as we analyzed in the first bullet given after lemma \ref{rowwiseerror}, \cite{lei2015consistency} applies $\|A_{re}-\Omega\|\leq C\sqrt{\rho n}$ to build their consistency results. Instead, we apply $\|A-\Omega\|\leq C\sqrt{\rho n\mathrm{log}(n)}$ to built \cite{lei2015consistency}'s theoretical results, the error bound of Corollary 3.2 \cite{lei2015consistency} is $O(\frac{\mathrm{log}(n)}{\sigma^{2}_{K}(\tilde{P})\rho n})$, which returns same separation condition as ours Lemma \ref{AddConditions} and \cite{MixedSCORE}'s Theorem 2.2 now.  As analyzed in the first bullet given after Lemma \ref{rowwiseerror}, whether using $\|A-\Omega\|\leq C\sqrt{\rho n\mathrm{log}(n)}$ or $\|A_{\mathrm{re}}-\Omega\|\leq C\sqrt{\rho n}$ does not change our error rates. By carefully analyzing the proof of 2.1 of \cite{MixedSCORE}, we see that whether using $\|A-\Omega\|\leq C\sqrt{\rho n\mathrm{log}(n)}$ or $\|A_{\mathrm{re}}-\Omega\|\leq C\sqrt{\rho n}$ also does not change their row-wise large deviation, hence it does not influence their upper bound of error rate for their Mixed-SCORE.

Similar as \cite{abbe2016exact}, set $p_{\mathrm{in}}=\alpha_{\mathrm{in}}\frac{\mathrm{log}(n)}{n}, p_{\mathrm{out}}=\alpha_{\mathrm{out}}\frac{\mathrm{log}(n)}{n}$ (note that $\alpha_{\mathrm{in}}>\alpha_{\mathrm{out}}$ when $\omega\in(0,1]$.), we can obtain an alternative version of separation condition $\frac{\alpha_{\mathrm{in}}-\alpha_{\mathrm{out}}}{\sqrt{\alpha_{\mathrm{in}}}}$ such that if $\frac{\alpha_{\mathrm{in}}-\alpha_{\mathrm{out}}}{\sqrt{\alpha_{\mathrm{in}}}}\gg1$,  recovering the memberships for with high probability is possible, and vice verse. In this paper, we call $\frac{|\alpha_{\mathrm{in}}-\alpha_{\mathrm{out}}|}{\sqrt{\alpha_{\mathrm{in}}}}$ as alternative separation condition. Now we provide the details. Since $p_{\mathrm{in}}=\alpha_{\mathrm{in}}\frac{\mathrm{log}(n)}{n}=\rho$ and $p_{\mathrm{out}}=\alpha_{\mathrm{out}}\frac{\mathrm{log}(n)}{n}=\rho(1-\omega)$, we have $p_{\mathrm{in}}-p_{\mathrm{out}}=(\alpha_{\mathrm{in}}-\alpha_{\mathrm{out}})\frac{\mathrm{log}(n)}{n}=\rho\omega$ and $\rho=\alpha_{\mathrm{in}}\frac{\mathrm{log}(n)}{n}$. (a') Under such $P$ and settings in our Corollary \ref{AddConditions}, since the error rate is $O(\frac{1}{\omega}\sqrt{\frac{\mathrm{log}(n)}{\rho n}})$, for consistent estimation, $\omega$ should grow faster than $\sqrt{\frac{\mathrm{log}(n)}{\rho n}}$. Hence, $\rho\omega$ should grow faster than $\sqrt{\frac{\rho\mathrm{log}(n)}{n}}$. Since $\rho\omega=(\alpha_{\mathrm{in}}-\alpha_{\mathrm{out}})\frac{\mathrm{log}(n)}{n}$, we have $(\alpha_{\mathrm{in}}-\alpha_{\mathrm{out}})\frac{\mathrm{log}(n)}{n}\gg\sqrt{\frac{\rho\mathrm{log}(n)}{n}}$, and it gives  $\frac{\alpha_{\mathrm{in}}-\alpha_{\mathrm{out}}}{\sqrt{\alpha_{\mathrm{in}}}}\gg 1$ since $\rho=\alpha_{\mathrm{in}}\frac{\mathrm{log}(n)}{n}$. (b') Since Theorem 2.2 \citep{MixedSCORE} shares the same error rate $O(\frac{1}{\sigma_{K}(\tilde{P})}\sqrt{\frac{\mathrm{log}(n)}{\rho n}})$, \cite{MixedSCORE} enjoys the same alternative separation condition. (c') Using $\|A_{\mathrm{re}}-\Omega\|\leq C\sqrt{\rho n}$, \cite{mao2020estimating}'s Eq (3) is $O(\frac{\mathrm{log}^{\xi}(n)}{\sigma_{K}(\tilde{P})\sqrt{\rho n}})$, follow similar analysis, the alternative separation condition for \cite{mao2020estimating} is $\frac{\alpha_{\mathrm{in}}-\alpha_{\mathrm{out}}}{\sqrt{\alpha_{\mathrm{in}}}}\gg \mathrm{log}^{\xi-0.5}(n)$. Using $\|A-\Omega\|\leq C\sqrt{\rho n\mathrm{log}(n)}$, \cite{mao2020estimating}'s Eq (3) is $O(\frac{\mathrm{log}^{\xi+0.5}(n)}{\sigma_{K}(\tilde{P})\sqrt{\rho n}})$, then the  alternative separation condition for \cite{mao2020estimating} now is $\frac{\alpha_{\mathrm{in}}-\alpha_{\mathrm{out}}}{\sqrt{\alpha_{\mathrm{in}}}}\gg \mathrm{log}^{\xi}(n)$. (d') Using $\|A_{re}-\Omega\|\leq C\sqrt{\rho n}$, error bound of Corollary 3.2 \cite{lei2015consistency} built under SBM is $O(\frac{1}{\sigma^{2}_{K}(\tilde{P})\rho n})$ when $\kappa(\Pi'\Pi)=O(1)$ and $K=O(1)$. Follow similar analysis, the alternative separation condition for \cite{lei2015consistency} is $\frac{\alpha_{\mathrm{in}}-\alpha_{\mathrm{out}}}{\sqrt{\alpha_{\mathrm{in}}}}\gg \sqrt{\frac{1}{\mathrm{log}(n)}}$. Using $\|A-\Omega\|\leq C\sqrt{\rho n\mathrm{log}(n)}$ , the error bound of Corollary 3.2 \cite{lei2015consistency} is $O(\frac{\mathrm{log}(n)}{\sigma^{2}_{K}(\tilde{P})\rho n})$, which returns same alternative separation condition as ours Lemma \ref{AddConditions} and \cite{MixedSCORE}'s Theorem 2.2 now.
\begin{rem}
A large body of literature in statistics and computer science \cite{abbe2015community,abbe2016exact,hajek2016achieving,agarwal2017multisection,bandeira2018random} has focused on detecting communities of network with 2 equal size clusters under SBM, and finds that recovering the communities is possible when
$\sqrt{\alpha_{\mathrm{in}}}-\sqrt{\alpha_{\mathrm{out}}}>\sqrt{2}$. This threshold can be achieved by semidefinite relaxations \cite{abbe2016exact,hajek2016achieving,agarwal2017multisection,bandeira2018random} and spectral methods with local refinements \cite{abbe2015community,gao2017achieving} . For our alternative separation condition $\frac{\alpha_{\mathrm{in}}-\alpha_{\mathrm{out}}}{\sqrt{\alpha_{\mathrm{in}}}}\gg1$, though it has more rougher form than that of $\sqrt{\alpha_{\mathrm{in}}}-\sqrt{\alpha_{\mathrm{out}}}>\sqrt{2}$, it is useful in studying optimality of estimation and comparing the error rates of different spectral methods under different models, as shown in Table \ref{aSCST}.
\end{rem}
\textbf{\textit{Sharp threshold.}} Consider the  Erd\"os-R\'enyi (ER) random graph $G(n,p)$ \citep{erdos2011on}. To construct the ER random graph $G(n,p)$, set $K=1$ and $\Pi$ is an $n\times 1$ vector with all entries being ones. Since $K=1$ and the maximum entry of $\tilde{P}$ is assumed to be 1, we have $\tilde{P}=1$ in $G(n,p)$ and hence $\sigma_{K}(\tilde{P})=1$. Then we have $\Omega=\Pi \rho \tilde{P}\Pi'=\Pi\rho\Pi'=\Pi p\Pi'$, i.e, $p=\rho$. Since the error rate is $O(\frac{1}{\sigma_{K}(\tilde{P})}\sqrt{\frac{\mathrm{log}(n)}{\rho n}})=O(\sqrt{\frac{\mathrm{log}(n)}{pn}})$, for consistent estimation, we see that $p$ should grow faster than $\frac{\mathrm{log}(n)}{n}$, which is just the sharp threshold in \citep{erdos2011on}, Theorem 4.6 \citep{blum2020foundations}, strongly consistent of \cite{zhao2012consistency}, and the first bullet in Section 2.5 \citep{abbe2017community} (call the lower bound requirement of $p$ for ER random graph to enjoy consistent estimation as sharp threshold). Since the sharp threshold is obtained when $K=1$ which means a connected ER random graph $G(n,p)$, and this is also consistent with the connectivity in Table 2 of \cite{abbe2016exact}. Meanwhile, since our Assumption (A1) requires $\rho n\geq\mathrm{log}(n)$, it gives that $p$ should grow faster than $\frac{\mathrm{log}(n)}{n}$ since $p=\rho$ under $G(n,p)$, which is consistent with the sharp threshold. Since \cite{MixedSCORE}'s Theorem 2.2 enjoys same error rate as ours under the settings in Corollary \ref{AddConditions}, \cite{MixedSCORE} also reaches the sharp threshold as $\frac{\mathrm{log}(n)}{n}$.  Furthermore, Remark \ref{ComMaoSeparation} says that bound for error rate in Eq (3)  \cite{mao2020estimating} should be $O(\frac{\mathrm{log}^{\xi}(n)}{\sigma_{K}(\tilde{P})\sqrt{\rho n}})$ when using $\|A_{\mathrm{re}}-\Omega\|\leq C\sqrt{\rho n}$, follow similar analysis, we see that the sharp threshold for \cite{mao2020estimating} is $\frac{\mathrm{log}^{2\xi}(n)}{n}$, which is sub-optimal compared with ours. When using $\|A-\Omega\|\leq C\sqrt{\rho n\mathrm{log}(n)}$, the sharp threshold for \cite{mao2020estimating} is $\frac{\mathrm{log}^{2\xi+1}(n)}{n}$. Similarly, the error bound of Corollary 3.2 \cite{lei2015consistency} is $O(\frac{1}{\sigma^{2}_{K}(\tilde{P})\rho n})\equiv O(\frac{1}{pn})$ under ER $G(n,p)$ since $p=\rho, \sigma_{K}(\tilde{P})=1$ and $K=1$. Hence, the sharp threshold obtained from the theoretical upper bound for error rates of \cite{lei2015consistency} is $\frac{1}{n}$, which does not match classical result. Instead, we apply $\|A-\Omega\|\leq C\sqrt{\rho n\mathrm{log}(n)}$ with high probability to build \cite{lei2015consistency}'s theoretical results, the error bound of Corollary 3.2 \cite{lei2015consistency} is $O(\frac{\mathrm{log}(n)}{p n})$, which returns the classical sharp threshold $\frac{\mathrm{log}(n)}{n}$ now.

Table \ref{SCST} summaries the comparisons of separation condition and sharp threshold. Table \ref{aSCST} records the respective alternative separation condition. The delicate analysis given above supports our statement that the separation condition of a standard network and sharp threshold of ER random graph $G(n,p)$ can be seen as unified criterions to compare theoretical results of spectral methods under different models. To conclude the above analysis, here we summarize the main steps to apply the separation condition and shrap threshold criterion (SCSTC for short) to check the consistency of theoretical results or compare results of spectral methods under different models, where spectral methods means methods developed based on the application of the eigenvectors or singular vectors of the adjacency matrix or its variants for community detection. The four-stage SCSTC is given below:
\begin{itemize}
  \item [{$step_{1}$}] Check whether the theoretical upper bound of error rate contains $\sigma_{K}(\tilde{P})$, where the separation parameter $\sigma_{K}(\tilde{P})$ always appears when considering the lower bound of $\sigma_{K}(\Omega)$. If it contains $\sigma_{K}(\tilde{P})$, move to the next step. Otherwise, it suggests possible improvements for the consistency by considering $\sigma_{K}(\tilde{P})$ in the proofs.
  \item [{$step_{2}$}] Let the number of communities as $O(1)$ and the network degenerate to standard network whose numbers of nodes in each community are in the same order and can been seen as $O(\frac{n}{K})$. Let the model degenerate to SBM and then obtain the newly theoretical upper bound of error rate. Note that if the model does consider degree heterogeneity, the sparsity parameter $\rho$ should be considered in the theoretical upper bound of error rate in $step_{1}$. If the model considers degree heterogeneity, when it degenerates to SBM, $\rho$ appears at this step. Meanwhile, if $\rho$ is not contained in the error rate of $step_{1}$ when the model does not consider degree heterogeneity, it suggests possible improvements by considering $\rho$.
  \item [{$step_{3}$}] Let $\tilde{P}=\omega I_{K}+(1-\omega)\mathbf{1}\mathbf{1}'$ for $0<\omega<1$ (note that $\sigma_{K}(\tilde{P})=\omega$), set $P=\rho \tilde{P}$ as the probability matrix when the model degenerates to SBM. Next compute the lower bound requirement of $\omega$ for consistency estimation through analyzing the newly bound obtained in the last step (note that, we have $p_{\mathrm{in}}=\rho, p_{\mathrm{out}}=\rho(1-\omega)$  and $p_{\mathrm{in}}-p_{\mathrm{out}}=\rho\omega$ under the above settings of SCSTC). Compute the separation condition $\frac{|p_{\mathrm{in}}-p_{\mathrm{out}}|}{\sqrt{p_{\mathrm{in}}}}=\omega\sqrt{\rho}$ using the lower bound requirement for $\omega$. The sharp threshold for ER random graph $G(n,p)$ is obtained from the lower bound requirement on $\rho$ for consistency estimation under the setting that $K=1, \sigma_{K}(\tilde{P})=1$ and $p=\rho$.
  \item [{$step_{4}$}] Compare the separation condition and sharp threshold obtained in the last step with the classical results in Corollary 1 of \cite{mcsherry2001spectral} and the first bullet in Section 2.5 \citep{abbe2017community} (or our results given in Table \ref{SCST}), respectively. If the sharp threshold $\gg\frac{\mathrm{log}(n)}{n}$ or separation condition $\gg\sqrt{\frac{\mathrm{log}(n)}{n}}$, then this leaves improvements on the network sparsity or theoretical upper bound of error rate. If the sharp threshold is $\frac{\mathrm{log}(n)}{n}$ and the separation condition is $\sqrt{\frac{\mathrm{log}(n)}{n}}$, the optimality of theoretical results on both error rates and requirement of network sparsity is guaranteed. Finally, if the sharp threshold $\ll\frac{\mathrm{log}(n)}{n}$ or separation condition $\ll\sqrt{\frac{\mathrm{log}(n)}{n}}$, this suggests that the theoretical result is obtained based on $\|A_{\mathrm{re}}-\Omega\|$ instead of $\|A-\Omega\|$.
\end{itemize}
Below remarks gives some explanations on the four steps of SCSTC.
\begin{rem}
\begin{itemize}
  \item In $step_{1}$, we give a few examples. When applying SCSTC to the main results of \cite{RSC,DISIM,OCCAM}, we stop at $step_{1}$ as analyzed in Remark \ref{FailSC}, suggesting possible improvements by considering $\sigma_{K}(\tilde{P})$ for these works. Meanwhile, for theoretical result without considering $\sigma_{K}(\tilde{P})$, we can also move to $step_{2}$ to obtain the newly theoretical upper bound of error rate which is related with $\rho$ and $n$. Discussions on theoretical upper bounds of error rates of \cite{SCORE,DSCORE} given in Remark \ref{FailSC} are examples of this case.
  \item In $step_{2}$, letting $K=O(1)$ and the network be balanced can always simplify the theoretical upper bound of error rate, as shown by our Corollaries \ref{AddConditions} and \ref{AddConditionsDCMM}. Here,  we provide some examples about how to make a model degenerate to SBM. For $MMSB_{n}(K,\tilde{P},\Pi,\rho)$ in this paper, when all nodes are pure, MMSB degenerates to SBM; for the $DCMM_{n}(K,\tilde{P},\Pi,\Theta)$ model introduced in Section \ref{intDCMM} or DCSBM considered in \cite{RSC,SCORE, lei2015consistency}, setting $\Theta=\sqrt{\rho}I$ makes DCMM and DCSBM degenerates to SBM when all nodes are pure; similar for the ScBM and DCScBM considered in \cite{DISIM,DSCORE,zhou2019analysis,qingDiSC}, the OCCAM model of \cite{OCCAM}, the stochastic blockmodel with overlap proposed in \cite{kaufmann2017a}, the BiMMSB model in \cite{qingBiMMSB}, the DiDCMM model in \cite{qingDiDCMM}, the extensions of SBM and DCSBM for hypergraph networks considered in \cite{ghoshdastidar2014consistency,ke2019community,cole2020exact}, and so forth.  Meanwhile, when we say that a model degenerates to SBM, we means that the model can degenerates to a special case of SBM and do not mean that it can exactly degenerate to SBM. For example, the $DCMM_{n}(K,\tilde{P},\Pi,\Theta)$ considered in next Section \ref{intDCMM}, it requires $P$ has unit diagonal entries, which suggests that all diagonal entries of $\tilde{P}$ considered in $step_{3}$ should be the same while SBM can model network whose probability matrix has various entries.
  \item In $step_{3}$, the probability matrix $P$ has diagonal entries $p_{\mathrm{in}}$ and non-diagonal entries $p_{\mathrm{out}}$. When $p_{\mathrm{in}}>p_{\mathrm{out}}$, such $P$ is always full rank, and it is considered by various models (to name a few, DCSBM \cite{DCSBM}, MMSB \cite{MMSB}, OCCAM \cite{OCCAM}, DCMM \cite{MixedSCORE}, ScBM and DCScBM \cite{DISIM}, BiMMSB \cite{qingBiMMSB}, DiDCMM \cite{qingDiDCMM}, and so forth.) that can degenerate to SBM.  Meanwhile, $\tilde{P}$ is set such that it has unit diagonals and $\omega\in (0,1]$ as off-diagonals because we have assumed the maximum entry of $\tilde{P}$ is $1$ under MMSB for model identifiability. Actually, for the case that $P$ has unit diagonals and $\beta-1>1$ as off diagonals such that $P$'s diagonal entries are smaller than non-diagonal entries, we can also obtain similar separation condition, see discussions after Corollary \ref{AddConditionsDCMM}. Sure, in $step_{3}$ and $step_{4}$, the separation condition can be replaced by alternative separation condition. Furthermore, when we say ``a model degenerates to SBM'', we do not mean that the model can degenerate to SBM exactly. Instead, we mean that when a SBM models a network generated by the above $\tilde{P}$, the model can degenerate to such SBM.
\end{itemize}
\end{rem}
The above analysis shows that SCSTC can be used to study the consistent estimation of model based spectral methods. Use SCSTC, the following remark lists a few works whose main theoretical results leave possible improvements.
\begin{rem}\label{FailSC}
The unknown separation condition, or sub-optimal error rates, or a lack of requirement of network sparsity of some previous works, suggest possible improvements of their theoretical results. Here, we list a few works whose main results can be possibly improved until considering separation condition.
\begin{itemize}
  \item Theorem 4.4 of \cite{RSC} proposes upper bound of error rate for their regularized spectral clustering algorithm RSC under DCSBM. However, since \cite{RSC} does not study the lower bound (in \cite{RSC}'s language) of $\lambda_{K}$ and $m$, we can not directly obtain separation condition from their main theorem. Meanwhile, main result of \cite{RSC} does not consider the requirement on the network sparsity, which leaves some improvements.
  \item \cite{rohe2011spectral} and \cite{joseph2016impact} study two algorithms designed based on Laplcaian matrix and its regularized version under SBM. They obtain meaningful results, but do not consider the network sparsity parameter $\rho$ and separation parameter $\sigma_{K}(\tilde{P})$.
  \item Theorem 2.2 of \cite{SCORE} provides upper bound of their SCORE algorithm under DCSBM. However, since they does not consider the influence of $\sigma_{K}(\tilde{P})$, we can not directly obtain separation condition from their main result. Meanwhile, by setting their $\Theta=\sqrt{\rho}I$, then DCSBM degenerates to SBM, which gives that their $err_{n}=\frac{1}{\rho^{2}n}(1+\frac{\mathrm{log}(n)}{\rho n})=O(\frac{1}{\rho^{2}n})$ by their assumption Eq (2.9). Hence, when $\Theta=\sqrt{\rho}I$, upper bound of Theorem 2.2 in \cite{SCORE} is $O(\frac{\mathrm{log}^{3}(n)}{\rho^{2}n})$. Since the upper bound of error rate in Corollary 3.2 of \cite{lei2015consistency} is $O(\frac{\mathrm{log}(n)}{\rho n})$ when using $\|A-\Omega\|\leq C\sqrt{\rho n\mathrm{log}(n)}$ under the setting that $\kappa(\Pi)=O(1), K=O(1)$ and $\sigma_{K}(\tilde{P})=O(1)$. We see that $\frac{\mathrm{log}{^3}(n)}{\rho^{2}n}$ grows faster than $\frac{\mathrm{log}(n)}{\rho n}$, which suggests that there leaves space to improve main result of \cite{SCORE} in the aspects of separation condition and error rates.
  \item \cite{DISIM} proposes two models ScBM and DCScBM to model directed networks and an algorithm DiSIM based on directed regularized Laplacian matrix to fit DCScBM. However, similar as \cite{RSC}, their main theoretical result in their Theorem C.1 does not consider the lower bound of (in \cite{DISIM}'s language) $\sigma_{K}, m_{y}, m_{z}$ and $\gamma_{z}$, which causes that we can not obtain separation condition when DCScBM degenerates to SBM. Meanwhile, their Theorem C.1 also lacks a lower bound requirement on network sparsity. Hence, there leaves space to improve \cite{DISIM}'s theoretical guarantees.
  \item \cite{DSCORE} mainly studies the theoretical guarantee for the D-SCORE algorithm proposed by \cite{ji2016coauthorship} to fit a special case of DCScBM model for directed networks. By setting their $\theta(i)=\sqrt{\rho}, \delta(j)=\sqrt{\rho}$ for $i,j\in[n]$, then their directed-DCBM degenerates to SBM. Meanwhile, since their $err_{n}=\frac{1}{\rho}$,  their mis-clustering rate is $O(\frac{\mathrm{T^{2}_{n}\mathrm{log}(n)}}{\rho n})$, which matches that of \cite{lei2015consistency} under SBM when setting $T_{n}$ as a constant. However, if setting $T_{n}$ as $\mathrm{log}(n)$, then the error rate is $O(\frac{\mathrm{log}^{3}(n)}{\rho n})$, which is sub-optimal compared with that of \cite{lei2015consistency}. Meanwhile, similar as \cite{SCORE}, \cite{DSCORE}'s main result does not consider the influences of $K$ and $\sigma_{K}(\tilde{P})$, causing a lack of separation condition. Hence, main results of \cite{DSCORE} can be improved by considering $K$, $\sigma_{K}(P)$, or a more optimal choice of $T_{n}$ to make their main results be comparable with that of \cite{lei2015consistency} when directed-DCBM degenerates to SBM.
\end{itemize}
\end{rem}
\section{Degree corrected mixed membership model}\label{intDCMM}
Using SCSTC to Theorem 3.2 of \cite{MaoSVM}, as shown in Tables \ref{SCST} and \ref{aSCST}, results in Theorem 3.2 \cite{MaoSVM} are sub-optimal. To obtain improvement theoretical results, we give a formal introduction of the degree corrected mixed membership (DCMM) model proposed in \cite{MixedSCORE} first, then we review the SVM-cone-DCMMSB algorithm of \cite{MaoSVM} and provide improvement theoretical results. A DCMM for generating $A$ is as follows.
\begin{align}\label{DCMMmodel}
\Omega:=\Theta\Pi \tilde{P}\Pi'\Theta~~~~~~~~~A(i,j)\sim\mathrm{Bernoulli}(\Omega(i,j))~~~~i,j\in[n],
\end{align}
where $\Theta\in\mathbb{R}^{n\times n}$ is a diagonal matrix whose $i$-th diagonal entry is the degree heterogeneity of node $i$ for $i\in[n]$. Let $\theta\in\mathbb{R}^{n\times 1}$ with $\theta(i)=\Theta(i,i)$ for $i\in[n]$. Set $\theta_{\mathrm{max}}=\mathrm{max}_{i\in[n]}\theta(i), \theta_{\mathrm{min}}=\mathrm{min}_{i\in[n]}\theta(i)$ and $\tilde{P}_{\mathrm{max}}=\mathrm{max}_{k,l\in[K]}\tilde{P}(k,l), \tilde{P}_{\mathrm{min}}=\mathrm{min}_{k,l\in[K]}\tilde{P}(k,l)$.
\begin{defin}
Call model (\ref{DCMMmodel}) the degree corrected mixed membership (DCMM) model, and denote it by $DCMM_{n}(K,\tilde{P},\Pi,\Theta)$.
\end{defin}
Note that if we set $\tilde{\Pi}=\Theta \Pi$ and choose $\Theta$ such that $\tilde{\Pi}\in\{0,1\}^{n\times K}$, then we have $\Omega=\tilde{\Pi}\tilde{P}\tilde{\Pi}'$, which means that the stochastic blockmodel with overlap (SBMO) proposed in \cite{kaufmann2017a} is just a special case of DCMM.  Meanwhile, if we write $\Theta$ as $\Theta=\tilde{\Theta}D_{o}$ where $\tilde{\Theta},D_{o}$ are two positive diagonal matrices and let $\Pi_{o}=D_{o}\Pi$, then we can choose $D_{0}$ such that $\|\Pi_{o}(i,:)\|_{F}=1$. By $\Omega=\Theta \Pi \tilde{P}\Pi'\Theta=\tilde{\Theta}\Pi_{o}\tilde{P}\Pi'_{o}\tilde{\Theta}$, we see that the OCCAM model proposed in \cite{OCCAM} equals DCMM model actually. By Eq (1.3) and Proposition 1.1 of \cite{MixedSCORE}, the following conditions are sufficient for the identifiability of DCMM, when $\theta_{\mathrm{max}}\tilde{P}_{\mathrm{max}}\leq 1$,
\begin{itemize}
  \item (II1) $\mathrm{rank}(\tilde{P})=K$ and $\tilde{P}$ has unit diagonals.
  \item (II2) There is at least one pure node for each of the $K$ communities.
\end{itemize}
Note that though diagonal entries of $\tilde{P}$ are ones, $\tilde{P}_{\mathrm{max}}$ may be larger than $1$ as long as $\theta_{\mathrm{max}}\tilde{P}_{\mathrm{max}}\leq 1$ under DCMM, and this is slightly different as the setting that $\mathrm{max}_{k,l\in[K]}\tilde{P}(k,l)=1$ under MMSB. Similar as Eq (2.14) \cite{MixedSCORE}, let $\tilde{P}_{\mathrm{max}}\leq C$ for convenience. Meanwhile, from Condition (II1), though DCMM is an extension of SBM,MMSB and DCSBM, it can only model networks whose probability has equal positive entries.

Without causing confusion, under $DCMM_{n}(K,\tilde{P},\Pi,\Theta)$, we still let $\Omega=U\Lambda U'$ be the top-$K$ eigen value decomposition of $\Omega$ such that  $U\in\mathbb{R}^{n\times K}, \Lambda\in\mathbb{R}^{K\times K}$ and $U'U=I_{K}$. Set $U_{*}\in \mathbb{R}^{n\times K}$ by $U_{*}(i,:)=\frac{U(i,:)}{\|U(i,:)\|_{F}}$ and let $N_{U}\in\mathbb{R}^{n\times n}$ be a diagonal matrix such that $N_{U}(i,i)=\frac{1}{\|U(i,:)\|_{F}}$ for $i\in[n]$. Then $U_{*}$ can be rewritten as $U_{*}=N_{U}U$. The existence of the Ideal Cone (IC for short) structure inherent in $U_{*}$ mentioned in \cite{MaoSVM} is guaranteed by the following lemma.
\begin{lem}\label{IC}
Under $DCMM_{n}(K,\tilde{P},\Pi,\Theta)$, $U_{*}=YU_{*}(\mathcal{I},:)$
where $Y=N_{M}\Pi\Theta^{-1}(\mathcal{I},\mathcal{I})N_{U}^{-1}(\mathcal{I},\mathcal{I})$ with $N_{M}$ being an $n\times n$ diagonal matrix whose diagonal entries are positive.
\end{lem}
Lemma \ref{IC} gives $Y=U_{*}U^{-1}_{*}(\mathcal{I},:)$.
Since $U_{*}=N_{U}U$ and $Y=N_{M}\Pi\Theta^{-1}(\mathcal{I},\mathcal{I})N_{U}^{-1}(\mathcal{I},\mathcal{I})$, we have
\begin{align}\label{ZYJ1}
N_{U}^{-1}N_{M}\Pi=UU^{-1}_{*}(\mathcal{I},:)N_{U}(\mathcal{I},\mathcal{I})\Theta(\mathcal{I},\mathcal{I}).
\end{align}
Since $\Omega(\mathcal{I},\mathcal{I})=\Theta(\mathcal{I},\mathcal{I})\Pi(\mathcal{I},:)\tilde{P}\Pi'(\mathcal{I},:)=\Theta(\mathcal{I},\mathcal{I})\tilde{P}\Theta(\mathcal{I},\mathcal{I})=U(\mathcal{I},:)\Lambda U'(\mathcal{I},:)$, we have $\Theta(\mathcal{I},\mathcal{I})\tilde{P}\Theta(\mathcal{I},\mathcal{I})=U(\mathcal{I},:)\Lambda U'(\mathcal{I},:)$. Then we have $\Theta(\mathcal{I},\mathcal{I})=\sqrt{\mathrm{diag}(U(\mathcal{I},:)\Lambda U'(\mathcal{I},:))}$ when Condition (II1) holds such that $\tilde{P}$ has unit-diagonals. Set $J_{*}=N_{U}(\mathcal{I},\mathcal{I})\Theta(\mathcal{I},\mathcal{I})\equiv\sqrt{\mathrm{diag}(U_{*}(\mathcal{I},:)\Lambda U'_{*}(\mathcal{I},:))},
Z_{*}=N_{U}^{-1}N_{M}\Pi, Y_{*}=UU^{-1}_{*}(\mathcal{I},:)$. By Eq (\ref{ZYJ1}), we have
\begin{align}\label{Z1}
Z_{*}=Y_{*}J_{*}\equiv UU^{-1}_{*}(\mathcal{I},:)\mathrm{diag}(U_{*}(\mathcal{I},:)\Lambda U'_{*}(\mathcal{I},:)).
\end{align}
Meanwhile, since $N_{U}^{-1}N_{M}$ is an $n\times n$ positive diagonal matrix, we have
\begin{align}\label{Pi}
\Pi(i,:)=\frac{Z_{*}(i,:)}{\|Z_{*}(i,:)\|_{1}}, i\in[n].
\end{align}
With given $\Omega$ and $K$, we can obtain $U,U_{*}$ and $\Lambda$. The above analysis shows that once $U_{*}(\mathcal{I},:)$ is known, we can exactly recover $\Pi$ by Eq. (\ref{Z1}) and Eq. (\ref{Pi}). From Lemma \ref{IC}, we know that $U_{*}=YU_{*}(\mathcal{I},:)$ forms the IC structure. \cite{MaoSVM} proposes SVM-cone algorithm (i.e., Algorithm \ref{alg:SVMcone}) which can exactly obtain $U_{*}(\mathcal{I},:)$ from the Ideal Cone $U_{*}=YU_{*}(\mathcal{I},:)$ with inputs $U_{*}$ and $K$.

Based on the above analysis, we are now ready to give the ideal SVM-cone-DCMMSB algorithm. Input $\Omega, K$. Output: $\Pi$.
\begin{itemize}
  \item Let $\Omega=U\Lambda U'$ be the top-$K$ eigen decomposition of $\Omega$ such that $U\in\mathbb{R}^{n\times K}, \Lambda\in\mathbb{R}^{K\times K},U'U=I$. Let $U_{*}=N_{U}U$, where $N_{U}$ is an $n\times n$ diagonal matrix whose $i$-th diagonal entry is $\frac{1}{\|U(i,:)\|_{F}}$ for $i\in[n]$.
  \item Run SVM-cone algorithm on $U_{*}$ assuming that there are $K$ communities to obtain $\mathcal{I}$.
  \item Set $J_{*}=\sqrt{\mathrm{diag}(U_{*}(\mathcal{I},:)\Lambda U'_{*}(\mathcal{I},:))}, Y_{*}=UU^{-1}_{*}(\mathcal{I},:), Z_{*}=Y_{*}J_{*}$.
  \item Recover $\Pi$ by setting $\Pi(i,:)=\frac{Z_{*}(i,:)}{\|Z_{*}(i,:)\|_{1}}$ for $i\in[n]$.
\end{itemize}
With given $U_{*}$ and $K$, since SVM-cone algorithm returns $U_{*}(\mathcal{I},:)$, the ideal SVM-cone-DCMMSB exactly (for detail, see Appendix \ref{VHAandExactly}) returns $\Pi$.

Now, we review the SVM-cone-DCMMSB algorithm of \cite{MaoSVM}, where this algorithm can be seen as an extension of SPACL designed under MMSB to fit DCMM. For the real case, use $\hat{Y}_{*}, \hat{J}_{*},\hat{Z}_{*},\hat{\Pi}_{*}$ given in Algorithm \ref{alg:SVMDCMM} to estimate $Y_{*}, J_{*}, Z_{*}, \Pi$, respectively.
\begin{algorithm}
\caption{SVM-cone-DCMMSB \cite{MaoSVM}}
\label{alg:SVMDCMM}
\begin{algorithmic}[1]
\Require The adjacency matrix $A\in \mathbb{R}^{n\times n}$ and the number of communities $K$.
\Ensure The estimated $n\times K$ membership matrix $\hat{\Pi}_{*}$.
\State Obtain $\tilde{A}=\hat{U}\hat{\Lambda}\hat{U}'$, the top $K$ eigen-decomposition of $A$. Let $\hat{U}_{*}\in\mathbb{R}^{n\times K}$ such that $\hat{U}_{*}(i,:)=\frac{\hat{U}(i,:)}{\|\hat{U}(i,:)\|_{F}}$ for $i\in[n]$.
\State Apply SVM-cone algorithm (i.e., Algorithm \ref{alg:SVMcone}) on the rows of $\hat{U}_{*}$ assuming there are $K$ communities to obtain $\mathcal{\hat{I}}_{*}$, the index set returned by SVM-cone algorithm.
\State Set $\hat{J}_{*}=\sqrt{\mathrm{diag}(\hat{U}_{*}(\hat{I}_{*},:)\hat{\Lambda}\hat{U}'_{*}(\hat{\mathcal{I}}_{*},:))}, \hat{Y}_{*}=\hat{U}\hat{U}^{-1}_{*}(\hat{\mathcal{I}}_{*},:), \hat{Z}_{*}=\hat{Y}_{*}\hat{J}_{*}$ . Then set $\hat{Z}_{*}=\mathrm{max}(0, \hat{Z}_{*})$.
\State Estimate $\Pi(i,:)$ by $\hat{\Pi}_{*}(i,:)=\hat{Z}_{*}(i,:)/\|\hat{Z}_{*}(i,:)\|_{1}, i\in[n]$.
\end{algorithmic}
\end{algorithm}
\subsection{Consistency under DCMM}\label{ConsistencyDCMM}
Assume that
\begin{itemize}
  \item [(A2)]$\tilde{P}_{\mathrm{max}}\theta_{\mathrm{max}}\|\theta\|_{1}\geq \mathrm{log}(n)$.
\end{itemize}
Since we let $\tilde{P}_{\mathrm{max}}\leq C$, Assumption (A2) equals $\theta_{\mathrm{max}}\|\theta\|_{1}\geq\mathrm{log}(n)/C$. The following lemma bounds $\|A-\Omega\|$ under $DCMM_{n}(K,P,\Pi,\Theta)$ when Assumption (A2) holds.
\begin{lem}\label{BoundAOmegaDCMM}
Under $DCMM_{n}(K,\tilde{P},\Pi,\Theta)$, when Assumption (A2) holds, with probability at least $1-o(n^{-\alpha})$, we have
\begin{align*}
\|A-\Omega\|\leq \frac{\alpha+1+\sqrt{(\alpha+1)(\alpha+19)}}{3}\sqrt{\tilde{P}_{\mathrm{max}}\theta_{\mathrm{max}}\|\theta\|_{1}\mathrm{log}(n)}.
\end{align*}
\end{lem}
\begin{rem}
Consider a special case when $\Theta=\sqrt{\rho}I$ such that DCMM degenerates to MMSB, since $\tilde{P}_{\mathrm{max}}$ is assumed to be $1$ under MMSB, Assumption (A2) and the upper bound of $\|A-\Omega\|$ in Lemma \ref{BoundAOmegaDCMM} are consistent with Lemma \ref{BoundAOmega}. When all nodes are pure, DCMM degenerates to DCSBM \cite{DCSBM}, then the upper bound of $\|A-\Omega\|$ in Lemma \ref{BoundAOmegaDCMM} is also consistent with Lemma 2.2 of \cite{SCORE}. Meanwhile, this bound is also consistent with Eq (6.34) in the first version of \cite{MixedSCORE} which also applies the Bernstein inequality to bound $\|A-\Omega\|$. However, the bound is $C\sqrt{\theta_{\mathrm{max}}\|\theta\|_{1}}$ in Eq (C.53) of the latest version for \cite{MixedSCORE} which applies Corollary 3.12 and Remark 3.13 of \cite{bandeira2016sharp} to obtain the bound. Though the bound in Eq (C.53) of the latest version for \cite{MixedSCORE} is sharper by a $\sqrt{\mathrm{log}(n)}$ term, corollary 3.12 of \cite{bandeira2016sharp} has constraints on $W(i,j)$ (here, $W=A-\Omega$) such that $W(i,j)$ can be written as $W(i,j)=\xi_{ij}b_{ij}$, where $\{\xi_{i,j}:i\geq j\}$ are independent symmetric random variables with unit
variance and $\{b_{i,j}: i\geq j\}$ are given scalars, see the proof of Corollary 3.12 \cite{bandeira2016sharp} for detail. Therefore, without causing confusion, we  also use $A_{\mathrm{re}}$ to denote the constraint $A$ used in \cite{MixedSCORE} such that $\|A_{\mathrm{re}}-\Omega\|\leq C\sqrt{\theta_{\mathrm{max}}\|\theta\|_{1}}$. Furthermore, if we set $\rho\geq\mathrm{max}_{i,j}\Omega(i,j)$ such that $\rho\geq \theta^{2}_{\mathrm{max}}$, the bound in Lemma \ref{BoundAOmegaDCMM} also equals $\|A-\Omega\|\leq C\sqrt{\rho n\mathrm{log}(n)}$ and the assumption (A2) reads $\tilde{P}_{\mathrm{max}}\rho n\geq \mathrm{log}(n)$. The bound $\|A_{\mathrm{re}}-\Omega\|\leq C\sqrt{\theta_{\mathrm{max}}\|\theta\|_{1}}$ in Eq (C.53) of \cite{MixedSCORE} reads $\|A_{\mathrm{re}}-\Omega||\leq C\sqrt{\rho n}$.
\end{rem}
\begin{lem}\label{rowwiseerrorDCMM}
	(Row-wise eigenspace error) Under $DCMM_{n}(K,\tilde{P},\Pi,\Theta)$, when Assumption (A2) holds, suppose $\sigma_{K}(\Omega)\geq C\theta_{\mathrm{max}}\sqrt{\tilde{P}_{\mathrm{max}}n\mathrm{log}(n)}$, with probability at least $1-o(n^{-\alpha})$,
\begin{itemize}
\item when we apply Theorem 4.2.1 of \cite{chen2020spectral},
we have
\begin{align*}
\|\hat{U}\hat{U}'-UU'\|_{2\rightarrow\infty}=O(\frac{\theta_{\mathrm{max}}\sqrt{\tilde{P}_{\mathrm{max}}K}(\frac{\theta_{\mathrm{max}}\kappa(\Omega)}{\theta_{\mathrm{min}}}\sqrt{\frac{n}{K\lambda_{K}(\Pi'\Pi)}}+\sqrt{\mathrm{log}(n)})}{\theta^{2}_{\mathrm{min}}\sigma_{K}(\tilde{P})\lambda_{K}(\Pi'\Pi)}).
\end{align*}
\item when we apply Theorem 4.2 of \cite{cape2019the}, we have
\begin{align*}
\|\hat{U}\hat{U}'-UU'\|_{2\rightarrow\infty}=O(\frac{\theta_{\mathrm{max}}\sqrt{\tilde{P}_{\mathrm{max}}\theta_{\mathrm{max}}\|\theta\|_{1}\mathrm{log}(n)}}{\theta^{3}_{\mathrm{min}}\sigma_{K}(\tilde{P})\lambda^{1.5}_{K}(\Pi'\Pi)}).
\end{align*}
\end{itemize}
\end{lem}
Without causing confusion, we also use $\varpi,\varpi_{1},\varpi_{2}$ under DCMM as Lemma \ref{rowwiseerror} for notation convenience.
\begin{rem}
When $\Theta=\sqrt{\rho}I$ such that DCMM degenerates to MMSB, bounds in Lemma \ref{rowwiseerrorDCMM} are consistent with that of Lemma \ref{rowwiseerror}.
\end{rem}
\begin{rem}\label{rowDCMMRemark}
(Comparison to Theorem I.3 \cite{MaoSVM}) Note that the $\rho$ in \cite{MaoSVM} is $\theta^{2}_{\mathrm{max}}$, which gives that the row-wise eigenspace concentration in Theorem I.3 \cite{MaoSVM} is $O(\frac{\theta_{\mathrm{max}}\sqrt{Kn}\|U\|_{2\rightarrow\infty}\mathrm{log}^{\xi}(n)}{\sigma_{K}(\Omega)})$ when using $\|A_{\mathrm{re}}-\Omega\|\leq C\sqrt{\rho n}$ and this value is at least $O(\frac{\sqrt{\theta_{\mathrm{max}}\|\theta\|_{1}K}\|U\|_{2\rightarrow\infty}\mathrm{log}^{\xi}(n)}{\sigma_{K}(\Omega)})$. Since $\|U\|_{2\rightarrow\infty}\leq \frac{\theta_{\mathrm{max}}}{\theta_{\mathrm{min}}\sqrt{\lambda_{K}(\Pi'\Pi)}}$ by Lemma II.1 of \cite{MaoSVM} and $\sigma_{K}(\Omega)\geq \theta^{2}_{\mathrm{min}}\sigma_{K}(\tilde{P})\lambda_{K}(\Pi'\Pi)$ by the proof of Lemma \ref{rowwiseerrorDCMM}, we see that the upper bound of Theorem I.3 \cite{MaoSVM} is $O(\frac{\theta_{\mathrm{max}}\sqrt{K\theta_{\mathrm{max}}\|\theta\|_{1}}\mathrm{log}^{\xi}(n)}{\theta^{3}_{\mathrm{min}}\sigma_{K}(\tilde{P})\lambda^{1.5}_{K}(\Pi'\Pi)})$, which is $\sqrt{K}\mathrm{log}^{\xi-0.5}(n)$ (recall that $\xi>1$) times than our $\varpi_{2}$. Again, Theorem I.3 \cite{MaoSVM} has stronger requirements on the sparsity of $\theta_{\mathrm{max}}\|\theta\|_{1}$ and the lower bound of $\sigma_{K}(\Omega)$ than our Lemma \ref{rowwiseerrorDCMM}. When using the bound of $\|A-\Omega\|$ in our Lemma \ref{BoundAOmegaDCMM} to obtain the row-wise eigenspace concentration in Theorem I.3 \cite{MaoSVM}, their upper bound is $\sqrt{K}\mathrm{log}^{\xi}(n)$ times than our $\varpi_{2}$. Similar as  the first bullet given after Lemma \ref{rowwiseerror}, whether using $\|A-\Omega\|\leq C\sqrt{\theta_{\mathrm{max}}\|\theta\|_{1}\mathrm{log}(n)}$ or $\|A_{\mathrm{re}}-\Omega\|\leq C\sqrt{\theta_{\mathrm{max}}\|\theta\|_{1}}$ does not change our $\varpi$ under DCMM.
\end{rem}
\begin{rem}
(Comparison to Lemma 2.1 \cite{MixedSCORE}) The fourth bullet of Lemma 2.1 \cite{MixedSCORE} is the row-wise deviation bound for the eigenvectors of the adjacency matrix under some assumptions translated to our $\kappa(\Pi'\Pi)=O(1)$ , Assumption (A2) and lower bound requirement on $\sigma_{K}(\Omega)$ since they applies Lemma C.2 \cite{MixedSCORE}. The row-wise deviation bound in the fourth bullet of Lemma 2.1 \cite{MixedSCORE} reads $O(\frac{\theta_{\mathrm{max}}K^{1.5}\sqrt{\theta_{\mathrm{max}}\|\theta\|_{1}\mathrm{log}(n)}}{\sigma_{K}(\tilde{P})\|\theta\|^{3}_{F}})$, where the denominator is $\sigma_{K}(\tilde{P})\|\theta\|^{3}_{F}$ instead of our $\theta^{3}_{\mathrm{min}}\sigma_{K}(\tilde{P})\lambda^{1.5}_{K}(\Pi'\Pi)$ due to the fact that \cite{MixedSCORE} uses $\frac{\sigma_{K}(\tilde{P})\|\theta\|^{2}_{F}}{K}$ to roughly estimate $\sigma_{K}(\Omega)$ while we apply $\theta^{2}_{\mathrm{min}}\sigma_{K}(\tilde{P})\lambda_{K}(\Pi'\Pi)$ to strictly control the lower bound of $\sigma_{K}(\Omega)$. Therefore, we see that the row-wise deviation bound in the fourth bullet of Lemma 2.1 \cite{MixedSCORE} is consistent with our bounds in Lemma \ref{rowwiseerrorDCMM} when $\kappa(\Pi'\Pi)=O(1)$ while our row-wise eigenspace errors in Lemma \ref{rowwiseerrorDCMM} are more applicable than that of \cite{MixedSCORE} since we do not need to add constraint on $\Pi'\Pi$ such that $\kappa(\Pi'\Pi)=O(1)$. The upper bound of $\|A-\Omega\|$ of \cite{MixedSCORE} is $C\sqrt{\theta_{\mathrm{max}}\|\theta\|_{1}}$ given in their Eq (C.53) under $DCMM_{n}(K,\tilde{P},\Pi,\Theta)$, while ours is $C\sqrt{\theta_{\mathrm{max}}\|\theta\|_{1}\mathrm{log}(n)}$ in Lemma \ref{BoundAOmegaDCMM}, since our bound of row-wise eigenspace error in Lemma \ref{rowwiseerrorDCMM} is consistent with the fourth bullet of Lemma 2.1 \cite{MixedSCORE}, this supports the statement that the row-wise eigenspace error does not rely on $\|A-\Omega\|$ given in the first bullet after Lemma \ref{rowwiseerror}.
\end{rem}
Let $\pi_{\mathrm{min}}=\mathrm{min}_{1\leq k\leq K}\mathbf{1}'\Pi e_{k}$ , where $\pi_{\mathrm{min}}$ measures the minimum summation of nodes belong to a certain community. Increasing $\pi_{\mathrm{min}}$ makes the network tend to be more balanced, vice verse. Meanwhile, the term $\pi_{\mathrm{min}}$ appears when we propose a lower bound of $\eta$ defined in Lemma \ref{boundUeta} to keep track of model parameters in our main theorem under $DCMM_{n}(K,\tilde{P},\Pi,\Theta)$. Next theorem gives theoretical bounds on estimations of memberships under DCMM.
\begin{thm}\label{MainDCMM}
Under $DCMM_{n}(K,\tilde{P},\Pi,\Theta)$, suppose conditions in Lemma \ref{rowwiseerrorDCMM} hold, there exists a permutation matrix $\mathcal{P}_{*}\in\mathbb{R}^{K\times K}$ such that with probability at least $1-o(n^{-\alpha})$, we have
\begin{align*}
\mathrm{max}_{i\in[n]}\|e'_{i}(\hat{\Pi}_{*}-\Pi\mathcal{P}_{*})\|_{1}=O(\frac{\theta^{15}_{\mathrm{max}}K^{5}\varpi\kappa^{4.5}(\Pi'\Pi)\lambda^{1.5}_{1}(\Pi'\Pi)}{\theta^{15}_{\mathrm{min}}\pi_{\mathrm{min}}}).
\end{align*}
\end{thm}
For comparison, Table \ref{CompareDCMM} summaries the necessary conditions and dependence on model parameters of rates for Theorem \ref{MainDCMM} and Theorem 3.2 \cite{MaoSVM}, where the dependence on $K$ and $\mathrm{log}(n)$ are analyzed in Remark \ref{DCMMComK} given below.
\begin{rem}\label{DCMMComK}
(Comparison to Theorem 3.2 \cite{MaoSVM}) Our bound in Theorem \ref{MainDCMM} is written as combinations of model parameters and $\Pi$ can follow any distribution as long as Condition (II2) holds where such model parameters related form of estimation bound is  convenient for further theoretical analysis, see Corollary \ref{AddConditionsDCMM}, while bound in Theorem 3.2 \cite{MaoSVM} is built when $\Pi$ follows a Dirichlet distribution and $\kappa(\Pi'\Theta^{2}\Pi)=O(1)$. Meanwhile, since Theorem 3.2 \cite{MaoSVM} applies Theorem I.3 \cite{MaoSVM} to obtain the row-wise eigenspace error, bound in Theorem 3.2 \cite{MaoSVM} should multiple $\mathrm{log}^{\xi}(n)$ by Remark \ref{rowDCMMRemark}, and this is also supported by the fact that in the proof of Theorem 3.1 \cite{MaoSVM}, when computing bound of $\epsilon_{0}$ (in \cite{MaoSVM}'s language), \cite{MaoSVM} ignores the $\mathrm{log}^{\xi}(n)$ term.

Consider a special case by setting $\lambda_{K}(\Pi'\Pi)=O(\frac{n}{K}), \pi_{\mathrm{min}}=O(\frac{n}{K})$ and $\frac{\theta_{\mathrm{max}}}{\theta_{\mathrm{min}}}=O(1)$ with $\theta_{\mathrm{max}}=\sqrt{\rho}$, where such case matches the setting $\kappa(\Pi'\Theta^{2}\Pi)=O(1)$ in Theorem 3.2 \cite{MaoSVM}. Now we focus on analysing the powers of $K$ in our Theorem \ref{MainDCMM} and Theorem 3.2 \cite{MaoSVM}. Under this case, the power of $K$ in the estimation bound of our Theorem \ref{MainDCMM} is 6 by basic algebra; since $\mathrm{min}(K^{2},\kappa^{2}(\Omega))=\mathrm{min}(K^{2},O(1))=O(1), \frac{1}{\lambda^{2}_{K}(\Pi'\Theta^{2}\Pi)}=O(\frac{K^{2}}{\rho^{2}n^{2}})$, $\frac{1}{\eta}=O(K)$ by Lemma \ref{boundUeta} where $\eta$ in Lemma \ref{boundUeta} follows same definition as that of Theorem 3.2 \cite{MaoSVM}, and the bound in Theorem 3.2 \cite{MaoSVM} should multiply $\sqrt{K}$ because (in \cite{MaoSVM}'s language) $\|(\hat{Y}_{C}\hat{Y}'_{C})^{-1}\|_{F}$ should be no larger than $\frac{\sqrt{K}}{\lambda_{K}(\hat{Y}_{C}\hat{Y}'_{C})}$ instead of $\frac{1}{\lambda_{K}(\hat{Y}_{C}\hat{Y}'_{C})}$ in the proof of Theorem 2.8 \cite{MaoSVM},  the power of $K$ is 6 by checking the bound of Theorem 3.2 \cite{MaoSVM}. Meanwhile, note that our bound in Theorem \ref{MainDCMM} is $l_{1}$ bound while bound in Theorem 3.2 \cite{MaoSVM} is $l_{2}$ bound, when we translate the $l_{2}$ bound of Theorem 3.2 \cite{MaoSVM} into $l_{1}$ bound, the power of $K$ is 6.5 for Theorem 3.2 \cite{MaoSVM}, suggesting that our bound in Theorem \ref{MainDCMM} has less dependence on $K$ than that of Theorem 3.2 \cite{MaoSVM}.
\end{rem}
\begin{table}
\scriptsize
\centering
\caption{Comparison of error rates between our Theorem \ref{MainDCMM} and Theorem 3.2 \cite{MaoSVM} under $DCMM_{n}(K,P,\Pi,\Theta)$. The dependence on $K$ is obtained when $\kappa(\Pi'\Pi)=O(1)$. For comparison, we have adjusted the $l_{2}$ error rates of Theorem 3.2 \cite{MaoSVM} into $l_{1}$ error rates. Since Theorem \ref{MainDCMM} enjoys the  same separation condition and sharp threshold as Theorem \ref{Main}, and Theorem 3.2 \cite{MaoSVM} enjoys the  same separation condition and sharp threshold as Theorem 3.2 \cite{mao2020estimating}, we do not report them in this table. Note that as analyzed in Remark \ref{rowDCMMRemark}, whether using $\|A-\Omega\|\leq C\sqrt{\theta_{\mathrm{max}}\|\theta\|_{1}\mathrm{log}(n)}$ or $\|A_{\mathrm{re}}-\Omega\|\leq C\sqrt{\theta_{\mathrm{max}}\|\theta\|_{1}}$ does not change our $\varpi$ under DCMM, and has no influence results in Theorem \ref{MainDCMM}. For \cite{MaoSVM}: using $\|A_{\mathrm{re}}-\Omega\|\sqrt{\theta_{\mathrm{max}}\|\theta\|_{1}}$, the power of $\mathrm{log}(n)$ in their Theorem 3.2 is $\xi$; using $\|A-\Omega\|\sqrt{\theta_{\mathrm{max}}\|\theta\|_{1}\mathrm{log}(n)}$, the power of $\mathrm{log}(n)$ in their Theorem 3.2 is $\xi+0.5$.}
\label{CompareDCMM}
\begin{tabular}{c|ccccc|cccc}
\toprule
&$\Pi(i,:)$&$\theta_{\mathrm{max}}\|\theta\|_{1}$&$\sigma_{K}(\Omega)$&$\kappa(\Pi'\Theta^{2}\Pi)$&Dependence on $K$&Dependence on $\mathrm{log}(n)$\\
\midrule
Ours&arbitrary&$\geq\mathrm{log}(n)$&$\succeq \theta_{\mathrm{max}}\sqrt{n\mathrm{log}(n)}$&$\geq1$&$K^{6}$&$\mathrm{log}^{0.5}(n)$&\\
\hline
\cite{MaoSVM}&$iid$ from Dirichlet&$\geq\mathrm{log}^{2\xi}(n)$&$\succeq \theta_{\mathrm{max}}\sqrt{n}\mathrm{log}^{\xi}(n)$&$=O(1)$&$K^{6.5}$&$\mathrm{log}^{\xi}(n)$&\\
\bottomrule
\end{tabular}
\end{table}
The following corollary is obtained by adding some conditions on model parameters.
\begin{cor}\label{AddConditionsDCMM}
Under $DCMM_{n}(K,\tilde{P},\Pi,\Theta)$, when conditions of Lemma \ref{rowwiseerrorDCMM} hold, suppose $\lambda_{K}(\Pi'\Pi)=O(\frac{n}{K}), \pi_{\mathrm{min}}=O(\frac{n}{K})$ and $K=O(1)$, with probability at least $1-o(n^{-\alpha})$, we have
\begin{align*}
\mathrm{max}_{i\in[n]}\|e'_{i}(\hat{\Pi}_{*}-\Pi\mathcal{P}_{*})\|_{1}=O(\frac{\theta^{16}_{\mathrm{max}}\sqrt{\theta_{\mathrm{max}}\|\theta\|_{1}\mathrm{log}(n)}}{\theta^{18}_{\mathrm{min}}\sigma_{K}(\tilde{P})n}).
\end{align*}
Meanwhile, when $\theta_{\mathrm{max}}=O(\sqrt{\rho}), \theta_{\mathrm{min}}=O(\sqrt{\rho})$ (i.e.,
$\frac{\theta_{\mathrm{min}}}{\theta_{\mathrm{max}}}=O(1)$), we have
\begin{align*}
\mathrm{max}_{i\in[n]}\|e'_{i}(\hat{\Pi}_{*}-\Pi\mathcal{P}_{*})\|_{1}=O(\frac{1}{\sigma_{K}(\tilde{P})}\sqrt{\frac{\mathrm{log}(n)}{\rho n}}).
\end{align*}
\end{cor}
\begin{rem}
When $\lambda_{K}(\Pi'\Pi)=O(\frac{n}{K}), K=O(1), \theta_{\mathrm{max}}=O(\sqrt{\rho})$ and $\theta_{\mathrm{min}}=O(\sqrt{\rho})$, the requirement $\sigma_{K}(\Omega)\geq C\theta_{\mathrm{max}}\sqrt{\tilde{P}_{\mathrm{max}}n\mathrm{log}(n)}$ in Lemma \ref{rowwiseerrorDCMM} holds naturally. By the proof of Lemma \ref{rowwiseerrorDCMM}, $\sigma_{K}(\Omega)$ has a lower bound $\theta^{2}_{\mathrm{min}}\sigma_{K}(\tilde{P})\lambda_{K}(\Pi'\Pi)=O(\theta^{2}_{\mathrm{min}}\sigma_{K}(P)n)$. To make the requirement $\sigma_{K}(\Omega)\geq C\theta_{\mathrm{max}}\sqrt{\tilde{P}_{\mathrm{max}}n\mathrm{log}(n)}$ always hold, we just need $\theta^{2}_{\mathrm{min}}\sigma_{K}(\tilde{P})n\geq C\theta_{\mathrm{max}}\sqrt{\tilde{P}_{\mathrm{max}}n\mathrm{log}(n)}$, and it gives $\sigma_{K}(\tilde{P})\geq C\sqrt{\frac{\mathrm{log}(n)}{\rho n}}$, which matches the requirement of consistent estimation in Corollary \ref{AddConditionsDCMM}.
\end{rem}
Using SCSTC to Corollary \ref{AddConditionsDCMM}, let $\Theta=\sqrt{\rho}I$ such that DCMM degenerates to MMSB, it is easy to see that bound in Lemma \ref{AddConditionsDCMM} is consistent with that of Lemma \ref{AddConditions}. Therefore,  separation condition, alternative separation condition and sharp threshold obtained from Corollary \ref{AddConditionsDCMM} for the extended version of SPACL under DCMM are consistent with classical results, as shown in Tables \ref{SCST} and \ref{aSCST}. Meanwhile, when $\theta_{\mathrm{max}}=O(\sqrt{\rho}), \theta_{\mathrm{min}}=O(\sqrt{\rho})$ and settings in Corollary \ref{AddConditionsDCMM} hold, bound in Theorem 2.2 \cite{MixedSCORE} is of order $\frac{1}{\sigma_{K}(\tilde{P})}\sqrt{\frac{\mathrm{log}(n)}{\rho n}}$, which is consistent with our bound in Corollary \ref{AddConditionsDCMM}.

Consider a mixed membership network under the settings of Corollary \ref{AddConditionsDCMM} when $\Theta=\sqrt{\rho}I$ such that DCMM degenerates to SBM. By Corollary \ref{AddConditionsDCMM}, $\sigma_{K}(\tilde{P})$ should grow faster than $\sqrt{\frac{\mathrm{log}(n)}{\rho n}}$. We further assume that $\tilde{P}=(2-\beta)I_{K}+(\beta-1)\textbf{1}\textbf{1}'$ for $\beta\in [1,2)\cup(2,\infty)$, we see that this $\tilde{P}$ with unit diagonals and $\beta-1$ as non-diagonal entries still satisfies Condition (II1). Meanwhile, $\sigma_{K}(\tilde{P})=|\beta-2|=\tilde{P}_{\mathrm{max}}-\tilde{P}_{\mathrm{min}}$. When $\beta\in[1,2)$, this $\tilde{P}$ is the standard setting considered for separation condition in Section \ref{MainSCSTC}. Instead, we consider the case that $\beta\in (2,\infty)$ here. Then we have $\sigma_{K}(\tilde{P})=\beta-2$, and it should grow faster than $\sqrt{\frac{\mathrm{log}(n)}{\rho n}}$ for consistent estimation. Set $P=\rho P$ as the probability matrix for such $\tilde{P}$, we have  $p_{\mathrm{out}}=\rho (\beta-1), p_{\mathrm{in}}=\rho$, and the diagonal entries of $P$ are $p_{\mathrm{in}}$ (note that $p_{\mathrm{out}}>p_{\mathrm{in}}$ when $\beta>2$).  To obtain consistency estimation, $p_{\mathrm{out}}-p_{\mathrm{in}}=\rho (\beta-2)$ should grow faster than $\sqrt{\frac{\rho\mathrm{log}(n)}{n}}$.  Since $p_{\mathrm{in}}=\rho$, we see that $\frac{p_{\mathrm{out}}-p_{\mathrm{in}}}{\sqrt{p_{\mathrm{in}}}}$ should grow faster than $\sqrt{\frac{\mathrm{log}(n)}{n}}$ for consistent estimation. For the alternative separation condition, set $_{\mathrm{out}}=\alpha_{\mathrm{out}}\frac{\mathrm{log}(n)}{n}=\rho (\beta-1), p_{\mathrm{in}}=\alpha_{\mathrm{in}}\frac{\mathrm{log}(n)}{n}=\rho$ (note that $\alpha_{\mathrm{in}}<\alpha_{\mathrm{out}}$ when $\beta>2$), and we have $p_{\mathrm{out}}-p_{\mathrm{in}}=(\alpha_{\mathrm{out}}-\alpha_{\mathrm{in}})\frac{\mathrm{log}(n)}{n}\equiv\rho(\beta-2)$. For consistent estimation, $(\alpha_{\mathrm{out}}-\alpha_{\mathrm{in}})\frac{\mathrm{log}(n)}{n}$ should grow faster than $\sqrt{\frac{\rho \mathrm{log}(n)}{n}}$, which gives that $\frac{\alpha_{\mathrm{out}}-\alpha_{\mathrm{in}}}{\sqrt{\alpha_{\mathrm{in}}}}\gg1$. Follow similar analysis as that of separation condition and alternative separation condition, we obtain results in Tables \ref{SCST} and \ref{aSCST}.
\section{Conclusion}\label{Conclusion}
In this paper, the four step separation condition and sharp threshold criterion SCSTC is summarized as a unified framework to study consistencies and compare theoretical error rates of spectral methods under models that can degenerate to SBM in community detection area. With an application of this criterion, we find some inconsistent phenomena of a few previous works. Especially, using SCSTC we find that the original theoretical upper bounds on error rates of the SPACL algorithm under MMSB and its extended version under DCMM are sub-optimal at error rates and requirements on network sparsity. To find how the inconsistent phenomena occur, we re-establish theoretical upper bounds of error rats for both SPACL and its extended version by using recent techniques on row-wise eigenvector deviation. The resulting error bounds explicitly keep track of seven independent model parameters $(K,\rho,\sigma_{K}(\tilde{P}), \lambda_{K}(\Pi'\Pi), \lambda_{1}(\Pi'\Pi),\theta_{\mathrm{min}}, \theta_{\mathrm{max}})$, which allow us to have further delicate analysis. Compared with the original theoretical results, ours have smaller error rates with lesser dependence on $K$ and $\mathrm{log}(n)$, weaker requirements on the network sparsity and the lower bound of the smallest nonzero singular value of population adjacency matrix under both MMSB and DCMM. For DCMM, we have no constraint on the distribution of the membership matrix as long as it satisfies the identifiability condition. When considering the separation condition of a standard network and the probability to generate a connected Erd\"os-R\'enyi (ER) random graph by using SCSTC, our theoretical results match classical results. Meanwhile, our theoretical results also match that of Theorem 2.2 \cite{MixedSCORE} under mild conditions, and when DCMM degenerates to MMSB, theoretical results under DCMM are consistent with those under MMSB. Using the SCSTC criterion, we find that reasons behind the inconsistent phenomena are the sup-optimality of the original theoretical upper bounds on error rates for SPACL as well as its extended version, and whether using a regularization version of the adjacency matrix when builds theoretical results for spectral methods designed to detect nodes labels for non-mixed network. The processes of finding these inconsistent phenomena, sub-optimality theoretical results on error rates and the formation mechanism of these inconsistent phenomena, guarantee the usefulness of the SCSTC criterion. As shown by Remark \ref{FailSC}, theoretical results of some previous works can be improved by applying this criterion.  A limitation of this criterion is, it is only used for studying the consistency of spectral methods for a standard network with constant number of communities. It would be interesting to develop a more general criterion that can study consistency of all methods besides spectral methods and models besides those can degenerate to SBM for non-standard network with large $K$.
\bibliographystyle{imsart-nameyear}
\bibliography{refMMSBDCMM}
\appendix
\section{Vertex hunting algorithms}\label{VHAandExactly}
The SP algorithm is written as below.
	\begin{algorithm}
		\caption{\textbf{Successive Projection (SP)} \citep{gillis2015semidefinite}}
		\label{alg:SP}
		\begin{algorithmic}[1]
			\Require Near-separable matrix $Y_{sp}=S_{sp}M_{sp}+Z_{sp}\in\mathbb{R}^{m\times n}_{+}$ , where $S_{sp}, M_{sp}$ should satisfy Assumption 1 \cite{gillis2015semidefinite}, the number $r$ of columns to be extracted.
			\Ensure Set of indices $\mathcal{K}$ such that $Y(\mathcal{K},:)\approx S$ (up to permutation)
			\State Let $R=Y_{sp}, \mathcal{K}=\{\}, k=1$.
			\State \textbf{While} $R\neq 0$ and $k\leq r$ \textbf{do}
			\State ~~~~~~~$k_{*}=\mathrm{argmax}_{k}\|R(k,:)\|_{F}$.
			\State ~~~~~~$u_{k}=R(k_{*},:)$.
			\State ~~~~~~$R\leftarrow (I-\frac{u_{k}u'_{k}}{\|u_{k}\|^{2}_{F}})R$.
			\State ~~~~~~$\mathcal{K}=\mathcal{K}\cup \{k_{*}\}$.
			\State ~~~~~~k=k+1.
			\State \textbf{end while}
		\end{algorithmic}
	\end{algorithm}

Based on Algorithm \ref{alg:SP}, the following theorem is Theorem 1.1 in \cite{gillis2015semidefinite}, and it is also the Lemma VII.1 in \cite{mao2020estimating}. This theorem provides bound between the corner matrix $S_{sp}$ and its estimated version  returned by letting $Y_{sp}$ as input of SP algorithm when $M'_{sp}S'_{sp}$ enjoys the ideal simplex structure.
\begin{thm}\label{gillis2015siamSP}
Fix $m\geq r$ and $n\geq r$. Consider a matrix $Y_{sp}=S_{sp}M_{sp}+Z_{sp}$, where $S_{sp}\in\mathbb{R}^{m\times r}$ has a full column rank, $M_{sp}\in \mathbb{R}^{r\times n}$ is a nonnegative matrix such that the sum of each column is at most 1, and $Z_{sp}=[Z_{sp,1},\ldots, Z_{sp,n}]\in \mathbb{R}^{m\times n}$. Suppose $M_{sp}$ has a submatrix equal to $I_{r}$. Write $\epsilon\leq \mathrm{max}_{1\leq i\leq n}\|Z_{sp,i}\|_{F}$. Suppose $\epsilon=O(\frac{\sigma_{\mathrm{min}}(S_{sp})}{\sqrt{r}\kappa^{2}(S_{sp})})$, where $\sigma_{\mathrm{min}}(S_{sp})$ and $\kappa(S_{sp})$ are the minimum singular value and condition number of $S_{sp}$, respectively. If we apply the SP algorithm to columns of $Y_{sp}$, then it outputs an index set $\mathcal{K}\subset \{1,2,\ldots, n\}$ such that $|\mathcal{K}|=r$ and $\mathrm{max}_{1\leq k\leq r}\mathrm{min}_{j\in\mathcal{K}}\|S_{sp}(:,k)-Y_{sp}(:,j)\|_{F}=O(\epsilon \kappa^{2}(S_{sp}))$, where $S_{sp}(:,k)$ is the $k$-th column of $S_{sp}$.
\end{thm}
For the ideal SPACL algorithm, since inputs of the ideal SPACL are $\Omega$ and $K$, we see that the inputs of SP algorithm are $U$ and $K$. Let $m=K, r=K, Y_{sp}=U', Z_{sp}=U'-U'\equiv0, S_{sp}=U'(\mathcal{I},:),$ and $M_{sp}=\Pi'$. Then, we have $\mathrm{max}_{i\in[n]}\|U(i,:)-U(i,:)\|_{F}=0$. By Theorem \ref{gillis2015siamSP}, SP algorithm returns $\mathcal{I}$ up to permutation when the input is $U$ assuming there are $K$ communities. Since $U=\Pi U(\mathcal{I},:)$ under $MMSB_{n}(K,P,\Pi,\rho)$, we see that $U(i,:)=U(j,:)$ as long as $\Pi(i,:)=\Pi(j,:)$. Therefore, though $\mathcal{I}$ may be different up to permutation, $U(\mathcal{I},:)$ is unchanged. Therefore, follow the four steps of the ideal SPACL algorithm, we see that it exactly returns $\Pi$.

Algorithm \ref{alg:SVMcone} below is the SVM-cone algorithm provided in \cite{MaoSVM}.
\begin{algorithm}
\caption{\textbf{SVM-cone}\cite{MaoSVM}}
\label{alg:SVMcone}
\begin{algorithmic}[1]
\Require $\hat{S}\in \mathbb{R}^{n\times m}$ with rows have unit $l_{2}$ norm, number of corners $K$, estimated distance corners from hyperplane $\gamma$.
\Ensure The near-corner index set $\mathcal{\hat{I}}$.
\State Run one-class SVM on $\hat{S}(i,:)$ to get $\hat{\textbf{w}}$ and $\hat{b}$
\State Run K-means algorithm to the set $\{\hat{S}(i,:)| \hat{S}(i,:)\hat{\textbf{w}}\leq \hat{b}+\gamma\}$ that are close to the hyperplane into $K$ clusters
\State Pick one point from each cluster to get the near-corner set $\mathcal{\hat{I}}$
\end{algorithmic}
\end{algorithm}
As suggested in \cite{MaoSVM}, we can start $\gamma=0$ and incrementally increase it until $K$ distinct clusters are found. Meanwhile, for the ideal SVM-cone-DCMMSB algorithm, when setting $U_{*}$ and $K$ as the inputs of the SVM-cone algorithms, since $\|U_{*}-U_{*}\|_{2\rightarrow\infty}=0$, Lemma F.1. \cite{MaoSVM} guarantees that SVM-cone algorithm returns $\mathcal{I}$ up to permutation. Since $U_{*}=Y U_{*}(\mathcal{I},:)$ by Lemma \ref{IC} under $DCMM_{n}(K,P,\Pi,\Theta)$, we have $U_{*}(i,:)=U_{*}(j,:)$ when $\Pi(i,:)=\Pi(j,:)$ by basic algebra, which gives that $U_{*}(\mathcal{I},:)$ is unchanged though $\mathcal{I}$ may be different up to permutation. Therefore, the ideal SVM-cone-DCMMSB exactly recovers $\Pi$.
\section{Proof of consistency under MMSB}
\subsection{Proof of Lemma \ref{BoundAOmega}}
\begin{proof}
We apply Theorem 1.4 (Bernstein inequality) in \cite{tropp2012user} to bound $\|A-\Omega\|$, and this theorem is written as below
\begin{thm}\label{Bern}
Consider a finite sequence $\{X_{k}\}$ of independent, random, self-adjoint matrices with dimension $d$. Assume that each random matrix satisfies
\begin{align*}
\mathbb{E}[X_{k}]=0, \mathrm{and~}\lambda_{\mathrm{max}}(X_{k})\leq R~\mathrm{almost~surely}.
\end{align*}
Then, for all $t\geq 0$,
\begin{align*}
\mathbb{P}(\lambda_{\mathrm{max}}(\sum_{k}X_{k})\geq t)\leq d\cdot \mathrm{exp}(\frac{-t^{2}/2}{\sigma^{2}+Rt/3}),
\end{align*}
where $\sigma^{2}:=\|\sum_{k}\mathbb{E}[X^{2}_{k}]\|$.
\end{thm}
Let $e_{i}$ be an $n\times 1$ vector, where $e_{i}(i)=1$ and 0 elsewhere,for $i\in[n]$. For convenience, set $W=A-\Omega$. Then we can write $W$ as $W=\sum_{i=1}^{n}\sum_{j=1}^{n}W(i,j)e_{i}e'_{j}$. Set $W^{(i,j)}$ as the $n\times n$ matrix such that $W^{(i,j)}=W(i,j)(e_{i}e'_{j}+e_{j}e_{i}')$, which gives
$W=\sum_{1\leq i< j\leq n}W^{(i,j)}$ where $\mathbb{E}[W^{(i,j)}]=0$ and
\begin{align*}
\|W^{(i,j)}\|&=\|W(i,j)(e_{i}e'_{j}+e_{j}e_{i})\|=|W(i,j)|\|(e_{i}e'_{j}+e_{j}e_{i}')\|=|W(i,j)|=|A(i,j)-\Omega(i,j)|\leq 1.
\end{align*}
For the variance parameter $\sigma^{2}:=\|\sum_{1\leq i<j\leq n}\mathbb{E}[(W^{(i,j)})^{2}]\|$. We bound $\mathbb{E}(W^{2}(i,j))$ as below
\begin{align*}
\mathbb{E}(W^{2}(i,j))=\mathbb{E}((A(i,j)-\Omega(i,j))^{2})=\mathrm{Var}(A(i,j))=\Omega(i,j)(1-\Omega(i,j))\leq\Omega(i,j)=\rho\Pi(i,:)\tilde{P}\Pi'(j,:)\leq \rho.
\end{align*}
Next we bound $\sigma^{2}$ as below
\begin{align*}
\sigma^{2}&=\|\sum_{1\leq i<j\leq n}\mathbb{E}(W^{2}(i,j))(e_{i}e_{j}'+e_{j}e_{i}')(e_{i}e_{j}'+e_{j}e_{i}')\|=\|\sum_{1\leq i<j\leq n}\mathbb{E}[W^{2}(i,j)(e_{i}e'_{i}+e_{j}e_{j}')]\|\\
&\leq\underset{1\leq i\leq n}{\mathrm{max}}|\sum_{j=1}^{n}\mathbb{E}(W^{2}(i,j))|\leq \underset{1\leq i\leq n}{\mathrm{max}}\sum_{j=1}^{n}\rho=\rho n.
\end{align*}
Set $t=\frac{\alpha+1+\sqrt{(\alpha+1)(\alpha+19)}}{3}\sqrt{\rho n\mathrm{log}(n)}$ for any $\alpha>0$, combine Theorem \ref{Bern} with $\sigma^{2}\leq \rho n, R=1, d=n$, we have
\begin{align*}
\mathbb{P}(\|W\|\geq t)&=\mathbb{P}(\|\sum_{1\leq i<j\leq n}W^{(i,j)}\|\geq t)\leq n\cdot \mathrm{exp}(\frac{-t^{2}/2}{\sigma^{2}+Rt/3})\\
&\leq n\cdot\mathrm{exp}(\frac{-(\alpha+1)\mathrm{log}(n)}{\frac{18}{(\sqrt{\alpha+1}+\sqrt{\alpha+19})^{2}}+\frac{2\sqrt{\alpha+1}}{\sqrt{\alpha+1}+\sqrt{\alpha+19}}\sqrt{\frac{\mathrm{log}(n)}{\rho n}}})\leq \frac{1}{n^{\alpha}},
\end{align*}
where we have used Assumption (A1) such that $\frac{18}{(\sqrt{\alpha+1}+\sqrt{\alpha+19})^{2}}+\frac{2\sqrt{\alpha+1}}{\sqrt{\alpha+1}+\sqrt{\alpha+19}}\sqrt{\frac{\mathrm{log}(n)}{\rho n}}\leq \frac{18}{(\sqrt{\alpha+1}+\sqrt{\alpha+19})^{2}}+\frac{2\sqrt{\alpha+1}}{\sqrt{\alpha+1}+\sqrt{\alpha+19}}=1$.
\end{proof}
\subsection{Proof of Lemma \ref{rowwiseerror}}
\begin{proof}
Let $H=\hat{U}'U$, and $H=U_{H}\Sigma_{H}V'_{H}$ be the SVD decomposition of $H_{\hat{U}}$ with $U_{H},V_{H}\in \mathbb{R}^{n\times K}$, where $U_{H}$ and $V_{H}$ represent respectively the left and right singular matrices of $H$. Define $\mathrm{sgn}(H)=U_{H}V'_{H}$. Since $\mathbb{E}(A(i,j)-\Omega(i,j))=0$, $\mathbb{E}[(A(i,j)-\Omega(i,j))^{2}]\leq\rho$ by the proof of Lemma \ref{BoundAOmega}, $\frac{1}{\sqrt{\rho n/(\mu \mathrm{log}(n))}}\leq O(1)$ holds by Assumption (A1) where $\mu$ is the incoherence parameter defined as $\mu=\frac{n\|U\|^{2}_{2\rightarrow\infty}}{K}$. By Theorem 4.2.1. \cite{chen2020spectral}, with high probability, we have
\begin{align*}
\|\hat{U}\mathrm{sgn}(H)-U\|_{2\rightarrow\infty}\leq C\frac{\sqrt{K\rho}(\kappa(\Omega)\sqrt{\mu}+\sqrt{\mathrm{log}(n)})}{\sigma_{K}(\Omega)},
\end{align*}
provided that $c_{1}\sigma_{K}(\Omega)\geq \sqrt{\rho n\mathrm{log}(n)}$ for some sufficiently small constant $c_{1}$. By Lemma 3.1 of \cite{mao2020estimating}, we know that  $\|U\|^{2}_{2\rightarrow\infty}\leq \frac{1}{\lambda_{K}(\Pi'\Pi)}$, which gives
\begin{align}\label{errorhatUhatV}
\|\hat{U}\mathrm{sgn}(H)-U\|_{2\rightarrow\infty}\leq C\frac{\sqrt{K\rho}(\kappa(\Omega)\sqrt{\frac{n}{K\lambda_{K}(\Pi'\Pi)}}+\sqrt{\mathrm{log}(n)})}{\sigma_{K}(\Omega)}.
\end{align}
\begin{rem}
By Theorem 4.2 of \cite{cape2019the}, when $\sigma_{K}(\Omega)\geq 4\|A-\Omega\|_{\infty}$, we have
\begin{align*}
\|\hat{U}\mathrm{sgn}(H)-U\|_{2\rightarrow\infty}\leq 14\frac{\|A-\Omega\|_{\infty}}{\sigma_{K}(\Omega)}\|U\|_{2\rightarrow\infty}.
\end{align*}
By Lemma 3.1 \cite{mao2020estimating}, we have
\begin{align*}
\|\hat{U}\mathrm{sgn}(H)-U\|_{2\rightarrow\infty}\leq 14\frac{\|A-\Omega\|_{\infty}}{\sigma_{K}(\Omega)}\sqrt{\frac{1}{\lambda_{K}(\Pi'\Pi)}}.
\end{align*}
Unlike Lemma V.1 \cite{mao2020estimating} which bounds $\|A-\Omega\|_{\infty}$ via Chernoff bound and obtains $\|A-\Omega\|_{\infty}\leq C\rho n$ with high probability, we bound $\|A-\Omega\|_{\infty}$ by Bernstein inequality using similar idea as Eq (C.67) of \cite{MixedSCORE}. Let $y=(y_{1},y_{2},\ldots, y_{n})'$ be any $n\times 1$ vector, by Eq (C.67) \cite{MixedSCORE} we know that with an application of Bernstein inequality, for any $t\geq 0$ and $i\in[n]$, we have
\begin{align*}
\mathbb{P}(|\sum_{j=1}^{n}(A(i,j)-\Omega(i,j))y(j)|>t)\leq 2\mathrm{exp}(-\frac{t^{2}/2}{\sum_{j=1}^{n}\Omega(i,j)y^{2}(j)+\frac{t\|y\|_{\infty}}{3}}).
\end{align*}
By the proof of Lemma \ref{BoundAOmega}, we have $\Omega(i,j)\leq \rho$. Set $y(j)$ as $1$ or $-1$ such that $(A(i,j)-\Omega(i,j))y(j)=|A(i,j)-\Omega(i,j)|$, we have
\begin{align*}
\mathbb{P}(\|A-\Omega\|_{\infty}>t)\leq 2\mathrm{exp}(-\frac{t^{2}/2}{\rho n+\frac{t}{3}}).
\end{align*}
Set $t=\frac{\alpha+1+\sqrt{(\alpha+1)(\alpha+19)}}{3}\sqrt{\rho n\mathrm{log}(n)}$ for any $\alpha>0$, by Assumption (A1), we have
\begin{align*}
\mathbb{P}(\|A-\Omega\|_{\infty}>t)\leq 2\mathrm{exp}(-\frac{t^{2}/2}{\rho n+\frac{t}{3}})\leq n^{-\alpha}.
\end{align*}
Hence, when $\sigma_{K}(\Omega)\geq C_{0}\sqrt{\rho n\mathrm{log}(n)}$ where $C_{0}=4\frac{\alpha+1+\sqrt{(\alpha+1)(\alpha+19)}}{3}$, with probability at least $1-o(n^{-\alpha})$,
\begin{align*}
\|\hat{U}\mathrm{sgn}(H_{\hat{U}})-U\|_{2\rightarrow\infty}\leq C\frac{\sqrt{\rho n\mathrm{log}(n)}}{\sigma_{K}(\Omega)}\sqrt{\frac{1}{\lambda_{K}(\Pi'\Pi)}}.
\end{align*}
Note that when $\lambda_{K}(\Pi'\Pi)=O(\frac{n}{K})$, the above bound turns to be $C\frac{\sqrt{\rho K\mathrm{log}(n)}}{\sigma_{K}(\Omega)}$, which is consistent with that of Eq (\ref{errorhatUhatV}). Also note that this bound $\frac{\sqrt{\rho n\mathrm{log}(n)}}{\sigma_{K}(\Omega)}\sqrt{\frac{1}{\lambda_{K}(\Pi'\Pi)}}$ is sharper than the $\frac{\rho n}{\sigma_{K}(\Omega)}\sqrt{\frac{1}{\lambda_{K}(\Pi'\Pi)}}$ of Lemma V.1 \cite{mao2020estimating} by Assumption (A1).
\end{rem}
Since $\hat{U}$ and $U$ have orthonormal columns,
now we are ready to bound $\|\hat{U}\hat{U}'-UU'\|_{2\rightarrow\infty}$:
\begin{align*}	
&\|\hat{U}\hat{U}'-UU'\|_{2\rightarrow\infty}=\mathrm{max}_{i\in[n]}\|e'_{i}(UU'-\hat{U}\hat{U}')\|_{F}\\
&=\mathrm{max}_{i\in[n]}\|e'_{i}(UU'-\hat{U}\mathrm{sgn}(H)U'+\hat{U}\mathrm{sgn}(H)U'-\hat{U}\hat{U}')\|_{F}\\
&\leq\mathrm{max}_{i\in[n]}\|e'_{i}(U-\hat{U}\mathrm{sgn}(H))U'\|_{F}+\mathrm{max}_{i\in[n]}\|e'_{i}\hat{U}(\mathrm{sgn}(H)U'-\hat{U}')\|_{F}\\
&=\mathrm{max}_{i\in[n]}\|e'_{i}(U-\hat{U}\mathrm{sgn}(H))\|_{F}+\mathrm{max}_{i\in[n]}\|\hat{U}(\mathrm{sgn}(H)U'-\hat{U}')e_{i}\|_{F}\\
&=\mathrm{max}_{i\in[n]}\|e'_{i}(U-\hat{U}\mathrm{sgn}(H))\|_{F}+\mathrm{max}_{i\in[n]}\|(\mathrm{sgn}(H)U'-\hat{U}')e_{i}\|_{F}\\
&=\mathrm{max}_{i\in[n]}\|e'_{i}(U-\hat{U}\mathrm{sgn}(H))\|_{F}+\mathrm{max}_{i\in[n]}\|e'_{i}(U(\mathrm{sgn}(H))'-\hat{U})\|_{F}\\
&=\mathrm{max}_{i\in[n]}\|e'_{i}(U-\hat{U}\mathrm{sgn}(H))\|_{F}+\mathrm{max}_{i\in[n]}\|e'_{i}(U-\hat{U}\mathrm{sgn}(H))\|_{F}\\
&=2\mathrm{max}_{i\in[n]}\|e'_{i}(U-\hat{U}\mathrm{sgn}(H))\|_{F}=2\|U-\hat{U}\mathrm{sgn}(H)\|_{2\rightarrow\infty}\\
&\leq C\frac{\sqrt{K}(\kappa(\Omega)\sqrt{\frac{n}{K\lambda_{K}(\Pi'\Pi)}}+\sqrt{\mathrm{log}(n)})}{\sigma_{K}(\tilde{P})\sqrt{\rho}\lambda_{K}(\Pi'\Pi)},
\end{align*}
where the last inequality holds since $\sigma_{K}(\Omega)\geq \sigma_{K}(\tilde{P})\rho \lambda_{K}(\Pi'\Pi)$ under $MMSB_{n}(K,\tilde{P},\Pi,\rho)$ by Lemma II.4 \cite{mao2020estimating}
And this bound is $C\frac{\sqrt{n\mathrm{log}(n)}}{\sigma_{K}(\tilde{P})\sqrt{\rho}\lambda^{1.5}_{K}(\Pi'\Pi)}$ if we use Theorem 4.2 of \cite{cape2019the}.
\begin{rem}
By Theorem 4.5 \cite{cape2020orthogonal}, we have $\|\hat{U}\hat{U}'-UU'\|_{2\rightarrow\infty}\leq \sqrt{n}(\|\hat{U}\|_{2\rightarrow\infty}+\|U\|_{2\rightarrow\infty})\|U-\hat{U}\mathrm{sgn}(H)\|_{2\rightarrow\infty}\leq \sqrt{n}(\|U-\hat{U}\mathrm{sgn}(H)\|_{2\rightarrow\infty}+2\|U\|_{2\rightarrow\infty})\|U-\hat{U}\mathrm{sgn}(H)\|_{2\rightarrow\infty}\leq\sqrt{n}(\|U-\hat{U}\mathrm{sgn}(H)\|_{2\rightarrow\infty}+\frac{2}{\sqrt{\lambda_{K}(\Pi'\Pi)}})\|U-\hat{U}\mathrm{sgn}(H)\|_{2\rightarrow\infty}=O(\frac{2\sqrt{n}}{\sqrt{\lambda_{K}(\Pi'\Pi)}}\|U-\hat{U}\mathrm{sgn}(H)\|_{2\rightarrow\infty})$. Sure our bound $\|\hat{U}\hat{U}'-UU'\|_{2\rightarrow\infty}\leq 2\|U-\hat{U}\mathrm{sgn}(H)\|_{2\rightarrow\infty}$ enjoys concise form.  Especially, when $\lambda_{K}(\Pi'\Pi)=O(\frac{n}{K})$ and $K=O(1)$, the two bounds give that $\|\hat{U}\hat{U}'-UU'\|_{2\rightarrow\infty}=O(\|U-\hat{U}\mathrm{sgn}(H)\|_{2\rightarrow\infty})$, which provides same error bound of estimated memberships given in Corollary \ref{AddConditions} .
\end{rem}
\end{proof}
\subsection{Proof of Theorem \ref{Main}}
\begin{proof}
Follow almost the same proof as Eq (3) of \cite{mao2020estimating}, for $i\in[n]$, there exists a permutation matrix $\mathcal{P}\in\mathbb{R}^{K\times K}$ such that
\begin{align}\label{ourBoundonZ}
\|e'_{i}(\hat{Z}-Z\mathcal{P})\|_{F}=O(\varpi\kappa(\Pi'\Pi)\sqrt{K\lambda_{1}(\Pi'\Pi)}).
\end{align}
Note that the bound in Eq (\ref{ourBoundonZ}) is $\sqrt{K}$ times of the bound in Eq (3) of \cite{mao2020estimating}, this is because in the Eq (3) of  \cite{mao2020estimating}, (in \cite{mao2020estimating}'s language) $\|\hat{V}^{-1}_{p}\|$ denotes the Frobenius norm of $\hat{V}^{-1}_{p}$ instead of the spectral norm. Since $\|\hat{V}^{-1}_{p}\|_{F}\leq \frac{\sqrt{K}}{\sigma_{K}(\hat{V}_{p})}$, the bound in Eq (3) \cite{mao2020estimating} should multiply $\sqrt{K}$.

Recall that $Z=\Pi,\Pi(i,:)=\frac{Z(i,:)}{\|Z(i,:)\|_{1}}, \hat{\Pi}(i,:)=\frac{\hat{Z}(j,:)}{\|\hat{Z}(j,:)\|_{1}}$, for $i\in[n]$, since
\begin{align*}	\|e'_{i}(\hat{\Pi}-\Pi\mathcal{P})\|_{1}&=\|\frac{e'_{i}\hat{Z}}{\|e'_{i}\hat{Z}\|_{1}}-\frac{e'_{i}Z\mathcal{P}}{\|e'_{i}Z\mathcal{P}\|_{1}}\|_{1}=\|\frac{e'_{i}\hat{Z}\|e'_{i}Z\|_{1}-e'_{i}Z\mathcal{P}\|e'_{i}\hat{Z}\|_{1}}{\|e'_{i}\hat{Z}\|_{1}\|e'_{i}Z\|_{1}}\|_{1}\\	&\leq\frac{\|e'_{i}\hat{Z}\|e'_{i}Z\|_{1}-e'_{i}\hat{Z}\|e'_{i}\hat{Z}\|_{1}\|_{1}+\|e'_{i}\hat{Z}\|e'_{i}\hat{Z}\|_{1}-e'_{i}Z\mathcal{P}\|e'_{i}\hat{Z}\|_{1}\|_{1}}{\|e'_{i}\hat{Z}\|_{1}\|e'_{i}Z\|_{1}}\\	&=\frac{|\|e'_{i}Z\|_{1}-\|e'_{i}\hat{Z}\|_{1}|+\|e'_{i}\hat{Z}-e'_{i}Z\mathcal{P}\|_{1}}{\|e'_{i}Z\|_{1}}\leq\frac{2\|e'_{i}(\hat{Z}-Z\mathcal{P})\|_{1}}{\|e'_{i}Z\|_{1}}\\
&=\frac{2\|e'_{i}(\hat{Z}-Z\mathcal{P})\|_{1}}{\|e'_{i}\Pi\|_{1}}=2\|e'_{i}(\hat{Z}-Z\mathcal{P})\|_{1}\leq 2\sqrt{K}\|e'_{i}(\hat{Z}-Z\mathcal{P})\|_{F}\\
&=O(\varpi K\kappa(\Pi'\Pi)\sqrt{\lambda_{1}(\Pi'\Pi)}).
\end{align*}
\end{proof}
\subsection{Proof of Corollary \ref{AddConditions}}
\begin{proof}
Under  conditions of Corollary \ref{AddConditions}, we have
\begin{align*}
\mathrm{max}_{i\in[n]}\|e'_{i}(\hat{\Pi}-\Pi\mathcal{P})\|_{1}=O(\varpi\sqrt{n}).
\end{align*}
Under  conditions of Corollary \ref{AddConditions}, Lemma \ref{rowwiseerror} gives $\varpi=O(\frac{1}{\sigma_{K}(\tilde{P})}\frac{1}{\sqrt{n}}\sqrt{\frac{\mathrm{log}(n)}{\rho n}})$,
which gives that
\begin{align*}
\mathrm{max}_{i\in[n]}\|e'_{i}(\hat{\Pi}-\Pi\mathcal{P})\|_{1}=O(\frac{1}{\sigma_{K}(\tilde{P})}\sqrt{\frac{\mathrm{log}(n)}{\rho n}}).
\end{align*}
\end{proof}
\section{Proof of consistency under DCMM}
\subsection{Proof of Lemma \ref{IC}}
\begin{proof}
Since $\Omega=U\Lambda U'$, we have $U=\Omega U\Lambda^{-1}$ since $U'U=I_{K}$. Recall that $\Omega=\Theta\Pi \tilde{P}\Pi'\Theta$, we have $U=\Theta\Pi \tilde{P}\Pi'\Theta U\Lambda^{-1}=\Theta\Pi B$, where we set $B=\tilde{P}\Pi'\Theta U\Lambda^{-1}$ for convenience. Since $U(\mathcal{I},:)=\Theta(\mathcal{I},\mathcal{I})\Pi(\mathcal{I},:)B=\Theta(\mathcal{I},\mathcal{I})B$, we have $B=\Theta^{-1}(\mathcal{I},\mathcal{I})U(\mathcal{I},:)$.
	
Set $M=\Pi B$. Then we have $U=\Theta M$, which gives that $U(i,:)=e'_{i}U=\Theta(i,i)M(i,:)$ for $i\in[n]$. Therefore, $U_{*}(i,:)=\frac{U(i,:)}{\|U(i,:)\|_{F}}=\frac{M(i,:)}{\|M(i,:)\|_{F}}$, combine it with the fact that $B=\Theta^{-1}(\mathcal{I},\mathcal{I})U(\mathcal{I},:)\equiv\Theta^{-1}(\mathcal{I},\mathcal{I})N^{-1}_{U}(\mathcal{I},\mathcal{I})N_{U}(\mathcal{I},\mathcal{I})U(\mathcal{I},:)\equiv\Theta^{-1}(\mathcal{I},\mathcal{I})N^{-1}_{U}(\mathcal{I},\mathcal{I})U_{*}(\mathcal{I},:)$, we have
\begin{flalign*}
U_{*}=\begin{bmatrix}
\Pi(1,:)/\|M(1,:)\|_{F}\\
\Pi(2,:)/\|M(2,:)\|_{F}\\
\vdots\\
\Pi(n,:)/\|M(n,:)\|_{F} \end{bmatrix}B=\begin{bmatrix}
\Pi(1,:)/\|M(1,:)\|_{F}\\
\Pi(2,:)/\|M(2,:)\|_{F}\\
\vdots\\
\Pi(n,:)/\|M(n,:)\|_{F}
\end{bmatrix}\Theta^{-1}(\mathcal{I},\mathcal{I})N_{U}^{-1}(\mathcal{I},\mathcal{I})U_{*}(\mathcal{I},:).
\end{flalign*}
Therefore, we have
\begin{align*}
Y=N_{M}\Pi\Theta^{-1}(\mathcal{I},\mathcal{I})N_{U}^{-1}(\mathcal{I},\mathcal{I}),
\end{align*}
where $N_{M}$ is a diagonal matrix whose $i$-th diagonal entry is $\frac{1}{\|M(i,:)\|_{F}}$ for $i\in[n]$.
\end{proof}
\subsection{Proof of Lemma \ref{BoundAOmegaDCMM}}
\begin{proof}
Similar as the proof of Lemma \ref{BoundAOmega}, set $W=A-\Omega$ and $W^{(i,j)}=W(i,j)(e_{i}e'_{j}+e_{j}e_{i}')$, we have $W=\sum_{1\leq i< j\leq n}W^{(i,j)}$, $\mathbb{E}[W^{(i,j)}]=0$ and $ \|W^{(i,j)}\|\leq 1$. Since
\begin{align*} \mathbb{E}(W^{2}(i,j))&=\mathbb{E}((A(i,j)-\Omega(i,j))^{2})=\mathrm{Var}(A(i,j))= \Omega(i,j)(1-\Omega(i,j))\\
&\leq \Omega(i,j)=\theta(i)\theta(j)\Pi(i,:)\tilde{P}\Pi'(j,:)\leq \theta(i)\theta(j)\tilde{P}_{\mathrm{max}},
\end{align*}
we have
\begin{align*}
\sigma^{2}&=\|\sum_{1\leq i<j\leq n}\mathbb{E}(W^{2}(i,j))(e_{i}e_{j}'+e_{j}e_{i}')(e_{i}e_{j}'+e_{j}e_{i}')\|=\|\sum_{1\leq i<j\leq n}\mathbb{E}[W^{2}(i,j)(e_{i}e'_{i}+e_{j}e_{j}')]\|\\
	&\leq\underset{1\leq i\leq n}{\mathrm{max}}|\sum_{j=1}^{n}\mathbb{E}(W^{2}(i,j))|\leq \underset{1\leq i\leq n}{\mathrm{max}}\sum_{j=1}^{n}\theta(i)\theta(j)\tilde{P}_{\mathrm{max}}\leq \tilde{P}_{\mathrm{max}}\theta_{\mathrm{max}}\|\theta\|_{1}.
	\end{align*}
Set $t=\frac{\alpha+1+\sqrt{(\alpha+1)(\alpha+19)}}{3}\sqrt{\tilde{P}_{\mathrm{max}}\theta_{\mathrm{max}}\|\theta\|_{1}\mathrm{log}(n)}$ for any $\alpha>0$, combine Theorem \ref{Bern} with $\sigma^{2}\leq \tilde{P}_{\mathrm{max}}\theta_{\mathrm{max}}\|\theta\|_{1}, R=1, d=n$, we have
\begin{align*}
\mathbb{P}(\|W\|\geq t)&=\mathbb{P}(\|\sum_{1\leq i<j\leq n}W^{(i,j)}\|\geq t)\leq n\cdot \mathrm{exp}(\frac{-t^{2}/2}{\sigma^{2}+Rt/3})\\
&\leq n\cdot\mathrm{exp}(\frac{-(\alpha+1)\mathrm{log}(n)}{\frac{18}{(\sqrt{\alpha+1}+\sqrt{\alpha+19})^{2}}+\frac{2\sqrt{\alpha+1}}{\sqrt{\alpha+1}+\sqrt{\alpha+19}}\sqrt{\frac{\mathrm{log}(n)}{\tilde{P}_{\mathrm{max}}\theta_{\mathrm{max}}\|\theta\|_{1}}}})\leq \frac{1}{n^{\alpha}},
\end{align*}
where we have used Assumption (A2) such that $\frac{18}{(\sqrt{\alpha+1}+\sqrt{\alpha+19})^{2}}+\frac{2\sqrt{\alpha+1}}{\sqrt{\alpha+1}+\sqrt{\alpha+19}}\sqrt{\frac{\mathrm{log}(n)}{\tilde{P}_{\mathrm{max}}\theta_{\mathrm{max}}\|\theta\|_{1}}}\leq \frac{18}{(\sqrt{\alpha+1}+\sqrt{\alpha+19})^{2}}+\frac{2\sqrt{\alpha+1}}{\sqrt{\alpha+1}+\sqrt{\alpha+19}}=1$.
\end{proof}
\subsection{Proof of Lemma \ref{rowwiseerrorDCMM}}
\begin{proof}
The proof is similar as that of Lemma \ref{rowwiseerror}, so we omit most details. Since $\mathbb{E}(A(i,j)-\Omega(i,j))=0$, $\mathbb{E}[(A(i,j)-\Omega(i,j))^{2}]\leq \theta(i)\theta(j)\tilde{P}_{\mathrm{max}}\leq \theta^{2}_{\mathrm{max}}\tilde{P}_{\mathrm{max}}$,
$\frac{1}{\theta_{\mathrm{max}}\sqrt{\tilde{P}_{\mathrm{max}} n/(\mu \mathrm{log}(n))}}\leq O(1)$ holds by Assumption (A2) where $\mu=\frac{n\|U\|^{2}_{2\rightarrow\infty}}{K}$. By Theorem 4.2.1. \cite{chen2020spectral}, with high probability, we have
\begin{align*}
\|\hat{U}\mathrm{sgn}(H)-U\|_{2\rightarrow\infty}\leq C\frac{\theta_{\mathrm{max}}\sqrt{\tilde{P}_{\mathrm{max}}K}(\kappa(\Omega)\sqrt{\mu}+\sqrt{\mathrm{log}(n)})}{\sigma_{K}(\Omega)},
\end{align*}
provided that $c_{*}\sigma_{K}(\Omega)\geq \theta_{\mathrm{max}}\sqrt{\tilde{P}_{\mathrm{max}}n\mathrm{log}(n)}$ for some sufficiently small constant $c_{*}$. By Lemma H.1 of \cite{MaoSVM}, we know that  $\|U\|^{2}_{2\rightarrow\infty}\leq \frac{\theta^{2}_{\mathrm{max}}}{\lambda_{K}(\Pi'\Theta^{2}\Pi)}\leq \frac{\theta^{2}_{\mathrm{max}}}{\theta^{2}_{\mathrm{min}}\lambda_{K}(\Pi'\Pi)}$ under $DCMM_{n}(K,\tilde{P},\Pi,\Theta)$, which gives
\begin{align*}
\|\hat{U}\mathrm{sgn}(H)-U\|_{2\rightarrow\infty}\leq C\frac{\theta_{\mathrm{max}}\sqrt{\tilde{P}_{\mathrm{max}}K}(\frac{\theta_{\mathrm{max}}\kappa(\Omega)}{\theta_{\mathrm{min}}}\sqrt{\frac{n}{K\lambda_{K}(\Pi'\Pi)}}+\sqrt{\mathrm{log}(n)})}{\sigma_{K}(\Omega)},
\end{align*}
\begin{rem}
Similar as the proof of Lemma \ref{rowwiseerror},
by Theorem 4.2 of \cite{cape2019the}, when $\sigma_{K}(\Omega)\geq 4\|A-\Omega\|_{\infty}$, we have
\begin{align*}
\|\hat{U}\mathrm{sgn}(H)-U\|_{2\rightarrow\infty}\leq 14\frac{\|A-\Omega\|_{\infty}}{\sigma_{K}(\Omega)}\|U\|_{2\rightarrow\infty}\leq\frac{14\theta_{\mathrm{max}}\|A-\Omega\|_{\infty}}{\theta_{\mathrm{min}}\sigma_{K}(\Omega)\sqrt{\lambda_{K}(\Pi'\Pi)}}.
\end{align*}
Let $y=(y_{1},y_{2},\ldots, y_{n})'$ be any $n\times 1$ vector, by Bernstein inequality, for any $t\geq 0$ and $i\in[n]$, we have
\begin{align*}
\mathbb{P}(|\sum_{j=1}^{n}(A(i,j)-\Omega(i,j))y(j)|>t)\leq 2\mathrm{exp}(-\frac{t^{2}/2}{\sum_{j=1}^{n}\Omega(i,j)y^{2}(j)+\frac{t\|y\|_{\infty}}{3}}).
\end{align*}
By the proof of Lemma \ref{BoundAOmegaDCMM}, we have $\Omega(i,j)\leq\theta(i)\theta(j)\tilde{P}_{\mathrm{max}}$, which gives $\sum_{j=1}^{n}\Omega(i,j)\leq \tilde{P}_{\mathrm{max}}\theta_{\mathrm{max}}\|\theta\|_{1}$. Set $y(j)$ as $1$ or $-1$ such that $(A(i,j)-\Omega(i,j))y(j)=|A(i,j)-\Omega(i,j)|$, we have
\begin{align*}
\mathbb{P}(\|A-\Omega\|_{\infty}>t)\leq 2\mathrm{exp}(-\frac{t^{2}/2}{\tilde{P}_{\mathrm{max}}\theta_{\mathrm{max}}\|\theta\|_{1}+\frac{t}{3}}).
\end{align*}
Set $t=\frac{\alpha+1+\sqrt{(\alpha+1)(\alpha+19)}}{3}\sqrt{\tilde{P}_{\mathrm{max}}\theta_{\mathrm{max}}\|\theta\|_{1}\mathrm{log}(n)}$ for any $\alpha>0$, by Assumption (A2), we have
\begin{align*}
\mathbb{P}(\|A-\Omega\|_{\infty}>t)\leq 2\mathrm{exp}(-\frac{t^{2}/2}{\tilde{P}_{\mathrm{max}}\theta_{\mathrm{max}}\|\theta\|_{1}+\frac{t}{3}})\leq n^{-\alpha}.
\end{align*}
Hence, when $\sigma_{K}(\Omega)\geq C_{0}\sqrt{\tilde{P}_{\mathrm{max}}\theta_{\mathrm{max}}\|\theta\|_{1}\mathrm{log}(n)}$ where $C_{0}=4\frac{\alpha+1+\sqrt{(\alpha+1)(\alpha+19)}}{3}$, with probability at least $1-o(n^{-\alpha})$,
\begin{align*}
\|\hat{U}\mathrm{sgn}(H)-U\|_{2\rightarrow\infty}\leq C\frac{\theta_{\mathrm{max}}\sqrt{\tilde{P}_{\mathrm{max}}\theta_{\mathrm{max}}\|\theta\|_{1}\mathrm{log}(n)}}{\theta_{\mathrm{min}}\sigma_{K}(\Omega)\sqrt{\lambda_{K}(\Pi'\Pi)}}.
\end{align*}
Meanwhile, since $\sqrt{\tilde{P}_{\mathrm{max}}\theta_{\mathrm{max}}\|\theta\|_{1}\mathrm{log}(n)}\leq \theta_{\mathrm{max}}\sqrt{\tilde{P}_{\mathrm{max}}n\mathrm{log}(n)}$, for convenience, we let the lower bound requirement of $\sigma_{K}(\Omega)$ be $C\theta_{\mathrm{max}}\sqrt{\tilde{P}_{\mathrm{max}}n\mathrm{log}(n)}$.
\end{rem}
Similar as the proof of Lemma \ref{rowwiseerror}, we have
\begin{align*}	
&\|\hat{U}\hat{U}'-UU'\|_{2\rightarrow\infty}=\mathrm{max}_{i\in[n]}\|e'_{i}(UU'-\hat{U}\hat{U}')\|_{F}\leq2\|U-\hat{U}\mathrm{sgn}(H)\|_{2\rightarrow\infty}\\
&\leq C\frac{\theta_{\mathrm{max}}\sqrt{\tilde{P}_{\mathrm{max}}K}(\frac{\theta_{\mathrm{max}}\kappa(\Omega)}{\theta_{\mathrm{min}}}\sqrt{\frac{n}{K\lambda_{K}(\Pi'\Pi)}}+\sqrt{\mathrm{log}(n)})}{\sigma_{K}(\Omega)}\leq C\frac{\theta_{\mathrm{max}}\sqrt{\tilde{P}_{\mathrm{max}}K}(\frac{\theta_{\mathrm{max}}\kappa(\Omega)}{\theta_{\mathrm{min}}}\sqrt{\frac{n}{K\lambda_{K}(\Pi'\Pi)}}+\sqrt{\mathrm{log}(n)})}{\theta^{2}_{\mathrm{min}}\sigma_{K}(\tilde{P})\lambda_{K}(\Pi'\Pi)},
\end{align*}
where the last inequality holds since $\sigma_{K}(\Omega)=\sigma_{K}(\Theta \Pi \tilde{P}\Pi'\Theta)\geq \theta^{2}_{\mathrm{min}}\sigma_{K}(\Pi P\Pi')=\theta^{2}_{\mathrm{min}}\sigma_{K}(\Pi'\Pi \tilde{P})\geq \theta^{2}_{\mathrm{min}}\sigma_{K}(\tilde{P})\sigma_{K}(\Pi'\Pi)=\theta^{2}_{\mathrm{min}}\sigma_{K}(\tilde{P})\lambda_{K}(\Pi'\Pi)$. And this bound is $C\frac{\theta_{\mathrm{max}}\sqrt{\tilde{P}_{\mathrm{max}}\theta_{\mathrm{max}}\|\theta\|_{1}\mathrm{log}(n)}}{\theta^{3}_{\mathrm{min}}\sigma_{K}(\tilde{P})\lambda^{1.5}_{K}(\Pi'\Pi)}$ if we use Theorem 4.2 of \cite{cape2019the}.
\end{proof}
\subsection{Proof of Theorem \ref{MainDCMM}}
\begin{proof}
For $i\in[n]$, recall that $Z_{*}=Y_{*}J_{*}\equiv N^{-1}_{U}N_{M}\Pi, \hat{Z}_{*}=\hat{Y}_{*}\hat{J}_{*}, \Pi(i,:)=\frac{Z(i,:)}{\|Z(i,:)\|_{1}}$ and $\hat{\Pi}_{*}(i,:)=\frac{\hat{Z}_{*}(i,:)}{\|\hat{Z}_{*}(i,:)\|_{1}}$, where $N_{M}$ and $M$ are defined in the proof of Lemma \ref{IC} such that $U=\Theta M\equiv\Theta\Pi B_{*}$ and $N_{M}(i,i)=\frac{1}{\|M(i,:)\|_{F}}$, we have
\begin{align*}	\|e'_{i}(\hat{\Pi}_{*}-\Pi\mathcal{P}_{*})\|_{1}\leq\frac{2\|e'_{i}(\hat{Z}_{*}-Z_{*}\mathcal{P}_{*})\|_{1}}{\|e'_{i}Z_{*}\|_{1}}\leq \frac{2\sqrt{K}\|e'_{i}(\hat{Z}_{*}-Z_{*}\mathcal{P}_{*})\|_{F}}{\|e'_{i}Z_{*}\|_{1}}.
\end{align*}
Now, we provide a lower bound of $\|e'_{i}Z_{*}\|_{1}$ as below
\begin{align*}
\|e'_{i}Z_{*}\|_{1}&=\|e'_{i}N^{-1}_{U}N_{M}\Pi\|_{1}=\|N_{U}^{-1}(i,i)e'_{i}N_{M}\Pi\|_{1}=N^{-1}_{U}(i,i)\|N_{M}(i,i)e'_{i}\Pi\|_{1}=\frac{N_{M}(i,i)}{N_{U}(i,i)}\\
&=\|U(i,:)\|_{F}N_{M}(i,i)=\|U(i,:)\|_{F}\frac{1}{\|M(i,:)\|_{F}}=\|U(i,:)\|_{F}\frac{1}{\|e'_{i}M\|_{F}}=\|U(i,:)\|_{F}\frac{1}{\|e'_{i}\Theta^{-1}U\|_{F}}\\
&=\|U(i,:)\|_{F}\frac{1}{\|\Theta^{-1}(i,i)e'_{i}U\|_{F}}=\theta(i)\geq \theta_{\mathrm{min}}.
\end{align*}
Therefore, by Lemma \ref{boundZDCMM}, we have
\begin{align*}	\|e'_{i}(\hat{\Pi}_{*}-\Pi\mathcal{P}_{*})\|_{1}&\leq \frac{2\sqrt{K}\|e'_{i}(\hat{Z}_{*}-Z_{*}\mathcal{P}_{*})\|_{F}}{\|e'_{i}Z_{*}\|_{1}}\leq\frac{2\sqrt{K}\|e'_{i}(\hat{Z}_{*}-Z_{*}\mathcal{P}_{*})\|_{F}}{\theta_{\mathrm{min}}}\\
&=O(\frac{\theta^{15}_{\mathrm{max}}K^{5}\varpi\kappa^{4.5}(\Pi'\Pi)\lambda^{1.5}_{1}(\Pi'\Pi)}{\theta^{15}_{\mathrm{min}}\pi_{\mathrm{min}}}).
\end{align*}
\end{proof}
\subsection{Proof of Corollary \ref{AddConditionsDCMM}}
\begin{proof}
Under  conditions of Corollary \ref{AddConditionsDCMM}, we have
\begin{align*}
\mathrm{max}_{i\in[n]}\|e'_{i}(\hat{\Pi}_{*}-\Pi\mathcal{P}_{*})\|_{1}=O(\frac{\theta^{15}_{\mathrm{max}}K^{5}\varpi\kappa^{4.5}(\Pi'\Pi)\lambda^{1.5}_{1}(\Pi'\Pi)}{\theta^{15}_{\mathrm{min}}\pi_{\mathrm{min}}})=O(\frac{\theta^{15}_{\mathrm{max}}\varpi\sqrt{n}}{\theta^{15}_{\mathrm{min}}}).
\end{align*}
Under  conditions of Corollary \ref{AddConditionsDCMM}, Lemma \ref{rowwiseerrorDCMM} gives $\varpi=O(\frac{\theta_{\mathrm{max}}\sqrt{\theta_{\mathrm{max}}\|\theta\|_{1}\mathrm{log}(n)}}{\theta^{3}_{\mathrm{min}}\sigma_{K}(\tilde{P})n^{1.5}})$,
which gives that
\begin{align*}
\mathrm{max}_{i\in[n]}\|e'_{i}(\hat{\Pi}_{*}-\Pi\mathcal{P}_{*})\|_{1}=O(\frac{\theta^{15}_{\mathrm{max}}\varpi\sqrt{n}}{\theta^{15}_{\mathrm{min}}})=O(\frac{\theta^{16}_{\mathrm{max}}\sqrt{\theta_{\mathrm{max}}\|\theta\|_{1}\mathrm{log}(n)}}{\theta^{18}_{\mathrm{min}}\sigma_{K}(\tilde{P})n}).
\end{align*}
By basic algebra, this corollary follows.
\end{proof}
\subsection{Basic properties of $\Omega$ under DCMM}
\begin{lem}\label{boundUeta}
Under $DCMM_{n}(K,\tilde{P},\Pi,\Theta)$, we have
\begin{align*}
&\|U(i,:)\|_{F}\geq \frac{\theta_{\mathrm{min}}}{\theta_{\mathrm{max}}\sqrt{K\lambda_{1}(\Pi'\Pi)}}\mathrm{~for~}i\in[n], \mathrm{and~}\eta\geq\frac{\theta^{4}_{\mathrm{min}}\pi_{\mathrm{min}}}{\theta^{4}_{\mathrm{max}}K\lambda_{1}(\Pi'\Pi)},
\end{align*}
where $\eta=\mathrm{min}_{k\in[K]}((U_{*}(\mathcal{I},:)U'_{*}(\mathcal{I},:))^{-1}\mathbf{1})(k)$.
\end{lem}
\begin{proof}
	Since $I=U'U=U'(\mathcal{I},:)\Theta^{-1}(\mathcal{I},\mathcal{I})\Pi'\Theta^{2}\Pi\Theta^{-1}(\mathcal{I},\mathcal{I})U(\mathcal{I},:)$ by the proof of Lemma \ref{IC}, we have $((\Theta^{-1}(\mathcal{I},\mathcal{I})U(\mathcal{I},:))((\Theta^{-1}(\mathcal{I},\mathcal{I})U(\mathcal{I},:))')^{-1}=\Pi'\Theta^{2}\Pi$,
which gives that
\begin{align*}	\mathrm{min}_{k}\|e'_{k}(\Theta^{-1}(\mathcal{I},\mathcal{I})U(\mathcal{I},:))\|_{F}^{2}&=\mathrm{min}_{k}e'_{k}(\Theta^{-1}(\mathcal{I},\mathcal{I})U(\mathcal{I},:))(\Theta^{-1}(\mathcal{I},\mathcal{I})U(\mathcal{I},:))'e_{k}\\
&\geq\mathrm{min}_{\|x\|=1}x'(\Theta^{-1}(\mathcal{I},\mathcal{I})U(\mathcal{I},:))(\Theta^{-1}(\mathcal{I},\mathcal{I})U(\mathcal{I},:))'x\\ &=\lambda_{K}((\Theta^{-1}(\mathcal{I},\mathcal{I})U(\mathcal{I},:))(\Theta^{-1}(\mathcal{I},\mathcal{I})U(\mathcal{I},:))')=\frac{1}{\lambda_{1}(\Pi'\Theta^{2}\Pi)},
\end{align*}
where $x$ is a $K\times 1$ vector whose $l_{2}$ norm is 1. Then, for $i\in[n]$, we have
\begin{align*} \|U(i,:)\|_{F}&=\|\theta_{i}\Pi(i,:)\Theta^{-1}(\mathcal{I},\mathcal{I})U(\mathcal{I},:)\|_{F}=\theta_{i}\|\Pi(i,:)\Theta^{-1}(\mathcal{I},\mathcal{I})U(\mathcal{I},:)\|_{F}\\
&\geq\theta_{i}\mathrm{min}_{i}\|\Pi(i,:)\|_{F}\mathrm{min}_{i}\|e'_{i}(\Theta^{-1}(\mathcal{I},\mathcal{I})U(\mathcal{I},:))\|_{F}\geq\theta_{i}\mathrm{min}_{i}\|e'_{i}(\Theta^{-1}(\mathcal{I},\mathcal{I})U(\mathcal{I},:))\|_{F}/\sqrt{K}\\ &\geq\frac{\theta_{i}}{\sqrt{K\lambda_{1}(\Pi'\Theta^{2}\Pi)}}\geq\frac{\theta_{\mathrm{min}}}{\theta_{\mathrm{max}}\sqrt{K\lambda_{1}(\Pi'\Pi)}},
\end{align*}
where we use the fact that $\mathrm{min}_{i}\|\Pi(i,:)\|_{F}\geq \frac{1}{\sqrt{K}}$ since $\sum_{k=1}^{K}\Pi(i,k)=1$ and all entries of $\Pi$ are nonnegative.

Since $U_{*}=N_{U}U$, we have
\begin{align*}
(U_{*}(\mathcal{I},:)U'_{*}(\mathcal{I},:))^{-1}&=N_{U}^{-1}(\mathcal{I},\mathcal{I})\Theta^{-1}(\mathcal{I},\mathcal{I})\Pi'\Theta^{2}\Pi\Theta^{-1}(\mathcal{I},\mathcal{I})N_{U}^{-1}(\mathcal{I},\mathcal{I})\\
&\geq \frac{\theta^{2}_{\mathrm{min}}}{\theta^{2}_{\mathrm{max}}N^{2}_{U,\mathrm{max}}}\Pi'\Pi\geq\frac{\theta^{4}_{\mathrm{min}}}{\theta^{4}_{\mathrm{max}}K\lambda_{1}(\Pi'\Pi)}\Pi'\Pi,
\end{align*}
where we set $N_{U,\mathrm{max}}=\mathrm{max}_{i\in[n]}N_{U}(i,i)$ and we have used the facts that $N_{U}, \Theta$ are diagonal matrices, and $N_{U,\mathrm{max}}\leq \frac{\theta_{\mathrm{max}}\sqrt{K\lambda_{1}(\Pi'\Pi)}}{\theta_{\mathrm{min}}}$. Then we have
\begin{align*}
&\eta=\mathrm{min}_{k\in[K]}((U_{*}(\mathcal{I},:)U'_{*}(\mathcal{I},:))^{-1}\mathbf{1})(k)\geq \frac{\theta^{4}_{\mathrm{min}}}{\theta^{4}_{\mathrm{max}}K\lambda_{1}(\Pi'\Pi)}\mathrm{min}_{k\in[K]}e'_{k}\Pi'\Pi\mathbf{1}\\
&=\frac{\theta^{4}_{\mathrm{min}}}{\theta^{4}_{\mathrm{max}}K\lambda_{1}(\Pi'\Pi)}\mathrm{min}_{k\in[K]}e'_{k}\Pi'\mathbf{1}=\frac{\theta^{4}_{\mathrm{min}}\pi_{\mathrm{min}}}{\theta^{4}_{\mathrm{max}}K\lambda_{1}(\Pi'\Pi)}.
\end{align*}
\end{proof}
\subsection{Bounds between Ideal SVM-cone-DCMMSB and SVM-cone-DCMMSB}
Next lemma focus on the 2nd step of SVM-cone-DCMMSB and is the corner stone to characterize the behaviors of SVM-cone-DCMMSB.
\begin{lem}\label{boundCDCMM}
	Under $DCMM_{n}(K,\tilde{P},\Pi,\Theta)$, when conditions of Lemma \ref{rowwiseerrorDCMM} hold, there exists a permutation matrix $\mathcal{P}_{*}\in\mathbb{R}^{K\times K}$ such that with probability at least $1-o(n^{-\alpha})$, we have
\begin{align*} \mathrm{max}_{1\leq k\leq K}\|e'_{k}(\hat{U}_{*,2}(\mathcal{\hat{I}}_{*},:)-\mathcal{P}'_{*}U_{*,2}(\mathcal{I},:))\|_{F}=O(\frac{K^{3}\theta^{11}_{\mathrm{max}}\varpi\kappa^{3}(\Pi'\Pi)\lambda^{1.5}_{1}(\Pi'\Pi)}{\theta^{11}_{\mathrm{min}}\pi_{\mathrm{min}}}),
\end{align*}
where $U_{*,2}=U_{*}U', \hat{U}_{*,2}=\hat{U}_{*}\hat{U}'$, i.e., $U_{*,2}, \hat{U}_{*,2}$ are the row-normalized versions of $UU'$ and $\hat{U}\hat{U}'$, respectively.
\end{lem}
\begin{proof}
Lemma G.1. of \cite{MaoSVM} says that using $\hat{U}_{*,2}$ as input of the SVM-cone algorithm returns same result as using $\hat{U}_{*}$ as input. By Lemma F.1 of \cite{MaoSVM}, there exists a permutation matrix $\mathcal{P}_{*}\in\mathbb{R}^{K\times K}$ such that
\begin{align*}		
\mathrm{max}_{k\in[K]}\|e'_{k}(\hat{U}_{*,2}(\mathcal{\hat{I}}_{*},:)-\mathcal{P}'_{*}U_{*,2}(\mathcal{I},:))\|_{F}= O(\frac{\sqrt{K}\zeta\epsilon_{*}}{\lambda^{1.5}_{K}(U_{*,2}(\mathcal{I},:))U'_{*,2}(\mathcal{I},:)}),
\end{align*}
where $\zeta\leq\frac{4K}{\eta\lambda^{1.5}_{K}(U_{*,2}(\mathcal{I},:)U'_{*,2}(\mathcal{I},:))}=O(\frac{K}{\eta\lambda^{1.5}_{K}(U_{*}(\mathcal{I},:)U'_{*}(\mathcal{I},:))}), \epsilon_{*}=\mathrm{max}_{i\in[n]}\|\hat{U}_{*,2}(i,:)-U_{*,2}(i,:)\|_{F}$ and $\eta=\mathrm{min}_{1\leq k\leq K}((U_{*}(\mathcal{I},:)U'_{*}(\mathcal{I},:))^{-1}\mathbf{1})(k)$. Next we give upper bound of $\epsilon_{*}$.
\begin{align*}		&\|\hat{U}_{*,2}(i,:)-U_{*,2}(i,:)\|_{F}=\|\frac{\hat{U}_{2}(i,:)\|U_{2}(i,:)\|_{F}-U_{2}(i,:)\|\hat{U}_{2}(i,:)\|_{F}}{\|\hat{U}_{2}(i,:)\|_{F}\|U_{2}(i,:)\|_{F}}\|_{F}\leq\frac{2\|\hat{U}_{2}(i,:)-U_{2}(i,:)\|_{F}}{\|U_{2}(i,:)\|_{F}}\\
&\leq \frac{2\|\hat{U}_{2}-U_{2}\|_{2\rightarrow\infty}}{\|U_{2}(i,:)\|_{F}}\leq\frac{2\varpi}{\|U_{2}(i,:)\|_{F}}=\frac{2\varpi}{\|(UU')(i,:)\|_{F}}=\frac{2\varpi}{\|U(i,:)U'\|_{F}}=\frac{2\varpi}{\|U(i,:)\|_{F}}\leq \frac{2\theta_{\mathrm{max}}\varpi\sqrt{K\lambda_{1}(\Pi'\Pi)}}{\theta_{\mathrm{min}}},
\end{align*}
where the last inequality holds by Lemma \ref{boundUeta}.
Then, we have $\epsilon_{*}=O(\frac{\theta_{\mathrm{max}}\varpi\sqrt{K\lambda_{1}(\Pi'\Pi)}}{\theta_{\mathrm{min}}})$. By Lemma H.2. of \cite{MaoSVM}, $\lambda_{K}(U_{*}(\mathcal{I},:)U'_{*}(\mathcal{I},:))\geq \frac{\theta^{2}_{\mathrm{min}}\kappa^{-1}(\Pi'\Pi)}{\theta^{2}_{\mathrm{max}}}$.  By the lower bound of $\eta$ given in Lemma \ref{boundUeta}, we have
\begin{align*} \mathrm{max}_{k\in[K]}\|e'_{k}(\hat{U}_{*,2}(\mathcal{\hat{I}}_{*},:)-\mathcal{P}'_{*}U_{*,2}(\mathcal{I},:))\|_{F}=O(\frac{K^{3}\theta^{11}_{\mathrm{max}}\varpi\kappa^{3}(\Pi'\Pi)\lambda^{1.5}_{1}(\Pi'\Pi)}{\theta^{11}_{\mathrm{min}}\pi_{\mathrm{min}}}).
\end{align*}
\end{proof}
Next lemma focuses on the 3rd step of SVM-cone-DCMMSB and bounds $\mathrm{max}_{i\in[n]}\|e'_{i}(\hat{Z}_{*}-Z_{*}\mathcal{P}_{*})\|_{F}$.
\begin{lem}\label{boundZDCMM}
Under $DCMM_{n}(K,\tilde{P},\Pi,\Theta)$, when conditions of Lemma \ref{rowwiseerrorDCMM} hold,, with probability at least $1-o(n^{-\alpha})$, we have
\begin{align*}
\mathrm{max}_{i\in[n]}\|e'_{i}(\hat{Z}_{*}-Z_{*}\mathcal{P}_{*})\|_{F}=O(\frac{\theta^{15}_{\mathrm{max}}K^{4.5}\varpi\kappa^{4.5}(\Pi'\Pi)\lambda^{1.5}_{1}(\Pi'\Pi)}{\theta^{14}_{\mathrm{min}}\pi_{\mathrm{min}}}).
\end{align*}
\end{lem}
\begin{proof}
For $i\in[n]$, since $Z_{*}=Y_{*}J_{*}, \hat{Z}_{*}=\hat{Y}_{*}\hat{J}_{*}$ and $J_{*}, \hat{J}_{*}$ are diagonal matrices, we have
\begin{align*} &\|e'_{i}(\hat{Z}_{*}-Z_{*}\mathcal{P}_{*})\|_{F}=\|e'_{i}(\mathrm{max}(0, \hat{Y}_{*}\hat{J}_{*})-Y_{*}J_{*}\mathcal{P}_{*})\|_{F}\leq \|e'_{i}(\hat{Y}_{*}\hat{J}_{*}-Y_{*}J_{*}\mathcal{P}_{*})\|_{F}\\
&=\|e'_{i}(\hat{Y}_{*}-Y_{*}\mathcal{P}_{*})\hat{J}_{*}+e'_{i}Y_{*}\mathcal{P}_{*}(\hat{J}_{*}-\mathcal{P}'_{*}J_{*}\mathcal{P}_{*})\|_{F}\leq\|e'_{i}(\hat{Y}_{*}-Y_{*}\mathcal{P}_{*})\|_{F}\|\hat{J}_{*}\|+\|e'_{i}Y_{*}\mathcal{P}_{*}\|_{F}\|\hat{J}_{*}-\mathcal{P}'_{*}J_{*}\mathcal{P}_{*}\|\\
&=\|e'_{i}(\hat{Y}_{*}-Y_{*}\mathcal{P}_{*})\|_{F}\|\hat{J}_{*}\|+\|e'_{i}Y_{*}\|_{F}\|\hat{J}_{*}-\mathcal{P}'_{*}J_{*}\mathcal{P}_{*}\|
=\|e'_{i}(\hat{Y}_{*}-Y_{*}\mathcal{P}_{*})\|_{F}\|\hat{J}_{*}\|+\|e'_{i}Y_{*}\|_{F}\|J_{*}-\mathcal{P}_{*}\hat{J}_{*}\mathcal{P}'_{*}\|.
\end{align*}
Therefore, the bound of $\|e'_{i}(\hat{Z}_{*}-Z_{*}\mathcal{P}_{*})\|_{F}$ can be obtained as long as we bound $\|e'_{i}(\hat{Y}_{*}-Y_{*}\mathcal{P}_{*})\|_{F}, \|\hat{J}_{*}\|, \|e'_{i}Y_{*}\|_{F}$ and $\|J_{*}-\mathcal{P}_{*}\hat{J}_{*}\mathcal{P}'_{*}\|$. We bound the four terms as below:
\begin{itemize}
  \item we bound $\|e'_{i}(\hat{Y}_{*}-Y_{*}\mathcal{P}_{*})\|_{F}$ first. Set $U_{*}(\mathcal{I},:)=B_{*},\hat{U}_{*}(\mathcal{\hat{I}}_{*},:)=\hat{B}_{*}, U_{*,2}(\mathcal{I},:)=B_{2*}, \hat{U}_{*,2}(\mathcal{\hat{I}}_{*},:)=\hat{B}_{2*}$ for convenience. For $i\in[n]$, we have
\begin{align*} &\|e'_{i}(\hat{Y}_{*}-Y_{*}\mathcal{P}_{*})\|_{F}=\|e'_{i}(\hat{U}\hat{B}'_{*}(\hat{B}_{*}\hat{B}'_{*})^{-1}-UB'_{*}(B_{*}B'_{*})^{-1}\mathcal{P}_{*})\|_{F}\\ &=\|e'_{i}(\hat{U}-U(U'\hat{U}))\hat{B}'_{*}(\hat{B}_{*}\hat{B}'_{*})^{-1}+e'_{i}(U(U'\hat{U})\hat{B}'_{*}(\hat{B}_{*}\hat{B}'_{*})^{-1}-U(U'\hat{U})(\mathcal{P}'_{*}(B_{*}B'_{*})(B'_{*})^{-1}(U'\hat{U}))^{-1})\|_{F}\\ &\leq\|e'_{i}(\hat{U}-U(U'\hat{U}))\hat{B}'_{*}(\hat{B}_{*}\hat{B}'_{*})^{-1}\|_{F}+\|e'_{i}U(U'\hat{U})(\hat{B}'_{*}(\hat{B}_{*}\hat{B}'_{*})^{-1}-(\mathcal{P}'_{*}(B_{*}B'_{*})(B'_{*})^{-1}(U'\hat{U}))^{-1})\|_{F}\\
&\leq\|e'_{i}(\hat{U}-U(U'\hat{U}))\|_{F}\|\hat{B}^{-1}_{*}\|_{F}+\|e'_{i}U(U'\hat{U})(\hat{B}'_{*}(\hat{B}_{*}\hat{B}'_{*})^{-1}-(\mathcal{P}'_{*}(B_{*}B'_{*})(B'_{*})^{-1}(U'\hat{U}))^{-1})\|_{F}\\	&\leq \sqrt{K}\|e'_{i}(\hat{U}-U(U'\hat{U}))\|_{F}/\sqrt{\lambda_{K}(\hat{B}_{*}\hat{B}'_{*})}+\|e'_{i}U(U'\hat{U})(\hat{B}^{-1}_{*}-(\mathcal{P}_{*}'B_{*}(U'\hat{U}))^{-1})\|_{F}\\
&\overset{(i)}{=}\sqrt{K}\|e'_{i}(\hat{U}\hat{U}'-UU')\hat{U}\|_{F}O(\frac{\theta_{\mathrm{max}}\sqrt{\kappa(\Pi'\Pi)}}{\theta_{\mathrm{min}}})+\|e'_{i}U(U'\hat{U})(\hat{B}^{-1}_{*}-(\mathcal{P}_{*}'B_{*}(U'\hat{U}))^{-1})\|_{F}\\
&\leq \sqrt{K}\|e'_{i}(\hat{U}\hat{U}'-UU')\|_{F}O(\frac{\theta_{\mathrm{max}}\sqrt{\kappa(\Pi'\Pi)}}{\theta_{\mathrm{min}}})+\|e'_{i}U(U'\hat{U})(\hat{B}^{-1}_{*}-(\mathcal{P}'_{*}B_{*}(U'\hat{U}))^{-1})\|_{F}\\
&\leq \sqrt{K}\varpi O(\frac{\theta_{\mathrm{max}}\sqrt{\kappa(\Pi'\Pi)}}{\theta_{\mathrm{min}}})+\|e'_{i}U(U'\hat{U})(\hat{B}^{-1}_{*}-(\mathcal{P}'_{*}B_{*}(U'\hat{U}))^{-1})\|_{F}\\	&=O(\varpi\frac{\theta_{\mathrm{max}}\sqrt{K\kappa(\Pi'\Pi)}}{\theta_{\mathrm{min}}})+\|e'_{i}U(U'\hat{U})(\hat{B}^{-1}_{*}-(\mathcal{P}'_{*}B_{*}(U'\hat{U}))^{-1})\|_{F},
\end{align*}
where we have used similar idea in the proof of Lemma VII.3 in \cite{mao2020estimating} such that apply $O(\frac{1}{\lambda_{K}(B_{*}B'_{*})})$ to estimate $\frac{1}{\lambda_{K}(\hat{B}_{*}\hat{B}'_{*})}$.

Now we aim to bound $\|e'_{i}U(U'\hat{U})(\hat{B}^{-1}_{*}-(\mathcal{P}_{*}'B_{*}(U'\hat{U}))^{-1})\|_{F}$. For convenience, set $T=U'\hat{U}, S=\mathcal{P}_{*}'B_{*}T$. We have
\begin{align}	&\|e'_{i}U(U'\hat{U})(\hat{B}^{-1}_{*}-(\mathcal{P}'_{*}B_{*}(U'\hat{U}))^{-1})\|_{F}=\|e'_{i}UTS^{-1}(S-\hat{B}_{*})\hat{B}^{-1}_{*}\|_{F}\notag\\	&\leq\|e'_{i}UTS^{-1}(S-\hat{B}_{*})\|_{F}\|\hat{B}^{-1}_{*}\|_{F}\leq\|e'_{i}UTS^{-1}(S-\hat{B}_{*})\|_{F}\frac{\sqrt{K}}{|\lambda_{K}(\hat{B}_{*})|}\notag\\	&=\|e'_{i}UTS^{-1}(S-\hat{B}_{*})\|_{F}\frac{\sqrt{K}}{\sqrt{\lambda_{K}(\hat{B}_{*}\hat{B}'_{*})}}\leq\|e'_{i}UTS^{-1}(S-\hat{B}_{*})\|_{F}O(\frac{\theta_{\mathrm{max}}\sqrt{K\kappa(\Pi'\Pi)}}{\theta_{\mathrm{min}}})\notag\\	&=\|e'_{i}UTT^{-1}B'_{*}(B_{*}B'_{*})^{-1}\mathcal{P}_{*}(S-\hat{B}_{*})\|_{F}O(\frac{\theta_{\mathrm{max}}\sqrt{K\kappa(\Pi'\Pi)}}{\theta_{\mathrm{min}}})\notag\\	&=\|e'_{i}UB'_{*}(B_{*}B'_{*})^{-1}\mathcal{P}_{*}(S-\hat{B}_{*})\|_{F}O(\frac{\theta_{\mathrm{max}}\sqrt{K\kappa(\Pi'\Pi)}}{\theta_{\mathrm{min}}})\notag\\ &=\|e'_{i}Y_{*}\mathcal{P}_{*}(S-\hat{B}_{*})\|_{F}O(\frac{\theta_{\mathrm{max}}\sqrt{K\lambda_{1}(\Pi'\Pi)}}{\theta_{\mathrm{min}}})\leq \|e'_{i}Y_{*}\|_{F}\|S-\hat{B}_{*}\|_{F}O(\frac{\theta_{\mathrm{max}}\sqrt{K\lambda_{1}(\Pi'\Pi)}}{\theta_{\mathrm{min}}})\notag\\
&\overset{\mathrm{By~Eq~}(\ref{boundYstar})}{\leq}\frac{\theta^{2}_{\mathrm{max}}\sqrt{K\lambda_{1}(\Pi'\Pi)}}{\theta^{2}_{\mathrm{min}}\lambda_{K}(\Pi'\Pi)}\mathrm{max}_{1\leq k\leq K}\|e'_{k}(S-\hat{B}_{*})\|_{F}O(\frac{\theta_{\mathrm{max}}K\sqrt{\kappa(\Pi'\Pi)}}{\theta_{\mathrm{min}}})\notag\\ &=\mathrm{max}_{1\leq k\leq K}\|e'_{k}(\hat{B}_{*}-\mathcal{P}_{*}'B_{*}U'\hat{U})\|_{F}O(\frac{\theta^{3}_{\mathrm{max}}K^{1.5}\kappa(\Pi'\Pi)}{\theta^{3}_{\mathrm{min}}\sqrt{\lambda_{K}(\Pi'\Pi)}})\notag\\
&=\mathrm{max}_{1\leq k\leq K}\|e'_{k}(\hat{B}_{*}\hat{U}'-\mathcal{P}'_{*}B_{*}U')\hat{U}\|_{F}O(\frac{\theta^{3}_{\mathrm{max}}K^{1.5}\kappa(\Pi'\Pi)}{\theta^{3}_{\mathrm{min}}\sqrt{\lambda_{K}(\Pi'\Pi)}})\notag\\ &\leq\mathrm{max}_{1\leq k\leq K}\|e'_{k}(\hat{B}_{*}\hat{U}'-\mathcal{P}'B_{*}U')\|_{F}O(\frac{\theta^{3}_{\mathrm{max}}K^{1.5}\kappa(\Pi'\Pi)}{\theta^{3}_{\mathrm{min}}\sqrt{\lambda_{K}(\Pi'\Pi)}})\notag\\
&=\mathrm{max}_{1\leq k\leq K}\|e'_{k}(\hat{B}_{2*}-\mathcal{P}_{*}'B_{2*})\|_{F}O(\frac{\theta^{3}_{\mathrm{max}}K^{1.5}\kappa(\Pi'\Pi)}{\theta^{3}_{\mathrm{min}}\sqrt{\lambda_{K}(\Pi'\Pi)}})\notag\\
&\overset{\mathrm{By~Lemma~}\ref{boundCDCMM}}{=}O(\frac{K^{4.5}\theta^{14}_{\mathrm{max}}\varpi\kappa^{4.5}(\Pi'\Pi)\lambda_{1}(\Pi'\Pi)}{\theta^{14}_{\mathrm{min}}\pi_{\mathrm{min}}})\notag.
\end{align}
Then, we have
\begin{align*}
\|e'_{i}(\hat{Y}_{*}-Y_{*}\mathcal{P}_{*})\|_{F}&\leq O(\varpi\frac{\theta_{\mathrm{max}}\sqrt{K\kappa(\Pi'\Pi)}}{\theta_{\mathrm{min}}})+\|e'_{i}U(U'\hat{U})(\hat{B}^{-1}_{*}-(\mathcal{P}'_{*}B_{*}U'\hat{U}))^{-1})\|_{F}\\
&\leq O(\varpi\frac{\theta_{\mathrm{max}}\sqrt{K\kappa(\Pi'\Pi)}}{\theta_{\mathrm{min}}})+O(\frac{K^{4.5}\theta^{14}_{\mathrm{max}}\varpi\kappa^{4.5}(\Pi'\Pi)\lambda_{1}(\Pi'\Pi)}{\theta^{14}_{\mathrm{min}}\pi_{\mathrm{min}}})\\
&=O(\frac{K^{4.5}\theta^{14}_{\mathrm{max}}\varpi\kappa^{4.5}(\Pi'\Pi)\lambda_{1}(\Pi'\Pi)}{\theta^{14}_{\mathrm{min}}\pi_{\mathrm{min}}}).
\end{align*}
  \item for $\|e'_{i}Y_{*}\|_{F}$, since $Y_{*}=UU^{-1}_{*}(\mathcal{I},:)$, we have
  \begin{align}\label{boundYstar}
        \|e'_{i}Y_{*}\|_{F}\leq \|U(i,:)\|_{F}\|U^{-1}_{*}(\mathcal{I},:)\|_{F}\leq \frac{\sqrt{K}\|U(i,:)\|_{F}}{\sqrt{\lambda_{K}(U_{*}(\mathcal{I},:)U'_{*}(\mathcal{I},:))}}\leq \frac{\theta^{2}_{\mathrm{max}}\sqrt{K\lambda_{1}(\Pi'\Pi)}}{\theta^{2}_{\mathrm{min}}\lambda_{K}(\Pi'\Pi)}.
      \end{align}
  \item for $\|\hat{J}_{*}\|$, recall that $\hat{J}_{*}=\sqrt{\mathrm{diag}(\hat{U}_{*}(\hat{I}_{*},:)\hat{\Lambda}\hat{U}'_{*}(\hat{\mathcal{I}}_{*},:))}$, we have
      \begin{align*}
      \|\hat{J}_{*}\|^{2}&=\mathrm{max}_{k\in[K]}\hat{J}^{2}_{*}(k,k)=\mathrm{max}_{k\in[K]}e'_{k}\hat{U}_{*}(\hat{I}_{*},:)\hat{\Lambda}\hat{U}'_{*}(\hat{\mathcal{I}}_{*},:)e_{k}=\mathrm{max}_{k\in[K]}\|e'_{k}\hat{U}_{*}(\hat{I}_{*},:)\hat{\Lambda}\hat{U}'_{*}(\hat{\mathcal{I}}_{*},:)e_{k}\|\\
      &\leq \mathrm{max}_{k\in[K]}\|e'_{k}\hat{U}_{*}(\hat{I}_{*},:)\|^{2}\|\hat{\Lambda}\|\leq \mathrm{max}_{k\in[K]}\|e'_{k}\hat{U}_{*}(\hat{I}_{*},:)\|^{2}_{F}\|\hat{\Lambda}\|=\|\hat{\Lambda}\|,
      \end{align*}
where we have used the fact that $\|\hat{U}_{*}(i,:)\|_{F}=1$ for $i\in[n]$ in the last equality. Since we need $\sigma_{K}(\Omega)\geq C\theta_{\mathrm{max}}\sqrt{\tilde{P}_{\mathrm{max}}n\mathrm{log}(n)}\geq C\|A-\Omega\|$ in the proof of Lemma \ref{rowwiseerrorDCMM}, we have $\|\hat{\Lambda}\|=\|A\|=\|A-\Omega+\Omega\|\leq \|A-\Omega\|+\|\Omega\|\leq \sigma_{K}(\Omega)+\|\Omega\|\leq 2\|\Omega\|=2\|\Theta \Pi \tilde{P}\Pi'\Theta\|\leq 2C\theta^{2}_{\mathrm{max}}\tilde{P}_{\mathrm{max}}\lambda_{1}(\Pi'\Pi)=O(\theta^{2}_{\mathrm{max}}\tilde{P}_{\mathrm{max}}\lambda_{1}(\Pi'\Pi))$. Then we have
\begin{align*}
\|\hat{J}_{*}\|=O(\theta_{\mathrm{max}}\sqrt{\tilde{P}_{\mathrm{max}}\lambda_{1}(\Pi'\Pi)}).
\end{align*}
\item for $\|J_{*}-\mathcal{P}_{*}\hat{J}_{*}\mathcal{P}'_{*}\|$, we provide some simple facts first: $\|\hat{\Lambda}\|=\|A\|,\|\Lambda\|=\|\Omega\|,\Omega=U\Lambda U', \tilde{A}=\hat{U}\hat{\Lambda}\hat{U}',\|\hat{U}\|=1, \|U\|=1, \|e'_{k}\mathcal{P}_{*}\hat{B}_{2*}\|=\|\hat{B}_{2*}e_{k}\|=\|e'_{k}\hat{B}_{2*}\|\leq \|e'_{k}\hat{B}_{2*}\|_{F}=1$. Since $\tilde{A}$ is the best rank $K$ approximation to $A$ in spectral norm, therefore $\|\tilde{A}-A\|\leq \|\Omega-A\|$ since $\Omega=U\Lambda U'$ with rank $K$ and $\Omega$ can also be viewed as a rank $K$ approximation to $A$. This leads to $\|\Omega-\tilde{A}\|=\|\Omega-A+A-\tilde{A}\|\leq 2\|A-\Omega\|$. By Lemma H.2 \cite{MaoSVM}, $\|B_{*}\|=\|U_{*}(\mathcal{I},:)\|=\sqrt{\lambda_{1}(U_{*}(\mathcal{I},:)U'_{*}(\mathcal{I},:))}\leq \sqrt{\kappa(\Pi'\Theta^{2}\Pi)}\leq \frac{\theta_{\mathrm{max}}\kappa^{0.5}(\Pi'\Pi)}{\theta_{\mathrm{min}}}$. $\|A\|=\|A-\Omega+\Omega\|\leq \|A-\Omega\|+\|\Omega\|\leq \sigma_{K}(\Omega)+\|\Omega\|\leq 2\|\Omega\|$ by the lower bound requirement of $\sigma_{K}(\Omega)$ in Lemma \ref{rowwiseerrorDCMM}, and we also have $\|A-\Omega\|\leq \sigma_{K}(\Omega)\leq \|\Omega\|$.
    For $k\in[K]$, let $\tau_{k}=J_{*}(k,k), \hat{\tau}_{k}=(\mathcal{P}_{*}\hat{J}_{*}\mathcal{P}'_{*})(k,k)$ for convenience. Based on the above facts and Lemma \ref{boundCDCMM}, we have
\begin{align*}
&\mathrm{max}_{k\in[K]}|\tau^{2}_{k}-\hat{\tau}^{2}_{k}|=\mathrm{max}_{k\in[K]}\|e'_{k}U_{*}(\mathcal{I},:)\Lambda U'_{*}(\mathcal{I},:)e_{k}-e'_{k}\mathcal
P_{*}\hat{U}_{*}(\hat{\mathcal{I}}_{*},:)\hat{\Lambda}\hat{U}'_{*}(\hat{\mathcal{I}}_{*},:)P'_{*}e_{k}\|\\
&=\mathrm{max}_{k\in[K]}\|e'_{k}B_{2*}U\Lambda U'B_{2*}e_{k}-e'_{k}\mathcal
P_{*}\hat{B}_{2*}\hat{U}\hat{\Lambda}\hat{U}'\hat{B}'_{2*}P'_{*}e_{k}\|\\
&\leq\|e'_{k}(B_{2*}-\mathcal{P}_{*}\hat{B}_{2*})U\Lambda U'B'_{2*}e_{k}\|+\|e'_{k}\mathcal{P}_{*}\hat{B}_{2*}(U\Lambda U'-\hat{U}\hat{\Lambda} \hat{U}')B'_{2*}e_{k}\|+\|e'_{k}\mathcal{P}_{*}\hat{B}_{2*}\hat{U}\hat{\Lambda}\hat{U}'(B'_{2*}-\hat{B}'_{2*}\mathcal{P}'_{*})e_{k}\|\\
&\leq\|e'_{k}(B_{2*}-\mathcal{P}_{*}\hat{B}_{2*})\|\|U\|\|\Lambda\|\|U'\|\|B'_{2*}e_{k}\|+\|e'_{k}\mathcal{P}_{*}\hat{B}_{2*}\|\|U\Lambda U'-\hat{U}\hat{\Lambda}\hat{U}'\|\|B'_{2*}e_{k}\|\\ &~~~+\|e'_{k}\mathcal{P}_{*}\hat{B}_{2*}\|\|\hat{U}\|\|\hat{\Lambda}\|\|\hat{U}'\|\|(B'_{2*}-\hat{B}'_{2*}\mathcal{P}'_{*})e_{k}\|\\
&\leq\|e'_{k}(B_{2*}-\mathcal{P}_{*}\hat{B}_{2*})\|\|\Lambda\|\|B'_{2*}e_{k}\|+\|U\Lambda U'-\hat{U}\hat{\Lambda} \hat{U}'\|\|B'_{2*}e_{k}\|+\|\hat{\Lambda}\|\|(B'_{2*}-\hat{B}'_{2*}\mathcal{P}'_{*})e_{k}\|\\
&=\|e'_{k}(B_{2*}-\mathcal{P}_{*}\hat{B}_{2*})\|(\|\Omega\|\|B'_{2*}e_{k}\|+\|A\|)+\|\Omega-\tilde{A}\|\|B'_{2*}e_{k}\|\\
&=\|e'_{k}(\hat{B}_{2*}-\mathcal{P}'_{*}B_{2*})\|(\|\Omega\|\|B'_{2*}e_{k}\|+\|A\|)+\|\Omega-\tilde{A}\|\|B'_{2*}e_{k}\|\\
&\leq\|e'_{k}(\hat{B}_{2*}-\mathcal{P}'_{*}B_{2*})\|(\|\Omega\|\|B'_{2*}e_{k}\|+\|A\|)+2\|A-\Omega\|\|B'_{2*}e_{k}\|\\
&=\|e'_{k}(\hat{B}_{2*}-\mathcal{P}'_{*}B_{2*})\|(\|\Omega\|\|UB'_{*}e_{k}\|+\|A\|)+2\|A-\Omega\|\|UB'_{*}e_{k}\|\\
&\leq\|e'_{k}(\hat{B}_{2*}-\mathcal{P}'_{*}B_{2*})\|(\|\Omega\|\|B'_{*}e_{k}\|+\|A\|)+2\|A-\Omega\|\|B'_{*}e_{k}\|\\
&\leq\|e'_{k}(\hat{B}_{2*}-\mathcal{P}'_{*}B_{2*})\|(\|\Omega\|\|B'_{*}e_{k}\|+2\|\Omega\|)+2\|\Omega\|\|B'_{*}e_{k}\|\\
&=\|e'_{k}(\hat{B}_{2*}-\mathcal{P}'_{*}B_{2*})\|(\|B'_{*}e_{k}\|+1)O(\theta^{2}_{\mathrm{max}}\tilde{P}_{\mathrm{max}}\lambda_{1}(\Pi'\Pi))+\|B'_{*}e_{k}\|O(\theta^{2}_{\mathrm{max}}\tilde{P}_{\mathrm{max}}\lambda_{1}(\Pi'\Pi))\\
&\leq\|e'_{k}(\hat{B}_{2*}-\mathcal{P}'_{*}B_{2*})\|(\|B_{*}\|+1)O(\theta^{2}_{\mathrm{max}}\tilde{P}_{\mathrm{max}}\lambda_{1}(\Pi'\Pi))+\|B_{*}\|O(\theta^{2}_{\mathrm{max}}\tilde{P}_{\mathrm{max}}\lambda_{1}(\Pi'\Pi))\\
&\leq\|e'_{k}(\hat{B}_{2*}-\mathcal{P}'_{*}B_{2*})\|O(\theta^{3}_{\mathrm{max}}\tilde{P}_{\mathrm{max}}\kappa^{0.5}(\Pi'\Pi)\lambda_{1}(\Pi'\Pi)/\theta_{\mathrm{min}})+O(\theta^{3}_{\mathrm{max}}\tilde{P}_{\mathrm{max}}\kappa^{0.5}(\Pi'\Pi)\lambda_{1}(\Pi'\Pi)/\theta_{\mathrm{min}})\\
&\leq\|e'_{k}(\hat{B}_{2*}-\mathcal{P}'_{*}B_{2*})\|_{F}O(\theta^{3}_{\mathrm{max}}\tilde{P}_{\mathrm{max}}\kappa^{0.5}(\Pi'\Pi)\lambda_{1}(\Pi'\Pi)/\theta_{\mathrm{min}})+O(\theta^{3}_{\mathrm{max}}\tilde{P}_{\mathrm{max}}\kappa^{0.5}(\Pi'\Pi)\lambda_{1}(\Pi'\Pi)/\theta_{\mathrm{min}})\\
&=O(\frac{K^{3}\theta^{11}_{\mathrm{max}}\varpi\kappa^{3}(\Pi'\Pi)\lambda^{1.5}_{1}(\Pi'\Pi)}{\theta^{11}_{\mathrm{min}}\pi_{\mathrm{min}}})
O(\theta^{3}_{\mathrm{max}}\tilde{P}_{\mathrm{max}}\kappa^{0.5}(\Pi'\Pi)\lambda_{1}(\Pi'\Pi)/\theta_{\mathrm{min}})\\
&~~~+O(\theta^{3}_{\mathrm{max}}\tilde{P}_{\mathrm{max}}\kappa^{0.5}(\Pi'\Pi)\lambda_{1}(\Pi'\Pi)/\theta_{\mathrm{min}})=O(\frac{K^{3}\theta^{14}_{\mathrm{max}}\tilde{P}_{\mathrm{max}}\varpi\kappa^{3.5}(\Pi'\Pi)\lambda^{2.5}_{1}(\Pi'\Pi)}{\theta^{12}_{\mathrm{min}}\pi_{\mathrm{min}}}).
\end{align*}
Recall that $J_{*}=N_{U}(\mathcal{I},\mathcal{I})\Theta(\mathcal{I},\mathcal{I})$, we have $\|J_{*}\|\leq N_{U,\mathrm{max}}\theta_{\mathrm{max}}\leq \frac{\theta^{2}_{\mathrm{max}}\sqrt{K\lambda_{1}(\Pi'\Pi)}}{\theta_{\mathrm{min}}}$ where the last inequality holds by Lemma \ref{boundUeta}. Similarly, we have $J_{*}(k,k)\geq \theta_{\mathrm{min}}\mathrm{min}_{i\in[n]}\frac{1}{\|U(i,:)\|_{F}}\geq\frac{\theta^{2}_{\mathrm{min}}\sqrt{\lambda_{K}(\Pi'\Pi)}}{\theta_{\mathrm{max}}}$ where the last inequality holds by the proof of Lemma \ref{rowwiseerrorDCMM}. Then we have
\begin{align*}
&\|J_{*}-\mathcal{P}_{*}\hat{J}_{*}\mathcal{P}'_{*}\|=\mathrm{max}_{k\in[K]}|\hat{\tau}_{k}-\tau_{k}|=\mathrm{max}_{k\in[K]}\frac{|\hat{\tau}^{2}_{k}-\tau^{2}_{k}|}{\hat{\tau}_{k}+\tau_{k}}\leq\mathrm{max}_{k\in[K]}\frac{|\hat{\tau}^{2}_{k}-\tau^{2}_{k}|}{\tau_{k}}\\
&\leq\frac{\theta_{\mathrm{max}}}{\theta^{2}_{\mathrm{min}}\sqrt{\lambda_{K}(\Pi'\Pi)}}\mathrm{max}_{k\in[K]}|\hat{\tau}^{2}_{k}-\tau^{2}_{k}|=O(\frac{K^{3}\theta^{15}_{\mathrm{max}}\tilde{P}_{\mathrm{max}}\varpi\kappa^{3.5}(\Pi'\Pi)\lambda^{2.5}_{1}(\Pi'\Pi)}{\theta^{14}_{\mathrm{min}}\pi_{\mathrm{min}}\sqrt{\lambda_{K}(\Pi'\Pi)}}).
\end{align*}
\end{itemize}
Combine the above results, we have
\begin{align*} &\|e'_{i}(\hat{Z}_{*}-Z_{*}\mathcal{P}_{*})\|_{F}\leq\|e'_{i}(\hat{Y}_{*}-Y_{*}\mathcal{P}_{*})\|_{F}\|\hat{J}_{*}\|+\|e'_{i}Y_{*}\|_{F}\|J_{*}-\mathcal{P}_{*}\hat{J}_{*}\mathcal{P}'_{*}\|\\
&\leq O(\frac{K^{4.5}\theta^{14}_{\mathrm{max}}\varpi\kappa^{4.5}(\Pi'\Pi)\lambda_{1}(\Pi'\Pi)}{\theta^{14}_{\mathrm{min}}\pi_{\mathrm{min}}})O(\theta_{\mathrm{max}}\sqrt{\tilde{P}_{\mathrm{max}}\lambda_{1}(\Pi'\Pi)})\\
&~~~+\frac{\theta^{2}_{\mathrm{max}}\sqrt{K\lambda_{1}(\Pi'\Pi)}}{\theta^{2}_{\mathrm{min}}\lambda_{K}(\Pi'\Pi)}O(\frac{K^{3}\theta^{15}_{\mathrm{max}}\tilde{P}_{\mathrm{max}}\varpi\kappa^{3.5}(\Pi'\Pi)\lambda^{2.5}_{1}(\Pi'\Pi)}{\theta^{14}_{\mathrm{min}}\pi_{\mathrm{min}}\sqrt{\lambda_{K}(\Pi'\Pi)}})=O(\frac{\theta^{15}_{\mathrm{max}}K^{4.5}\varpi\kappa^{4.5}(\Pi'\Pi)\lambda^{1.5}_{1}(\Pi'\Pi)}{\theta^{14}_{\mathrm{min}}\pi_{\mathrm{min}}}).
\end{align*}
\end{proof}
\end{document}